\documentclass[10pt,journal,compsoc]{IEEEtran}

\ifCLASSOPTIONcompsoc
  \usepackage[nocompress]{cite}
\else
  \usepackage{cite}
\fi

\usepackage{amssymb}
\usepackage{amsmath}
\usepackage{amsthm}
\usepackage{ mathrsfs }
\newtheorem{mythm}{Theorem}
\newtheorem{mylem}{Lemma}
\newtheorem{mycor}{Corollary}

\usepackage{graphicx}
\usepackage[caption=false]{subfig}
\usepackage{sidecap}
\usepackage{booktabs}
\usepackage{longtable}
\usepackage[table]{xcolor}
\usepackage{multirow}

\usepackage[hyphens]{url}
\usepackage{comment}
\usepackage{todonotes}

\usepackage{color}
\definecolor{ChangesColor}{RGB}{0,0,0}

\usepackage{amsthm}
\theoremstyle{remark}

\begin{document}

%%%%%%%%%%%%%%%%%%%%%%%%%%%%%%%%%%%%%%%%%%%%%
%%%%%%%%%%%%%%% Title, Authors, Abstract, Keywords
%%%%%%%%%%%%%%%%%%%%%%%%%%%%%%%%%%%%%%%%%%%%%

%\title{Smoothed Generalized Regression \\ for Template Matching on SE(2)}
\title{Template Matching via Densities on the Roto-Translation Group}
%\title{\mbox{Template Matching on the Roto-Translation Group}}

\author{Erik~J.~Bekkers,
        Marco~Loog,
        Bart~M.~ter~Haar~Romeny,
        and~Remco~Duits% <-this % stops a space
\IEEEcompsocitemizethanks{
\IEEEcompsocthanksitem E.J. Bekkers and B.M. ter Haar Romeny are with the department of Biomedical Engineering, Eindhoven University of Technology (TU/e), the Netherlands. E-mail: \{e.j.bekkers,b.m.terhaarromeny\}@tue.nl
\IEEEcompsocthanksitem M. Loog is with the Pattern Recognition Laboratory, Delft University of Technology, the Netherlands. E-mail: m.loog@tudelft.nl
\IEEEcompsocthanksitem B.M. ter Haar Romeny is also with the department of Biomedical and Information Engineering, Northeastern University, Shenyang, China.
\IEEEcompsocthanksitem R. Duits is with the department of Mathematics and Computer Science, TU/e; he is also affiliated to the department of Biomedical Enginering, TU/e. Email: r.duits@tue.nl}% <-this % stops an unwanted space
\thanks{\color{white}Manuscript received ..... .., ....; revised ......... .., .....}}

% The paper headers
%\markboth{Journal of ...,~Vol.~.., No.~., .........~20..}%
%{Shell \MakeLowercase{\textit{et al.}}: Bare Demo of IEEEtran.cls for Computer Society Journals}

%, which we identify with the roto-translation group SE(2).
%The matching scheme is based on correlations on the domain of position and orientations, which we identify with the roto-translation group SE(2). As such, of complex valued invertible orientation scores.
\IEEEtitleabstractindextext{
\begin{abstract}
We propose a template matching method for the detection of 2D image objects that are characterized by orientation patterns. Our method is based on data representations via orientation scores, which are functions on the space of positions and orientations, and which are obtained via a wavelet-type transform. This new representation allows us to detect orientation patterns in an intuitive and direct way, namely via cross-correlations. Additionally, we propose a generalized linear regression framework for the construction of suitable templates using smoothing splines. Here, it is important to recognize a curved geometry on the position-orientation domain, which we identify with the Lie group SE(2): the roto-translation group. Templates are then optimized in a B-spline basis, and smoothness is defined with respect to the curved geometry. We achieve state-of-the-art results on three different applications: detection of the optic nerve head in the retina (99.83\% success rate on 1737 images), of the fovea in the retina (99.32\% success rate on 1616 images), and of the pupil in regular camera images ($95.86\%$ on 1521 images). The high performance is due to inclusion of both intensity and orientation features with effective geometric priors in the template matching. Moreover, our method is fast due to a cross-correlation based matching approach.
\end{abstract}

}

\maketitle
\IEEEdisplaynontitleabstractindextext
\IEEEpeerreviewmaketitle

%%%%%%%%%%%%%%%%%%%%%%%%%%%%%%%%%%%%%%%%%%%%%
%%%%%%%%%%%%%%% Main article
%%%%%%%%%%%%%%%%%%%%%%%%%%%%%%%%%%%%%%%%%%%%%

\IEEEraisesectionheading{\section{Introduction}}
\label{sec:intro}

\IEEEPARstart{W}{e} propose a cross-correlation based template matching scheme for the detection of objects characterized by orientation patterns. As one of the most basic forms of template matching, cross-correlation is intuitive, easy to implement, and due to the existence of optimization schemes for real-time processing a popular method to consider in computer vision tasks %\cite{Lewis1995,Yoo2009}
\cite{Yoo2009}. However, as intensity values alone provide little context, cross-correlation for the detection of objects has its limitations. More advanced data representations may be used, e.g. via wavelet transforms or feature descriptors \cite{ViolaJones2001,DalalTriggs2005,Lowe1999,Bay2006}. However, then standard cross-correlation can usually no longer be used and one typically resorts to classifiers, which take the new representations as input feature vectors. While in these generic approaches the detection performance often increases with the choice of a more complex representation, so does the computation time. In contrast, in this paper we stay in the framework of template matching via cross-correlation while working with a contextual representation of the image. To this end, we lift an image $f:\mathbb{R}^2 \rightarrow \mathbb{R}$ to an \emph{invertible orientation score} $U_f:\mathbb{R}^2 \rtimes S^1 \rightarrow \mathbb{C}$ via a wavelet-type transform using certain anisotropic filters \cite{Duits2007a,Bekkers2014}.

An orientation score is a complex valued function on the extended domain $\mathbb{R}^2 \rtimes S^1 \equiv SE(2)$ of positions and orientations, and provides a comprehensive decomposition of an image based on local orientations, see Fig.~\ref{fig:odOS} and \ref{fig:osFrame}. Cross-correlation based template matching is then defined via $\mathbb{L}_2$ inner-products of a template $T \in \mathbb{L}_2(SE(2))$ and an orientation score $U_f \in \mathbb{L}_2(SE(2))$. In this paper, we learn templates $T$  by means of generalized linear regression.

\begin{figure}[t]
\begin{center}
\includegraphics[width=\linewidth]{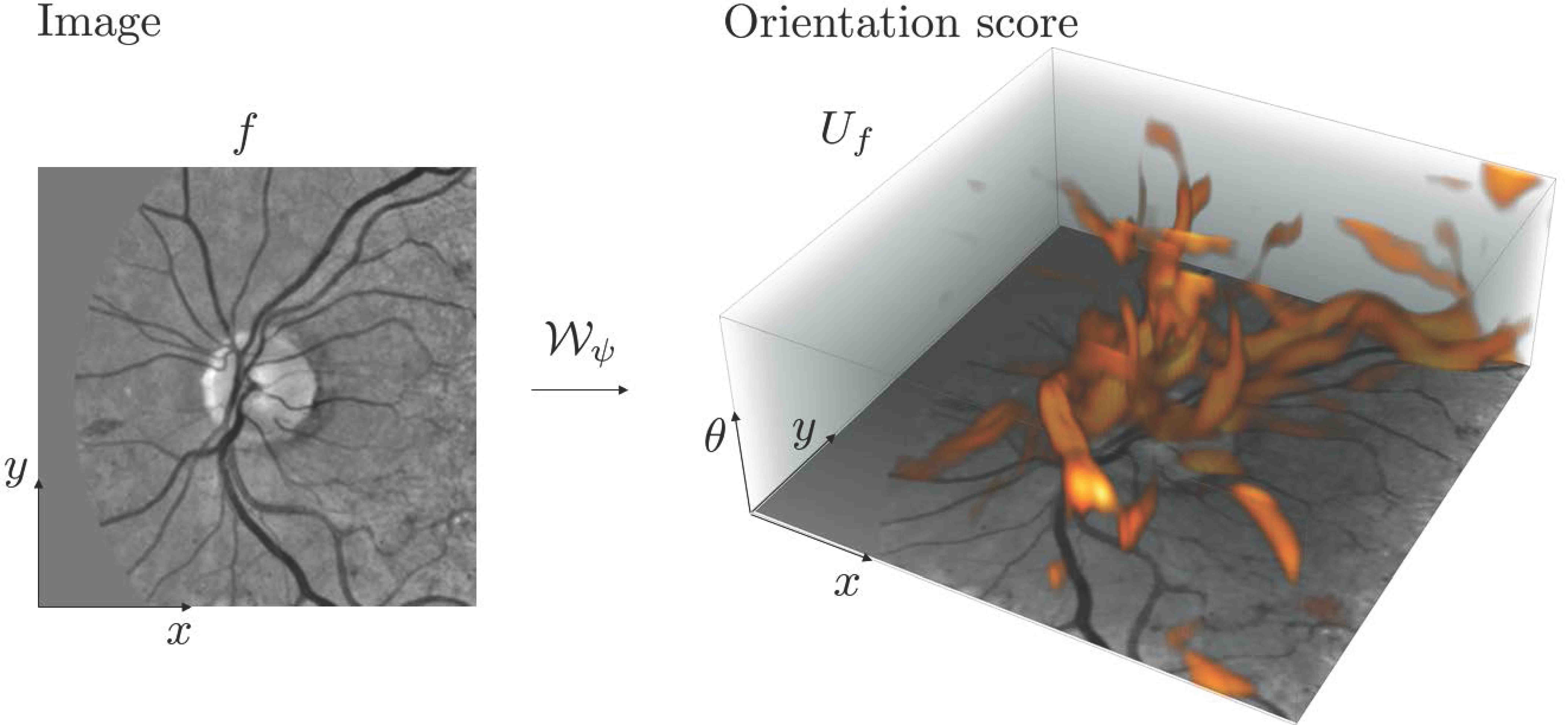}
\end{center}
\caption{A retinal image $f$ of the optic nerve head and a volume rendering of the orientation score $U_f$ (obtained via a wavelet transform $\mathcal{W}_\psi$).
%corresponding
%In an orientation score, image information is neatly organized on the basis of positions and orientations.
}
\label{fig:odOS}
\end{figure}

%
%We will stay in the framework of template matching via cross-correlation, however, the extension to orientation scores enables us to efficiently match patterns of orientation distributions, rather than pixel intensities. %Additionally, via a generalized linear regression approach we learn optimal templates to be used in this framework.

%In contrast to other orientation-based detection schemes using histograms of orientated gradients (HOG) \cite{DalalTriggs2005}, our method deals with multiple orientations per location (via orientation scores), and instead of computing histograms per sub-region of a detection window, we learn how to weigh lines structures relative to their combined position and orientation, using linear regression in the geometrical context of the Lie group $SE(2)$.

%Additionally, we propose a framework for the construction of templates $t \in \mathbb{L}_2(\mathbb{R}^2)$ and $T \in \mathbb{L}_2(SE(2))$ to be used in a cross-correlation based detection framework. We do so by optimizing energy functionals\footnote{Here we take the $\mathbb{R}^2$ case as an example and elaborate more on the extension to $SE(2)$ in Section \ref{sec:templateMatchingSE2}} of the form

In the $\mathbb{R}^2$-case (which we later extend to orientation scores, the $SE(2)$-case),
%\footnote{Here we take the $\mathbb{R}^2$ case as an example and elaborate more on the extension to $SE(2)$ in Section \ref{sec:templateMatchingSE2}.}
 we define templates $t \in \mathbb{L}_2(\mathbb{R}^2)$ via the optimization of energy functionals of the form
\begin{equation}
\label{eq:energyGeneric}
t^* =  \underset{t \in \mathbb{L}_2(\mathbb{R}^2) }{\operatorname{argmin}} \left\{ E(t) := S(t) + R(t) \right\},
\end{equation}
where the energy functional $E(t)$ consists of a data term $S(t)$, and a regularization term $R(t)$. Since the templates optimized in this form are used in a linear cross-correlation based framework, we will use inner products in $S$, in which case (\ref{eq:energyGeneric}) can be regarded as a generalized linear regression problem with a regularization term. For example, (\ref{eq:energyGeneric}) becomes a regression problem generally known under the name \emph{ridge regression} \cite{Hoerl1970}, when taking
$$
S(t) = \sum_{i=1}^N \left( (t , f_i)_{\mathbb{L}_2(\mathbb{R}^2)} - y_i\right)^2, \;\; \text{and} \;\;
R(t) = \mu \lVert t \rVert^2_{\mathbb{L}_2(\mathbb{R}^2)},
$$
where $f_i$ is one of $N$ image patches, $y_i \in \{0,1\}$ is the corresponding desired filter response, and where $\mu$ is a parameter weighting the regularization term. The regression is then from an input image patch $f_i$ to a desired response $y_i$, and the template $t$ can be regarded as the ``set of weights'' that are optimized in the regression problem. In this article we consider both quadratic (linear regression) and logistic (logistic regression) losses in $S$. For regularization we consider terms of the form
$$
R(t) = \lambda \int_{\mathbb{R}^2} \lVert \nabla t (\mathbf{x}) \rVert^2 {\rm d}\mathbf{x} + \mu \lVert t \rVert^2_{\mathbb{L}_2(\mathbb{R}^2)},
$$
and thus combine the classical ridge regression with a smoothing term (weighted by $\lambda$).

 %$SE(2)$ (the group of planar roto-transations)

In our extension of smoothed regression
%for template matching in
to orientation scores we employ similar techniques. However, here we must recognize a curved geometry on the domain $\mathbb{R}^2 \rtimes S^1$, which we identify with the group of roto-translations: the Lie group $SE(2)$ equipped with group product
\begin{equation}
\label{eq:gproduct}
g \cdot g' = (\mathbf{x},\theta)\cdot(\mathbf{x}',\theta') = (\mathbf{R}_\theta \mathbf{x}' + \mathbf{x}, \theta + \theta').
\end{equation}
In this product the orientation $\theta$ influences the product on the spatial part. Therefore we write $\mathbb{R}^{2} \rtimes S^{1}$ instead of $\mathbb{R}^{2} \times S^{1}$, as it is a semi-direct group product (and not a direct product). Accordingly, we must work with a rotating derivative frame (instead of axis aligned derivatives) that is aligned with the group elements $(\mathbf{x},\theta) \in SE(2)$,
%the Lie group $SE(2)$ (the group of planar roto-translations, that is the domain of $U_f$) we employ a similar smoothing technique. Here we must recognize a curved geometry on the domain of $U_f$ and use left-invariant gradients instead of axis aligned gradients. I.e., we rely on a rotating derivative frame that is aligned with the group elements $(\mathbf{x},\theta) \in SE(2)$.
see e.g. the $(\partial_\xi,\partial_\eta,\partial_\theta)$-frames in Fig.~\ref{fig:osFrame}. This derivative frame allows for (anisotropic) smoothing along oriented structures. As we will show in this article (Sec. \ref{sec:stochasticProcess}), the proposed smoothing scheme has the probabilistic interpretation of time integrated Brownian motion on $SE(2)$ \cite{ZhangDuits2014,Duits2010}.

%In our extension to $SE(2)$, a smoothness penalty term is defined using a rotating derivative frame (see Fig.~\ref{fig:osFrame}) that is aligned with the group elements $g \in SE(2)$. As we will show in this paper,

%that follows the geometrical principles of the Lie-group $SE(2)$.

%we will work with a rotating frame of reference to compute left-invariant derivatives, resulting in a regularization term where we we can control the amount of smoothness along line structures.

\begin{figure}[t]
\begin{center}
\includegraphics[width=\linewidth]{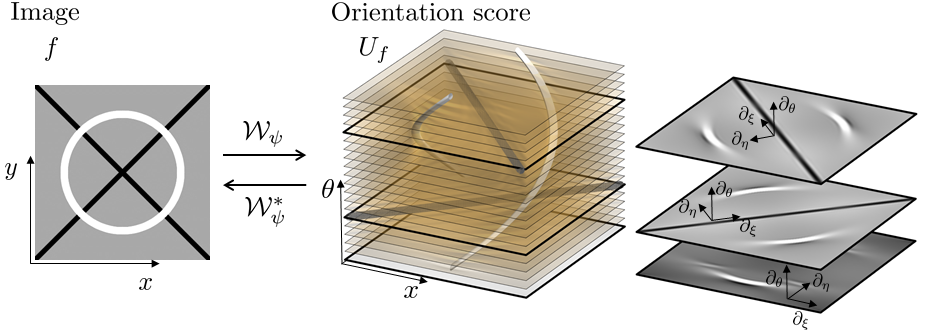}
\end{center}
\caption{In orientation scores $U_f$, constructed from an image $f$ via the orientation score transform $\mathcal{W}_\psi$, we make use of a left-invariant derivative frame $\{\partial_\xi,\partial_\eta,\partial_\theta\}$ that is aligned with the orientation $\theta$ corresponding to each layer in the score. Three slices and the corresponding left-invariant frames are shown separately (at $\theta \in \{0, \frac{\pi}{4}, \frac{3\pi}{4}\}$).}
\label{fig:osFrame}
\end{figure}

{\color{black} \emph{\textbf{Regression and Group Theory.}} }
Regularization in (generalized) linear regression generally leads to more robust classifiers/regressions, especially when a low number of training samples are available. Different types of regularizations in regression problems have been intensively studied in e.g. \cite{HastieBook,Hebiri2011,Cuingnet2013,Xu2009,Qazi2010}, and the choice for regularization-type depends on the problem: E.g. $\mathbb{L}_1$-type regularization is often used to sparsify the regression weights, whereas $\mathbb{L}_2$-type regularization is more generally used to prevent over-fitting by penalizing outliers (e.g. in ridge regression \cite{Hoerl1970}). Smoothing of regression coefficients by penalizing the $\mathbb{L}_2$-norm of the derivative along the coefficients is less common, but it can have a significant effect on performance \cite{LiLi2008,Cuingnet2013}.
%\cite{Tibshirani1994}

%Smoothing regularization is an essential part of non-parametric regression, and our problem falls in this category.
%: we do not put any restrictions on the shape of our templates $t$ and $T$.
We solve problem (\ref{eq:energyGeneric}) in the context of smoothing splines: We discretize the problem by expanding the templates in a finite B-spline basis, and optimize over the spline coefficients. For d-dimensional Euclidean spaces, smoothing splines have been well studied \cite{deBoor1978,Green1993,Unser1999,Unser93}. In this paper, we extend the concept to the curved space $SE(2)$ and provide explicit forms of the discrete regularization matrices. Furthermore, we show that the extended framework can be used for time integrated Brownian motions on $SE(2)$, and show near to perfect comparisons to the exact solutions found in \cite{ZhangDuits2014,Duits2010}.

In general, statistics and regression on Riemannian manifolds are powerful tools in medical imaging and computer vision \cite{Fletcher2013,Miolane2015,Pennec2006,Vidal2005}. More specifically in pattern matching and registration problems Lie groups are often used to describe deformations. E.g. in \cite{Tuzel2008} the authors learn a regression function $\mathbb{R}^m \rightarrow A(2)$ from a discrete $m$-dimensional feature vector to a deformation in the affine group $A(2)$. Their purpose is object tracking in video sequences. This work is however not concerned with deformation analysis, we instead learn a regression function $\mathbb{L}_2(SE(2))\rightarrow\mathbb{R}$ from continuous densities on the Lie group $SE(2)$ (obtained via an invertible orientation score transform) to a desired filter response. Our purpose is object detection in 2D images. In our regression we impose smoothed regression with a time-integrated hypo-elliptic Brownian motion prior and thereby extend least squares regression to smoothed regression on SE(2) involving first order variation in Sobolev-type of norms.

\emph{\textbf{Application Area of the Proposed Method.}}
The strength of our approach is demonstrated with the application to anatomical landmark detection in medical retinal images and pupil localization in regular camera images. In the retinal application we consider the problem of detecting the optic nerve head (ONH) and the fovea.
%, which are the two main anatomical landmarks in the retina, see Fig. \ref{fig:fundus}.
Many image analysis applications require the robust, accurate and fast detection of these structures, see e.g. \cite{Patton2006,Ramakanth2014,GegundezArias2013,Hansen2010}.
%In Sec.~\ref{sec:applications} the three respective problems will be discussed in more detail.
In all three detection problems the objects of interest are characterized by (surrounding) curvilinear structures (blood vessels in the retina; eyebrow, eyelid, pupil and other contours for pupil detection), which are conveniently represented in \emph{invertible} orientation scores. The invertibility condition implies that all image data is contained in the orientation score \cite{Duits2007}\cite{Bekkers2014}.
With the proposed method we achieve state-of-the-art results both in terms of detection performance and speed: high detection performance is achieved by learning templates that make optimal use of the line patterns in orientation scores; speed is achieved by a simple, yet effective, cross-correlation template matching approach. %Moreover, our method does not require vessel or ONH segmentation, or other costly preprocessing steps and is because of the cross-correlation approach easy to parallelize.

\emph{\textbf{Contribution of this Work.}}  This article builds upon two published conference papers \cite{Bekkers2014a,Bekkers2015EMMCVPR}. In the first we demonstrated that high detection performance could be achieved by considering cross-correlation based template matching in $SE(2)$, using only handcrafted templates and with the application of ONH detection in retinal images \cite{Bekkers2014a}. In the second we then showed on the same application that better performance could be achieved by training templates using the optimization of energy functionals of the form of ($\ref{eq:energyGeneric}$), where then only a (left-invariant) smoothing regularizer was considered \cite{Bekkers2015EMMCVPR}. In this article we provide a complete framework for training of templates and matching on $SE(2)$ and contribute to literature by:
\begin{enumerate}
\item Extending the linear regression $SE(2)$ framework\cite{Bekkers2015EMMCVPR} to logistic regression, with clear benefits in pupil detection using a single template (with an increase of success rate from $76\%$ to $94\%$).
\item Studying different types of regression priors, now introducing also a %{\color{changescolor2} (common)}
ridge regression prior.
\item We show that the $SE(2)$ smoothing prior corresponds to time-integrated hypo-elliptic diffusion on $SE(2)$, providing a Brownian motion interpretation.
%Establishing a link of the $SE(2)$ smoothing prior with hypo-elliptic diffusion equations on $SE(2)$.
\item We show the generic applicability of our method: with the exact same settings of our algorithm we obtain state-of-the-art results on three different applications (ONH detection, cf. Ch.~\ref{subsec:ONHDetection} and Table \ref{tab:stateOfTheArtONH}, fovea detection, cf. Subsec.~\ref{subsec:FoveaDetection} and Table \ref{tab:stateOfTheArtFovea}, and pupil detection, cf. Subsec.~\ref{subsec:PupilDetection} and Fig.~\ref{fig:PupilTemplates}).
% in ONH detection (Ch.~\ref{} and Table \ref{}), Fovea detection (Ch.~\ref{} and Table \ref{}) and pupil detection (Ch.~\ref{} and Fig.~\ref{}).}
%Showing state-of-the-art performance on two new benchmark applications: fovea and pupil detection.
\item Improving previous results on ONH detection (reducing the number of failed detections to 3 out of 1737 images).
\item Making our code publicly available at \url{http://erikbekkers.bitbucket.org/TMSE2.html}.
\end{enumerate}

%Our method does not require vessel or ONH segmentation, or other costly preprocessing steps and is because of its cross-correlation approach easy to parallelize. Out implementation in \emph{Mathematica}

\begin{comment}
\begin{figure}[b]
\begin{center}
\includegraphics[width=0.9\linewidth]{Figs/ODC_Color_TOPCON_ES.jpg}
\end{center}
\caption{Two fundus images of the same eye. A conventional color fundus image (left) and a pseudo color image constructed from a near infrared and green laser image (right).}
\label{fig:fundus}
\end{figure}
\end{comment}

%\subsection{Paper Outline}
\emph{\textbf{Paper Outline.}}
The remainder of this paper is organized as follows. In Sec.~\ref{sec:templateMatchingR2} we provide the theory for template matching and template construction in the $\mathbb{R}^2$-case. The theory is then extended to the $SE(2)$-case in Sec.~\ref{sec:templateMatchingSE2}.
%, including an introduction to orientation scores (Subsec.~\ref{subsec:orientationScores}).
Additionally, in Sec.~\ref{sec:stochasticProcess} we provide a probabilistic interpretation of the proposed $SE(2)$ prior, and relate it to Brownian motions on $SE(2)$.
%In Sec.~\ref{sec:normalizedCC} the method is further extended for use in a \emph{normalized} cross-correlation based framework.
%Although this section gives a very important probabilistic background to our choice of prior, it is not essential for the interpretation of the method, and may therefore be skipped before reading the application section.
In Sec.~\ref{sec:applications} we apply the method to retinal images for ONH (Subsec.~\ref{subsec:ONHDetection}) and fovea detection (Subsec.~\ref{subsec:FoveaDetection}), and to regular camera images for pupil detection (Subsec.~\ref{subsec:PupilDetection}). Finally, we conclude the paper in Sec.~\ref{sec:discussionAndConclusion}.
%in Subsec.~\ref{subsec:discussion} we provide a general discussion of the method, and conclude the paper in Subsec.~\ref{subsec:conclusion}.

\section{Template Matching \& Regression on $\mathbb{R}^2$}
\label{sec:templateMatchingR2}

\subsection{Object Detection via Cross-Correlation}
\label{subsec:objectDetectionR2}
We are considering the problem of finding the location of objects (with specific orientation patterns) in an image. While in principle an image may contain multiple objects of interest, the applications discussed
%that we consider
in this paper only require the detection of one object per image. We search for the most likely location
\begin{equation}
\label{eq:objectDetection}
\mathbf{x}^* = \underset{\mathbf{x} \in \mathbb{R}^2 }{\operatorname{argmax}} \;\;P(\mathbf{x}),
\end{equation}
with $P(\mathbf{x}) \in \mathbb{R}$ denoting the objective functional for finding the object of interest at location $\mathbf{x}$. We define $P$ based on inner products in a \emph{linear regression} and \emph{logistic regression} context, where we respectively define $P$ by
\begin{equation}
\label{eq:linearFunctional}
P(\mathbf{x}) = P_{lin}^{\mathbb{R}^2}(\mathbf{x}) := ( \mathcal{T}_\mathbf{x} \;t , f)_{\mathbb{L}_2(\mathbb{R}^2)},
\end{equation}
or
\begin{equation}
\label{eq:logisticFunctional}
\begin{aligned}
& P(\mathbf{x}) = P_{log}^{\mathbb{R}^2}(\mathbf{x})  := \sigma \left( ( \mathcal{T}_\mathbf{x} \; t , f)_{\mathbb{L}_2(\mathbb{R}^2)} \right),\\
& \text{with} \;\;\;\; \sigma(x)  =e^x/(1+e^x),
\end{aligned}
\end{equation}
%with sigmoid function $\sigma(x)=\frac{e^x}{1+e^x}$,
where $\mathcal{T}_\mathbf{x}$ denotes translation by $\mathbf{x}$ via
$$
( \mathcal{T}_\mathbf{x} t )(\tilde{\mathbf{x}}) = t(\tilde{\mathbf{x}} - \mathbf{x}),
$$
and where the $\mathbb{L}_2(\mathbb{R}^2)$ inner product is given by
\begin{equation}
(t,f)_{\mathbb{L}_2 (\mathbb{R}^2)} := \int_{\mathbb{R}^2} \overline{t(\tilde{\mathbf{x}})} f(\tilde{\mathbf{x}}) {\rm d}\tilde{\mathbf{x}},
\end{equation}
with associated norm $\lVert \cdot \rVert_{\mathbb{L}_2(\mathbb{R}^2)} = \sqrt{ (\cdot , \cdot )_{\mathbb{L}_2(\mathbb{R}^2)} }$.
Note that the inner-product based potentials $P(\mathbf{x})$ can be efficiently evaluated for each $\mathbf{x}$ using convolutions.

For a generalization of cross-correlation based template matching to \emph{normalized} cross correlation, we refer the reader to the supplementary materials. For speed considerations we will however not use normalized cross correlation, but instead use a (fast) preprocessing step to locally normalize the images (cf. Subsec.~\ref{subsubsec:processingPipeline}).

\subsection{Optimizing $t$ Using Linear Regression}
\label{subsec:linearRegresionR2}
Our aim is to construct templates $t$ that are ``aligned'' with image patches that contain the object of interest, and which are orthogonal to non-object patches. Hence, template $t$ is found via the minimization of the following energy %functional
\begin{multline}
\label{eq:energyR2LinearRegression}
E_{lin}(t) = \sum\limits_{i=1}^N \left(  ( t , f_i )_{\mathbb{L}_2(\mathbb{R}^2)} - y_i \right)^2 \\
 +  \lambda \; \int_{\mathbb{R}^2} \lVert \nabla t(\tilde{\mathbf{x}}) \rVert^2  {\rm d}\tilde{\mathbf{x}} + \mu \; \lVert t \rVert^2_{\mathbb{L}_2(\mathbb{R}^2)},
\end{multline}
with $f_i$ one of the $N$ training patches extracted from an image $f_\mathbf{x}$, and $y_i$ the corresponding label ($y_i=1$ for \emph{objects} and $y_i=0$ for \emph{non-objects}).
%For normalization of the data we consider two approaches (equally effective in practice). Either we normalize with respect to $\Vert \cdot \rVert_{\mathbb{L}_2(\mathbb{R}^2, m {\rm d}\tilde{\mathbf{x}})}$ as explained in Sec.~\ref{sec:normalizedCC}, or we normalize via application specific pre-processing pipelines, explained in Subsec.~\ref{subsubsec:processingPipeline}.
In (\ref{eq:energyR2LinearRegression}), the data-term (first term) aims for alignment of template $t$ with object patches, in which case the inner product $( t , f_i )_{\mathbb{L}_2(\mathbb{R}^2)}$ is ideally one, and indeed aims orthogonality to non-object patches (in which case the inner product is zero). The second term enforces spatial smoothness of the template by penalizing its gradient, controlled by $\lambda$. The third (ridge) term improves stability by dampening the $\mathbb{L}_2$-norm of $t$, controlled by $\mu$.
%Considering energy (\ref{eq:energyR2LinearRegression}), the most optimal template is defined by
%\begin{equation}
%\label{eq:optimizationR2LinearRegression}
%\hat{t}_{lin} =  \underset{\hat{t} \in \mathbb{L}_2(\mathbb{R}^2) }{\operatorname{argmin}} \;\; E^{\mathbb{R}^2}_{lin}(\hat{t}).
%\end{equation}
%We attach the label $lin$ to the optimal template as the template arises from a (regularized) linear regression problem, with the regression from input image patches $\hat{f}_i$ to desired responses $y_i \in \{0,1\}$.

\subsection{Optimizing $t$ Using Logistic Regression}
\label{subsec:logisticRegresionR2}
In object detection we are essentially considering a two-class classification problem: the object is either present or it is not. In this respect, the quadratic loss term in (\ref{eq:energyR2LinearRegression}) might not be the best choice as it penalizes any deviation from the desired response $y_i$, regardless of whether or not the response $( t , f_i )_{\mathbb{L}_2(\mathbb{R}^2)}$ is on the correct side of a decision boundary. %If we simply would like a template to have a high response when taking the inner product with positive \emph{object} image patches $\hat{f}_i$, and a low response otherwise, the quadratic loss might not be optimal in the sense that it penalizes ``overshooting''. I.e., any deviation from the desired response $y_i$ is quadratically punished, even for large positive responses in the case of ``easy'' \emph{object} patches.
In other words, the aim is not necessarily to construct a template that best maps an image patch $f_i$ to a response $y_i\in\{0,1\}$, but rather the aim is to construct a template that best makes the separation between \emph{object} and \emph{non-object} patches. With this in mind we resort to the logistic regression model, in which case we interpret the non-linear objective functional given in (\ref{eq:logisticFunctional}) as a probability, and define
\begin{equation}
\begin{array}{rl}
p_{1}( f_i \; ; \; t )     &= p( f_i \; ; \; t ),\\
p_{0}( f_i \; ; \; t ) &= 1 - p( f_i \; ; \; t),\\
\multicolumn{2}{c}{\text{with }  p( f_i \; ; \; t) = \sigma \left( ( t , f_i)_{\mathbb{L}_2(\mathbb{R}^2)} \right),}\\
%\text{with } & \;\;\; p( f_i \; ; \; t) = \sigma \left( ( t , f_i)_{\mathbb{L}_2(\mathbb{R}^2,m d \tilde{\mathbf{x}})} \right),
\end{array}
\end{equation}
with $p_{1}( f_i  ;  t )$ and $p_{0}( f_i  ;  t )$ denoting respectively the probabilities of a patch $f_i$ being an \emph{object} or \emph{non-object} patch. Our aim is now to maximize the likelihood (of each patch $f_i$ having maximum probability $p_{y_i}(f_i ; t)$ for correct label $y_i$):
\begin{equation}
\ell(t) = \prod_{i=1}^N p_{y_i}(f_i;t) = \prod_{i=1}^N p( f_i ; t)^{y_i} (1 - p( f_i ; t))^{1-y_i}.
\end{equation}
We maximize the log-likelood instead,
%Due to the non-linearity in the logistic sigmoid $\sigma$, it is more convenient to maximize the log-likelihood instead,
which is given by
\begin{equation}
\label{eq:loglikelihood}
\begin{array}{l}
%\begin{aligned}
\!\! \ell_{log}(t) := \log ( \; \ell(t) \; ) \\
\;\;= \sum\limits_{i=1}^N \log ( \; p( f_i ; t)^{y_i} (1 - p( f_i ; t))^{1-y_i} \; )\\
%
%\;\;= \sum\limits_{i=1}^N y_i \log( \; p( f_i  ;  t ) \; ) + (1 - y_i) \log( \; 1 - p( f_i  ;  t ) \; )\\
\;\;= \sum\limits_{i=1}^N y_i ( t , f_i )_{\mathbb{L}_2(\mathbb{R}^2)}
- \log\left(
1 + e^{
( t , f_i )_{   \mathbb{L}_2(  \mathbb{R}^2  )   }
}
\right).
%\end{aligned}
\end{array}
\end{equation}%
Maximizing (\ref{eq:loglikelihood}) is known as the problem of logistic regression. Similar to the linear regression case, we impose additional regularization and define the following regularized logistic regression energy, which we aim to \emph{maximize}:
\begin{equation}
\label{eq:energyR2LogisticRegressionLogLikelihood}
E^{\ell}_{log}(t) = \ell_{log}(t) -  \lambda \int_{\mathbb{R}^2} \lVert \nabla t(\tilde{\mathbf{x}}) \rVert^2  {\rm d}\tilde{\mathbf{x}}
- \; \mu \; \lVert t \rVert^2_{ \mathbb{L}_2( \mathbb{R}^2 ) }.
\end{equation}

\subsection{Template Optimization in a B-Spline Basis}
\label{subsec:splineBasisR2}

%\subsubsection{Templates in a B-Spline Basis}
\textbf{\emph{Templates in a B-Spline Basis.}}
%\textbf{Templates in a B-Spline Basis.}
In order to solve the optimizations (\ref{eq:energyR2LinearRegression}) and (\ref{eq:energyR2LogisticRegressionLogLikelihood}), the template is described in a basis of direct products of $n$-th order B-splines~$B^n$:
\begin{equation}
\label{eq:tBsplineR2}
t(x,y) = \sum \limits_{k=1}^{N_k} \sum \limits_{l=1}^{N_l} c_{k,l} \;B^n\!\left( \frac{x}{s_k} - k \right)B^n\!\left( \frac{y}{s_l} - l \right),
\end{equation}
with $B^n(x) = \left(1_{\left[-\frac{1}{2},\frac{1}{2}\right]}*^{(n)}1_{\left[-\frac{1}{2},\frac{1}{2}\right]}\right)(x)$ a $n$-th order B-spline obtained by $n$-fold convolution of the indicator function $1_{\left[-\frac{1}{2},\frac{1}{2}\right]}$, and $c_{k,l}$ the coefficients belonging to the shifted B-splines. Here $s_k$ and $s_l$ scale the B-splines and typically depend on the number $N_k$ and $N_l$ of B-splines.

%{\color{changescolor} The use of a B-spline basis allows us to compute compact forms of the continuous derivative (regularizaiton term) whilst using a finite discrete basis. Moreover, using this basis we can reduce the degrees of freedom by reducing the number $N_k$ and $N_l$ of basis functions used, thus reducing computation load.}

%\subsubsection{Linear Regression \label{ch:linearregression}}
\textbf{\emph{Linear Regression. \label{ch:linearregression}}}
%\textbf{Linear Regression.\label{ch:linearregression}}
\label{thm:R2Data}
By substitution of (\ref{eq:tBsplineR2}) in (\ref{eq:energyR2LinearRegression}), the energy functional can be expressed in matrix-vector form (see Section 2 of the supplementary materials):
\begin{equation}
\label{eq:energyR2LinearRegressionDiscrete}
E_{lin}^{B}(\mathbf{c}) = \lVert S \mathbf{c} - \mathbf{y} \rVert^2 + \lambda \; \mathbf{c}^\dagger R \mathbf{c} + \mu \; \mathbf{c}^\dagger I \mathbf{c}.
\end{equation}
Regarding our notations we note that for spatial template $t$ given by (\ref{eq:tBsplineR2}) we have  $E_{lin}(t)= E^{B}_{lin}(\mathbf{c})$, and label `B' indicates finite expansion in the B-spline basis.
The minimizer of (\ref{eq:energyR2LinearRegressionDiscrete}) is given by
\begin{equation}
\label{eq:minimizerR2LinearRegression}
%(S^\dagger S + \lambda_N \; R + \mu_N \; I)\mathbf{c} = S^\dagger \mathbf{y},
(S^\dagger S + \lambda R + \mu I)\mathbf{c} = S^\dagger \mathbf{y},
\end{equation}
with $^\dagger$ denoting the conjugate transpose, and $I$ denoting the identity matrix.
%, and where we introduced the variables $\lambda_N = N \lambda$ and $\mu_N = N \mu$ for notational convenience.
Here $S$ is a $[N \times N_k N_l]$ matrix given by
\begin{equation}
\begin{array}{rl}
S &= \{(s_{1,1}^i,...,s_{1,N_l}^i,s_{2,1}^i,...,s_{2,N_l}^i,...,...,s_{N_k,N_l}^i)\}_{i=1}^N,\\
s_{k,l} &= (\; B_{s_ks_l}^n * {f}_i \; )(k,l),
\end{array}
\end{equation}
with $B_{s_ks_l}^n (x,y)= B^n\!\left( \frac{x}{s_k} \right)B^n\!\left( \frac{y}{s_l} \right)$, for all (x,y) on the discrete spatial grid on which the input image ${f}_D:\{1,N_x\}\times\{1,N_y\}\rightarrow\mathbb{R}$ is defined. Here $N_k$ and $N_l$ denote the number of splines in resp. $x$ and $y$ direction, and $s_k=\frac{N_x}{N_k}$ and $s_l=\frac{N_y}{N_l}$ are the corresponding resolution parameters. The $[N_k N_l \times 1]$ column vector $\mathbf{c}$ contains the B-spline coefficients, and the $[N \times 1]$ column vector $\mathbf{y}$ contains the labels, stored in the following form
\begin{equation}
\begin{array}{rl}
%\label{eq:cyArray}
\mathbf{c} &= (c_{1,1},...,c_{1,N_l},c_{2,1},...,c_{2,N_{l}},...,...,c_{N_k,N_l})^T\\
\mathbf{y} &= (y_1,y_2,...,y_N)^T.
\end{array}
\end{equation}
The $[N_k N_l \times N_k N_l]$ regularization matrix $R$ is given by
\begin{equation}
\label{eq:regularizationMatrixR2}
R = R_x^{s_k} \otimes R_x^{s_l} + R_y^{s_k} \otimes R_y^{s_l},
\end{equation}
where $\otimes$ denotes the Kronecker product, and with
\begin{equation}
\label{eq:regularizationMatrixR2Elements}
\begin{array}{ll}
R_x^{s_k}(k,k') &= -\frac{1}{s_k}\frac{ \partial^2 B^{2n+1}}{\partial x^2}(k'-k), \\
R_x^{s_l}(l,l') &= s_l B^{2n+1}(l'-l), \\
R_y^{s_k}(k,k') &= s_k B^{2n+1}(k'-k),\\
R_y^{s_l}(l,l') &= -\frac{1}{s_l}\frac{ \partial^2 B^{2n+1}}{\partial y^2}(l'-l),
\end{array}
\end{equation}
with $k,k'={1,2,...,N_k}$ and $l,l'={1,2,...,N_l}$. The coefficients $\mathbf{c}$ can then be computed by solving (\ref{eq:minimizerR2LinearRegression}) directly, or via linear system solvers such as conjugate gradient descent. For a derivation of the regularization matrix $R$ we refer to supplementary materials, Sec.~2.
%Note that matrix $R$ is highly sparse because of the limited extend of the B-spline basis functions.
%Eq.~(\ref{eq:discreteMinimizer}) can be solved for $\mathbf{c}$ using a conjugate gradient method.

%\subsubsection{Logistic Regression \label{ch:log}}
\textbf{\emph{Logistic Regression. \label{ch:log}}}
%\textbf{Logistic Regression.\label{ch:log}}
The logistic regression log-likelihood functional (\ref{eq:energyR2LogisticRegressionLogLikelihood}) can be expressed in matrix-vector notations as follows:
\begin{multline}
\label{eq:energyR2LogisticRegressionLogLikelihoodDiscrete}
E_{log}^{\ell,B}(\mathbf{c}) = \left[ \mathbf{y}^\dagger S \mathbf{c} - \mathbf{1}_N^\dagger \log(\mathbf{1}_N + \exp( S \mathbf{c} ) ) \right] \\
- \lambda \; \mathbf{c}^\dagger R \mathbf{c} - \mu \; \mathbf{c}^\dagger I \mathbf{c},
\end{multline}
where $\mathbf{1}_N = \{1,1,...,1\}^T \in \mathbb{R}^{N \times 1}$, and where the exponential and logarithm are evaluated element-wise. We follow a standard approach for the optimization of (\ref{eq:energyR2LogisticRegressionLogLikelihoodDiscrete}), see e.g. \cite{HastieBook}, and find the minimizer by settings the derivative to $\mathbf{c}$ to zero
\begin{equation}
\label{eq:minimizerR2LogisticRegression}
%\frac{\partial E_{log}^{\ell,D}(\mathbf{c})}{\partial \mathbf{c}}
\nabla_{\mathbf{c}} E_{log}^{\ell,B}(\mathbf{c}) =  S^T (\mathbf{y} - \mathbf{p}) - \lambda \; R \mathbf{c} - \mu \; I \mathbf{c} = \mathbf{0},
\end{equation}
with $\mathbf{p}= (p_1,...,p_N)^T \in \mathbb{R}^{N \times 1}$, with $p_i = \sigma ( (S\mathbf{c})_i)$. To solve (\ref{eq:minimizerR2LogisticRegression}), we use a Newton-Raphson optimization scheme. This requires computation of the Hessian matrix, given by
\begin{equation}
%\frac{\partial^2 E_{log}^{\ell,D}(\mathbf{c})}{\partial \mathbf{c}^2}
\mathcal{H}(E_{log}^{\ell,B}) = - (S^T W S + \lambda \; R + \mu \; I),
\end{equation}
%with $W = \operatorname{diag}( \mathbf{p}  )(1 - \operatorname{diag}( \mathbf{p}  ))$.
with diagonal matrix $W = \underset{i \in \{1,...,N\}}{\operatorname{diag}} \left\{ p_i (1 - p_i) \right\}$.
The Newton-Raphson update rule is then given by
\begin{equation}
\label{eq:updateRule}
\begin{aligned}
%\mathbf{c}^{new} & = \mathbf{c}^{old} - \left(\frac{\partial^2 E_{log}^{\ell,D}(\mathbf{c})}{\partial \mathbf{c}^2}\right)^{-1}\frac{\partial E_{log}^{\ell,D}(\mathbf{c})}{\partial \mathbf{c}}\\
\mathbf{c}^{new} & = \mathbf{c}^{old} - \mathcal{H}(E_{log}^{\ell,D})^{-1} (\nabla_{\mathbf{c}} E_{log}^{\ell,D}(\mathbf{c})) \\
& = (S^T W S + \lambda \; R + \mu \; I)^{-1} S^T W \mathbf{z},
\end{aligned}
\end{equation}
with $\mathbf{z} = S \mathbf{c}^{old} + W^{-1} (\mathbf{y} - \mathbf{p})$, see e.g. \cite[ch. 4.4]{HastieBook}.
% for a similar derivation. In each iteration in the Newton-Raphson algorithm, $W$ and $\mathbf{z}$ are updated, and (\ref{eq:updateRule}) is solved using a linear system solver.
%We denote the optimal coefficients found at convergence of the algorithm by $\mathbf{c}^*$.
Optimal coefficients found at convergence are denoted with $\mathbf{c}^*$.

Summarizing, we obtain the solution of (\ref{eq:objectDetection}) by substituting the optimized B-spline coefficients $\mathbf{c}^*$ into (\ref{eq:tBsplineR2}), and the resulting $t$ enters (\ref{eq:linearFunctional}) or (\ref{eq:logisticFunctional}). The most likely object location $\mathbf{x}^*$ is then found via (\ref{eq:objectDetection}).

%\section{Template Matching \& Regression on $SE(2)$}
\section{\mbox{Template Matching \& Regression on $SE(2)$}}
\label{sec:templateMatchingSE2}

This section starts with details on the representation of image data in the form of orientation scores (Subsec. \ref{subsec:orientationScores}). Then, we repeat the sections from Sec. \ref{sec:templateMatchingR2} in Subsections \ref{subsec:objectDetectionSE2} to \ref{subsec:splineBasisSE2}, but now in the context of the extended domain $SE(2)$.

\subsection{Orientation Scores on $SE(2)$}
\label{subsec:orientationScores}
%\subsubsection{Transformation}
\textbf{\emph{Transformation.}}
%\textbf{Transformation.}
An orientation score, constructed from image $f:\mathbb{R}^2 \to \mathbb{R}$, is defined as a function $U_f : \mathbb{R}^2 \rtimes S^1 \rightarrow \mathbb{C}$ and depends on two variables ($\mathbf{x},\theta$), where $\mathbf{x}=(x,y) \in \mathbb{R}^2$ denotes position and $\theta \in [0,2\pi)$ denotes the orientation variable. An orientation score $U_f$ of image $f$ can be constructed by means of correlation with some anisotropic wavelet $\psi$ via
\begin{equation}
\label{eq:ostransform}
U_f(\mathbf{x},\theta) = (\mathcal{W}_\psi f)(\mathbf{x},\theta) = \int_{\mathbb{R}^2}\overline{\psi(\mathbf{R}_\theta^{-1}(\tilde{\mathbf{x}}-\mathbf{x}))}f(\tilde{\mathbf{x}}){\rm d}\tilde{\mathbf{x}},
\end{equation}
where  $\psi \in \mathbb{L}_2 (\mathbb{R}^2)$ is the correlation kernel, aligned with the $x$-axis, where $\mathcal{W}_\psi$ denotes the transformation between image $f$ and orientation score $U_f$, $\psi_\theta(\mathbf{x}) =  \psi(\mathbf{R}_\theta^{-1}\mathbf{x})$, and $\mathbf{R}_\theta$ is a counter clockwise rotation over angle $\theta$. %Note that $\tilde{\mathbf{x}} \in \mathbb{R}^2$ denotes a location in the image domain, whereas $(\mathbf{x},\theta)$ denotes a location in the orientation score domain.

In this work we choose cake wavelets \cite{Duits2007a,Bekkers2014} for $\psi$. While in general any kind of anisotropic wavelet could be used to lift the image to $SE(2)$, cake wavelets ensure that no data-evidence is lost during the transformation: By design the set of all rotated wavelets uniformly cover the full Fourier domain of disk-limited functions with zero mean, and have thereby the advantage over other oriented wavelets (s.a. Gabor wavelets for specific scales) that they capture all scales and allow for a stable inverse transformation $\mathcal{W}_\psi^*$ from the score back to the image \cite{Duits2010,Duits2007a}.

%\subsubsection{Left-Invariant derivatives}
\textbf{\emph{Left-Invariant Derivatives.}}
%\textbf{Left-Invariant Derivatives.}
The domain of an orientation score is essentially the classical Euclidean motion group $SE(2)$ of planar translations and rotations, and is equipped with group product $g \cdot g' = (\mathbf{x},\theta)\cdot(\mathbf{x}',\theta') = (\mathbf{R}_\theta \mathbf{x}' + \mathbf{x}, \theta + \theta')$. Here, we can recognize a curved geometry (cf. Fig.~\ref{fig:osFrame}), and it is therefore useful to work in rotating frame of reference. As such, we use the left invariant derivative frame \cite{Duits2010,ZhangDuits2014}:
\begin{equation}
\label{eq:leftInvariantDerivatives}
%\begin{array}{rl}
\left\{ \partial_{\xi}  := \cos \theta \, \partial_{x} +\sin \theta \, \partial_{y}, \partial_{\eta}  := -\sin \theta \, \partial_{x} + \cos \theta \, \partial_{y},
\partial_{\theta} \right\}.
%\end{array}
\end{equation}
Using this derivative frame we will construct in Subsec.~\ref{subsec:linearRegressionSE2} a regularization term in which we can control the amount of (anisotropic) smoothness along line structures.

\subsection{Object Detection via Cross-Correlation}
\label{subsec:objectDetectionSE2}
As in Section \ref{sec:templateMatchingR2}, we search for the most likely {object} location $\mathbf{x}^*$ via (\ref{eq:objectDetection}), but now we define functional $P$ respectively for the linear and logistic regression case in $SE(2)$ by\footnote{Since both the inner product and the construction of orientation scores $U_f$ from images $f$ are linear, template matching might as well be performed directly on the 2D images (likewise (\ref{eq:linearFunctional}) and (\ref{eq:logisticFunctional})). Hence, here we take the modulus of the score as a non-linear intermediate step \cite{Bekkers2015EMMCVPR}.}:
\begin{align}
\label{eq:linearFunctionalSE2}
P(\mathbf{x}) = {P}_{lin}^{SE(2)}(\mathbf{x}) := & ( \mathcal{T}_{\mathbf{x}}  \; T , \left|U_{f}\right|)_{\mathbb{L}_2(SE(2))},\;\;\;\;\text{or}\\
\label{eq:logisticFunctionalSE2}
P(\mathbf{x}) = {P}_{log}^{SE(2)}(\mathbf{x}) := & \sigma \left( ( \mathcal{T}_{\mathbf{x}}  \; T , \left|U_{f}\right|)_{\mathbb{L}_2(SE(2))} \right),
\end{align}
%\end{eqnarray}
with $(\mathcal{T}_{\mathbf{x}} T)(\tilde{\mathbf{x}},\tilde{\theta}) = T(\tilde{\mathbf{x}} - \mathbf{x}, \tilde{\theta})$. The $\mathbb{L}_2(SE(2))$-inner product is defined by
%with
%{\color{changescolor}
%\begin{equation}
%\label{eq:linearFunctionalSE2MaxTheta}
%P^{SE(2)}(\mathbf{x}):= \underset{\alpha\in[0,2\pi)}{\operatorname{max}} \; \tilde{P}^{SE(2)}(\mathbf{x},\alpha).
%\end{equation}
%}%
%However, now matching is based on $\mathbb{L}_2(SE(2))$ inner products, and we define the corresponding functionals
%\begin{eqnarray}
%\begin{align}
%\label{eq:linearFunctionalSE2}
%\tilde{P}_{lin}^{SE(2)}(\mathbf{x},\alpha) := & ( \mathcal{L}_{g}  \; T , U_{f})_{\mathbb{L}_2(SE(2))},\\
%\label{eq:logisticFunctionalSE2}
%\tilde{P}_{log}^{SE(2)}(\mathbf{x},\alpha) := & \sigma \left( ( \mathcal{L}_{g}  \; T , U_{f})_{\mathbb{L}_2(SE(2))} \right),
%\end{align}
%%\end{eqnarray}
%with $g = (\mathbf{x},\alpha)$, and with shift-twist operator $\mathcal{L}_g$ (shift by $\mathbf{x}$, rotation by $\alpha$) defined by
%\begin{equation}
%\label{eq:shiftTwist}
%(\mathcal{L}_g T)(\tilde{\mathbf{x}},\tilde{\theta}) = T(\mathbf{R}_\alpha^{-1}(\tilde{\mathbf{x}} - \mathbf{x}), \tilde{\theta} - \alpha).
%\end{equation}
%The $\mathbb{L}_2(SE(2))$-inner product is defined by
\begin{equation}
(T,\left|U_f\right|)_{\mathbb{L}_2 (SE(2))} :=
\int_{\mathbb{R}^2}\int_{0}^{2\pi} \overline{T(\tilde{\mathbf{x}},\tilde{\theta})} \left|U_f\right|(\tilde{\mathbf{x}},\tilde{\theta}) {\rm d}\tilde{\mathbf{x}}{\rm d}\tilde{\theta},
\end{equation}
with norm $\lVert \cdot \rVert_{\mathbb{L}_2(SE(2))} = \sqrt{ (\cdot , \cdot )_{\mathbb{L}_2(SE(2))} }$.
%Also here (cf.~(\ref{eq:mass})) mass $M$ indicates the support of the template, and has the property $\int_{\mathbb{R}^2}\int_{0}^{2\pi} M(\tilde{\mathbf{x}},\tilde{\theta}) {\rm d}\tilde{\mathbf{x}}{\rm d}\tilde{\theta} = 1$. We define
%\begin{equation}
%M(\mathbf{x},\theta) := \frac{1}{2\pi} m(\mathbf{x}),
%\end{equation}
%independent of $\theta$ and with $m(\mathbf{x})$ given by (\ref{eq:mass}).

\subsection{Optimizing $T$ Using Linear Regression}
\label{subsec:linearRegressionSE2}
Following the same reasoning as in Section \ref{subsec:linearRegresionR2} we search for the template that minimizes
\begin{multline}
\label{eq:energySE2LinearRegression}
\mathcal{E}_{lin}(T) = \sum\limits_{i=1}^N \left(  ( T , \left|U_{f_i}\right| )_{\mathbb{L}_2(SE(2))} - y_i \right)^2 \\
+ \lambda \int_{\mathbb{R}^2} \int_0^{2\pi} \lVert \nabla T(\tilde{\mathbf{x}},\tilde{\theta}) \rVert^2_{D}{\rm d}\tilde{\mathbf{x}}{\rm d}\tilde{\theta}
+ \mu \lVert T \rVert^2_{\mathbb{L}_2(SE(2))},
\end{multline}
with smoothing term:
\begin{equation}
\lVert \nabla T (g) \rVert_{D}^2 = D_{\xi\xi} \left| \frac{\partial T}{\partial \xi}(g)\right|^2 + D_{\eta\eta} \left| \frac{\partial T}{\partial \eta}(g)\right|^2 + D_{\theta\theta} \left| \frac{\partial T}{\partial \theta}(g)\right|^2.
\end{equation}
Here, $\nabla T = (\frac{\partial T}{\partial \xi},\frac{\partial T}{\partial \eta},\frac{\partial T}{\partial \theta})^T$ denotes the left-invariant gradient. Note that $\partial_\xi$ gives the spatial derivative in the direction aligned with the orientation score kernel used at layer $\theta$, recall Fig.~\ref{fig:osFrame}. The parameters $D_{\xi\xi}$, $D_{\eta\eta}$ and $D_{\theta\theta} \geq 0$ are then used to balance the regularization in the three directions. Similar to this problem, first order Tikhonov-regularization on $SE(2)$ is related, via temporal Laplace transforms, to left--invariant diffusions on the group $SE(2)$ (Sec.~\ref{sec:stochasticProcess}), in which case $D_{\xi\xi}$, $D_{\eta\eta}$ and $D_{\theta\theta}$ denote the diffusion constants in $\xi$, $\eta$ and $\theta$ direction. Here we set $D_{\xi\xi}=1$, $D_{\eta\eta}=0$, and thereby we get Laplace transforms of hypo-elliptic
diffusion processes \cite{Citti2006,Duits2010}.
%We elaborate more on the similarities between regularization via and left-invariant diffusions on $SE(2)$ in Subsection \ref{subsec:stochasticProcess}.
Parameter $D_{\theta\theta}$ can be used to tune between isotropic (large $D_{\theta\theta}$) and anisotropic (low $D_{\theta\theta}$) diffusion (see e.g. \cite[Fig.~3]{Bekkers2015EMMCVPR}). Note that anisotropic diffusion, via a low $D_{\theta\theta}$, is preferred as we want to maintain line structures in orientation scores.
%We will denote templates that optimize (\ref{eq:linearFunctionalSE2}) by $T_{lin}$.

\subsection{Optimizing $T$ Using Logistic Regression}
\label{subsec:logisticRegressionSE2}
Similarly to what is done in Subsec.~\ref{subsec:logisticRegresionR2} we can change the quadratic loss of (\ref{eq:energySE2LinearRegression}) to a logistic loss, yielding the following energy functional
%\begin{multline}
%\label{eq:energySE2LogisticRegression}
%\mathcal{E}_{log}(T) = \sum\limits_{i=1}^N \left( \sigma \left( ( T , U_{f_i} )_{\mathbb{L}_2(SE(2),M {\rm d}\tilde{g})} \right) - y_i \right)^2 \\
%+ \lambda \int_{\mathbb{R}^2} \int_0^{2\pi} \lVert \nabla T(\tilde{\mathbf{x}},\tilde{\theta}) \rVert^2_{D} M(\tilde{\mathbf{x}},\tilde{\theta}){\rm d}\tilde{\mathbf{x}}{\rm d}\tilde{\theta}\\
%+ \mu \lVert T \rVert^2_{\mathbb{L}_2(SE(2),M {\rm d}\tilde{g})},
%\end{multline}
\begin{multline}
\label{eq:energySE2LogisticRegression}
\mathcal{E}_{log}(T) = \mathscr{L}_{log}(T)
-\lambda \int_{\mathbb{R}^2} \int_0^{2\pi} \lVert \nabla T(\tilde{\mathbf{x}},\tilde{\theta}) \rVert^2_{D} {\rm d}\tilde{\mathbf{x}}{\rm d}\tilde{\theta}\\
- \mu \lVert T \rVert^2_{\mathbb{L}_2(SE(2))},
\end{multline}
with log-likelihood (akin to (\ref{eq:loglikelihood}) for the $\mathbb{R}^2$ case)
\begin{equation}
\begin{array}{rl}
\mathscr{L}_{log}(T) = &\sum\limits_{i=1}^N y_i ( T , \left|U_{f_i}\right| )_{\mathbb{L}_2(SE(2))}\\
&\;\;\;\;\;\;\;- \log\left(
1 + e^{
( T , \left|U_{f_i}\right| )_{   \mathbb{L}_2(  SE(2)  )   }
} \right).
\end{array}
\end{equation}
%We will denote templates that optimize (\ref{eq:energySE2LogisticRegression}) by $T_{log}$.
The optimization of (\ref{eq:energySE2LinearRegression}) and (\ref{eq:energySE2LogisticRegression}) follows quite closely the procedure as described in Sec.~\ref{sec:templateMatchingR2} for the 2D case. In fact, when $T$ is expanded in a B-spline basis, the exact same matrix-vector formulation can be used.

\subsection{Template Optimization in a B-Spline Basis}
\label{subsec:splineBasisSE2}
%\subsubsection{Templates in a B-Spline Basis}
\textbf{\emph{Templates in a B-Spline Basis.}}
The template $T$ is expanded in a B-spline basis as follows
\begin{equation} \label{splineexp}
\begin{array}{l}
T(x,y,\theta) = \sum \limits_{k=1}^{N_k} \sum \limits_{l=1}^{N_l} \sum \limits_{m=1}^{N_m} c_{k,l,m}\cdot \; \\
\\
\;\;\; B^n\!\left( \frac{x}{s_k} - k \right) B^n\!\left( \frac{y}{s_l} - l \right)B^n\!\left( \frac{\theta \!\!\!\!\! \mod 2\pi}{s_m} - m \right),
\end{array}
\end{equation}
with $N_k$, $N_l$ and $N_m$ the number of B-splines in respectively the $x$, $y$ and $\theta$ direction, $c_{k,l,m}$ the corresponding basis coefficients, and with angular resolution parameter $s_m = 2\pi/N_m$.

%\subsubsection{Linear Regression}\label{ch:linearRegressionSE2}
\textbf{\emph{Linear Regression.}}\label{ch:linearRegressionSE2}
%\textbf{Linear Regression.}\label{ch:linearRegressionSE2}
The shape of the minimizer of energy functional $\mathcal{E}_{lin}(T)$ in the $SE(2)$ case is the same as for $E_{lin}(t)$ in the $\mathbb{R}^2$ case, and is again of the form given in (\ref{eq:energyR2LinearRegressionDiscrete}). However, now the definitions of $S$, $R$ and $\mathbf{c}$ are different. Now, $S$ is a $[N \times N_k N_l N_m]$ matrix given by
\begin{equation}
\begin{aligned}
S &= \{(s_{1,1,1}^i,...,s_{1,1,N_m}^i,...,s_{1,N_l,N_m},...,s_{N_k,N_l,N_m}^i)\}_{i=1}^N,\\
s_{k,l,m} &= (\; B_{s_ks_ls_m}^n * U_{f_i} \;)(k,l,m),
\end{aligned}
\end{equation}
with $B_{s_ks_ls_m}^n(x,y,\theta) = B^n\!\left( \frac{x}{s_k} \right)B^n\!\left( \frac{y}{s_l} \right)B^n\!\left( \frac{\theta \!\! \mod 2\pi}{s_m} \right) $. Vector $\mathbf{c}$ is a $[N_k N_l N_m \times 1]$ column vector containing the B-spline coefficients and is stored as follows:
\begin{equation}
%\label{eq:cyArray}
\mathbf{c} = (c_{1,1,1},...,c_{1,1,N_m},...,c_{1,N_l,N_m},...,c_{N_k,N_l,N_m})^T.
\end{equation}
The explicit expression and the derivation of $[N_k N_l N_m \times N_k N_l N_m]$ matrix $R$, which encodes the left invariant derivatives, can be found in the supplementary materials Sec.~2.% in the Appendix.%App.~\ref{app:rmatrix}.
%and these expressions can be inserted in (\ref{eq:minimizerR2LinearRegression}) to get the Euler-Lagrange equation for linear regression on $SE(2)$.

%\subsubsection{Logistic Regression}\label{ch:logSE2}
\textbf{\emph{Logistic Regression}}\label{ch:logSE2}
%\textbf{Logistic Regression.}\label{ch:logSE2}
Also for the logistic regression case we optimize energy functional (\ref{eq:energySE2LogisticRegression}) in the same form as (\ref{eq:energyR2LogisticRegressionLogLikelihood}) in the $\mathbb{R}^2$ case, by using the corresponding expressions for $S$, $R$, and $\mathbf{c}$ in Eq.~(\ref{eq:energyR2LogisticRegressionLogLikelihoodDiscrete}).
%again
These expressions can be inserted in the functional (\ref{eq:energyR2LogisticRegressionLogLikelihoodDiscrete}) and again the same techniques (as presented in Subsection~\ref{ch:log})
can be used to minimize this cost on $SE(2)$.

\subsection{Probabilistic Interpretation of the $SE(2)$ Smoothing Prior}
\label{sec:stochasticProcess}
%\begin{comment}
In this section we only provide a brief introduction to the probabilistic interpretation of the $SE(2)$ smoothing prior, and refer the interested reader to the supplementary materials for full details.
%A classic approach to noise suppression in images is via diffusion regularizations with PDE's of the form
Consider the classic approach to noise suppression in images via diffusion regularizations with PDE's of the form
\begin{equation}
\label{eq:pdeDiffusion}
\left\{
\begin{array}{cl}
\tfrac{\partial}{\partial \tau}u &= \Delta u,\\
u |_{\tau=0} &= u_0,
\end{array}
\right.
\end{equation}
where $\Delta$ denotes the Laplace operator. Solving (\ref{eq:pdeDiffusion}) for any diffusion time $\tau>0$ gives a smoothed version of the input $u_0$. The time-resolvent process of the PDE is defined by the Laplace transform with respect to $\tau$; time $\tau$ is integrated out using a memoryless negative exponential distribution $P(\mathcal{T}=\tau)=\alpha e^{-\alpha \tau}$. Then, the time integrated solutions
$$
t(\mathbf{x}) = \alpha \int_0^\infty u(\mathbf{x},\tau) e^{-\alpha \tau} {\rm d}\tau,
$$
with decay parameter $\alpha$, are in fact the solutions \cite{DuitsBurgeth2007}
\begin{equation}
\label{eq:timeIntegratedDiffusion}
t = \underset{t \in \mathbb{L}_2(\mathbb{R}^2)}{\operatorname{argmin}} \left[ \lVert t - t_0 \rVert^2_{\mathbb{L}_2(\mathbb{R}^2)} + \lambda \int_{\mathbb{R}^2} \lVert \nabla t (\tilde{\mathbf{x}}) \rVert^2 \, {\rm d}\tilde{\mathbf{x}} \right],
\end{equation}
with $\lambda = \alpha^{-1}$. Such time integrated diffusions (Eq.~(\ref{eq:timeIntegratedDiffusion})) can also be obtained by optimization of the linear regression functionals given by Eq.~(\ref{eq:energyR2LinearRegression}) and Eq.~(\ref{eq:linearFunctionalSE2}) for the $\mathbb{R}^2$ and $SE(2)$ case respectively.

In the supplementary materials we establish this connection for the $SE(2)$ case, and show how the smoothing regularizer in (\ref{eq:energySE2LinearRegression}) and (\ref{eq:energySE2LogisticRegression}) relates to Laplace transforms of hypo-elliptic diffusions on the group $SE(2)$ \cite{ZhangDuits2014,Duits2010}.
More precisely, we formulate a special case of our problem (the \emph{single patch problem}) which involves only a single training sample $U_{f_1}$, and show in a formal theorem that the solution is up to scalar multiplication the same as the resolvent hypo-elliptic diffusion kernel. The underlying probabilistic interpretation is that of Brownian motions on $SE(2)$, where the resolvent hypo-elliptic diffusion kernel gives a probability density of finding a random brush stroke at location $\mathbf{x}$ with orientation $\theta$, given that a `drunkman's pencil' starts at the origin at time zero.

In the supplementary materials we demonstrate the high accuracy of our discrete numeric regression method using B-spline expansions with near to perfect comparisons to the continuous exact solutions of the single patch problem.
%\footnote{In fact, we have established an efficient B-spline finite element implementation of hypo-elliptic Brownian motions on $SE(2)$, in addition to other numerical approaches in \cite{ZhangDuits2014}.}.
In fact, we have established an efficient B-spline finite element implementation of hypo-elliptic Brownian motions on $SE(2)$, in addition to other numerical approaches in \cite{ZhangDuits2014}.

%In particular, the supplementary materials show that our smoothed regression method (Subsec.~\ref{ch:linearRegressionSE2}) provides a new algorithm that can be used for resolvent hypo-elliptic diffusion process on $SE(2)$. This is underpinned with experimental results where we compute the resolvent hypo-elliptic diffusion kernel, which shows near to perfect comparisons to the exact solution found in \cite{ZhangDuits2014,Duits2010}, see Fig.~\ref{fig:stochasticEnhancement} and the supplementary material, available online.

\begin{comment}
\begin{figure}
\centerline{
\includegraphics[width=\hsize]{Figs/FigKernel2.pdf}
}
\caption{
Isosurface of the time resolvent hypo-elliptic diffusion kernel on $SE(2)$, computed via the methods proposed in this paper (single patch problem, see supplementary materials, available online) and the exact solution. Right most figure is an illustration of the drunkman's pencil.
%kernel computed by solving the fundamental single patch problem (\ref{min}), the exact solution, and an illustration of the drunkman's pencil. For Monte Carlo simulations of the drunkman's pencil see the supplementary materials.% of this paper.
\label{fig:stochasticEnhancement}}
\end{figure}
\end{comment}
%\end{comment}

%\section{Probabilistic Interpretation of the Smoothing Prior in $SE(2)$}
%\label{sec:stochasticProcess}
%\input{4_Regularizer}

%\section{Normalized Cross-Correlation}
%\label{sec:normalizedCC}
%\input{4_NormalizedCC}

\section{Applications}
\label{sec:applications}

Our applications of interest are in retinal image analysis. In this section we establish and validate an algorithm pipeline for the detection of the optic nerve head (Subsec. \ref{subsec:ONHDetection}) and fovea (Subsec. \ref{subsec:FoveaDetection}) in retinal images, and the pupil (Subsec.~\ref{subsec:PupilDetection}) in regular camera images. Before we proceed to the application sections, we first describe the experimental set-up (Subsec.~\ref{subsec:detailsExperiments}). All experiments discussed in this section are reproducible; the data (with annotations) as well as the full code (Wolfram \emph{Mathematica} notebooks) used in the experiments are made available at: \url{http://erikbekkers.bitbucket.org/TMSE2.html}.
In the upcoming sections we only report the most relevant experimental results. More details on each application (examples of training samples, implementation details, a discussion on parameter settings, computation times, and examples of successful/failed detections) are provided in the supplementary materials.

%To demonstrate the broader applicability of our method, we also investigate the application to upright human detection in city scenes (Subsec. \ref{subsec:humanDetection}).

\subsection{The experimental set-up}
\label{subsec:detailsExperiments}
%\subsubsection{Templates}
\textbf{\emph{Templates.}}
\label{subsubsec:templateDetails}
In our experiments we compare the performance of different template types,
%indicated with the following labels:
which we label as follows:
\begin{itemize}
\item[$A$:] Templates obtained by taking the average of all positive patches ($y_i=1$) in the training set,
%, and normalizing it via (\ref{eq:templateNormalizationR2}) and (\ref{eq:normalizationSE2b}).
then normalized to zero mean and unit standard deviation.
\item[$B$:] Templates optimized without any regularization.% ($\mu = \lambda = 0$).
\item[$C$:] Templates optimized with an optimal $\mu$, and with $\lambda = 0$.
\item[$D$:] Templates optimized with an optimal $\lambda$ and with $\mu = 0$.
\item[$E$:] Templates optimized with optimal $\mu$ and $\lambda$.
\end{itemize}
The trained templates ($B$-$E$) are obtained either via linear or logistic regression in the $\mathbb{R}^2$ setting (see Subsec. \ref{ch:linearregression} and Subsec. \ref{ch:log}), or in the $SE(2)$ setting (see Subsec. \ref{ch:linearRegressionSE2} and Subsec. \ref{ch:logSE2}). In both the $\mathbb{R}^2$ and $SE(2)$ case, linear regression based templates are indicated with subscript ${}_{lin}$, and logistic regression based templates with ${}_{log}$. Optimality of parameter values is defined using generalized cross validation (GCV), which we soon explain in this section. We generally found that (via optimization using GCV) the optimal settings for template $E$ were $\mu \approx 0.5\mu^*$, and $\lambda \approx 0.5\lambda^*$, with $\mu^*$ and $\lambda^*$ respectively the optimal parameters for template $C$ and $D$.

%\subsubsection{Matching with Multiple Templates}
\textbf{\emph{Matching with Multiple Templates.}}
When performing template matching, we use Eq.~(\ref{eq:linearFunctional}) and Eq.~(\ref{eq:linearFunctionalSE2}) for respectively the $\mathbb{R}^2$ and $SE(2)$ case for templates obtained via linear regression and for template $A$. For templates obtained via logistic regression we use respectively Eq.~(\ref{eq:logisticFunctional}) and Eq.~(\ref{eq:logisticFunctionalSE2}). When we combine multiple templates we simply add the objective functionals. E.g, when combining template $C_{lin:\mathbb{R}^2}$ and $D_{log:SE(2)}$ we solve the problem
$$
\mathbf{x}^* = \underset{\mathbf{x} \in \mathbb{R}^2 }{\operatorname{argmax}} \;\;P_{C_{lin}}^{\mathbb{R}^2}(\mathbf{x}) + P_{D_{log}}^{SE(2)}(\mathbf{x}),
$$
where $P_{C_{lin}}^{\mathbb{R}^2}(\mathbf{x})$ is the objective functional (see Eq.~(\ref{eq:linearFunctional})) obtained with template $C_{lin:\mathbb{R}^2}$, and $P_{D_{log}}^{SE(2)}(\mathbf{x})$ (see Eq.~(\ref{eq:logisticFunctionalSE2})) is obtained with template $D_{log:SE(2)}$.

%\subsubsection{Rotation and Scale Invariance}
\textbf{\emph{Rotation and Scale Invariance.}}
The proposed template matching scheme can adapted for rotation-scale invariant matching, this is discussed in Sec.~5 of the supplementary materials. For a generic object recognition task, however, global rotation or scale invariance are not necessarily desired properties. Datasets often contain objects in a human environment context, in which some objects tend to appear in specific orientations (e.g. eye-brows are often horizontal above the eye and vascular trees in the retina depart the ONH typically along a vertical axis).
%Discarding those hints by including rotation and scale invariance a rotation invariant typically (likewise the scattering approach in \cite[ch:6.3.1]{siffre}) has an adversary effect on performances in the training phase, while increasing the computational load. See supplementary material where this can be observed.
%However, in the applications discussed in this manuscript the object of interest typically appears at a fixed orientation and scale (a retinal image is always obtained under a fixed orientation, and in pupil detection the person behind the camera is typically in upright position at an approximately fixed distance).
Discarding such knowledge by introducing rotation/scale invariance is likely to have an adversary effect on the performance, while increasing computational load.
%In such situations it is not desirable to introduce rotation or scale invariance.
In Sec.~5 of the supplementary materials we tested a rotation/scale invariant adaptation of our method and show that in the three discussed applications this did indeed not lead to improved results, but in fact worsened the results slightly.

%\subsubsection{Automatic Parameter Selection via Generalized Cross Validation}
\textbf{\emph{Automatic Parameter Selection via Generalized Cross Validation.}}
\label{subsubsec:GCV}
An ideal template generalizes well to new data samples, meaning that it has low prediction error on independent data samples. One method to predict how well the system generalizes to new data is via generalized cross validation (GCV), which is essentially an approximation of leave-one-out cross validation \cite{Craven1979}.
%Let $\tilde{y}_i = \mathbf{s}_i^\dagger \mathbf{c}_{\mu,\lambda}$ denote the predicted label of a given input sample $\mathbf{s}_i$ (a row in S). Here $\mathbf{c}_{\mu,\lambda}$ denotes the minimizer of (\ref{eq:energyR2LinearRegressionDiscrete}) given by the solution of (\ref{eq:minimizerR2LinearRegression}).
The vector containing all predictions is given by $\tilde{\mathbf{y}} = S \mathbf{c}_{\mu,\lambda}$, in which we can substitute the solution for $\mathbf{c}_{\mu,\lambda}$ (from Eq.~(\ref{eq:minimizerR2LinearRegression})) to obtain
%a linear relation between $\mathbf{y}$ and $\tilde{\mathbf{y}}$:
\begin{equation}
\begin{aligned}
\tilde{\mathbf{y}} &= A_{\mu,\lambda} \mathbf{y}, \;\;\;\;\;\;\;\;\;\;\;\;\text{with}\\
A_{\mu,\lambda}&= S (S^\dagger S + \lambda R + \mu I)^{-1} S^\dagger,
\end{aligned}
\end{equation}
where $A_{\mu,\lambda}$ is the so-called \emph{smoother matrix}. Then the generalized cross validation value \cite{Craven1979} is defined as
\begin{equation}
\label{eq:GCV}
GCV(\mu,\lambda) \equiv  \frac{ \frac{1}{N} \lVert \Omega  (I - A_{\mu,\lambda} ) \mathbf{y} \rVert^2}{ \left(1 - \operatorname{trace} (A_{\mu,\lambda} )/N\right )^2}.
\end{equation}
In the retinal imaging applications we set $\Omega = I$. In the pupil detection application we set $\Omega = \underset{i\in\{1,...,N\}}{\operatorname{diag}}\{y_i\}$.
% where $\omega_i$ weights the importance of each sample.
%The numerator in Eq.~(\ref{eq:GCV}) corresponds to the average prediction error (on the training set), $\operatorname{trace} A_{\mu,\lambda} $ in the denominator is often referred to as the effective degrees-of-freedom \cite{HastieBook}.
%We experimentally found that $\omega_i = y_i$,
As such, we do not penalize errors on negative samples as here the diversity of negative patches is too large for parameter optimization via GCV.
%, resulted in best generalization of the templates to unseen data. A reason for this could be that in unseen data the positive instances often better resemble those in the training set than that unseen negative instances resemble the negative training samples. Unseen negative samples could be anything and it turns out that it is better not to try to generalize the template to deal with such cases.
%The computation of $A_{\mu,\lambda}$ involves the inversion of a large matrix, however, this inversion can be done efficiently via diagonalization using eigen-decomposition in the case that either $\mu$ or $\lambda$ is zero, see e.g. \cite[ch. 5.4]{HastieBook}.% and \cite{Golub1979}.
Parameter settings are considered optimal when they minimize the GCV value.
%To reduce computation time, we perform a line search for either $\mu$ or $\lambda$, whilst keeping the other parameter at zero. E.g., for template $D$ we compute the GCV value for different values of $\lambda$, while $\mu=0$. The result is a GCV curve with its minimum at $\lambda = \lambda^*$.
%, see Fig.~\ref{fig:GCVcurve}.
%The optimal parameters for templates $C$ and $D$ are denoted with $\mu^*$ and $\lambda^*$. For template $E$ we set $\mu = 0.5 \mu^*$, $\lambda = 0.5 \lambda^*$.

In literature various extensions of GCV are proposed for generalized linear models \cite{O'Sullivan1986,Gu1992,Xiang1996}. For logistic regression we use the approach by O'Sullivan et al. \cite{O'Sullivan1986}: we iterate the Newton-Raphson algorithm until convergence, then, at the final iteration we compute the GCV value on the quadratic approximation (Eq.~(\ref{eq:updateRule})).

\begin{comment}
\begin{figure}[h]
\begin{center}
\includegraphics[width=\linewidth]{Figs/FigGCV.pdf}
\end{center}
\caption{Parameter optimization for the $SE(2)$ ONH template $D_{lin}$. The GCV value (Eq.~(\ref{eq:GCV}) is minimal for $\lambda = \lambda^* = 5.62*10^{-7}$, a lower value for $\lambda$ results in an over-fitted (under-smooth) template (left figure), a higher value results in an under-fitted (over-smooth) template. The template images are maximum intensity projections over $\theta$.}
\label{fig:GCVcurve}
\end{figure}
\end{comment}

\begin{comment}
While our final goal is to optimize the detection scheme (Eq.~..). In the template optimization (Eq.~.. and Eq.~..) however, we do not optimizer for succesrates, but rather the corresponding energy functionals. To tune the parameters, we chose $\lambda$, $D_{\theta\theta}$ and $\mu$ by minimizing the classification error $\frac{1}{N} \lVert S \mathbf{c} - \mathbf{y} \rVert$ (cf. the data term in Eq.~..). E.g.
\begin{equation}
\lambda^* = \underset{\lambda \in \mathbb{R} }{\operatorname{argmin}} \;\; \frac{1}{N} \lVert S^{test} \mathbf{c}^\lambda - \mathbf{y}^{test} \rVert,
\end{equation}
where $\mathbf{c}^\lambda$ is the solution obtain from Eq.~(..) with the specified $\lambda$.
\end{comment}

%\subsubsection{Success Rates}
\textbf{\emph{Success Rates.}}
%\textbf{Success Rates.}
Performance of the templates is evaluated using success rates. The success rate of a template is the percentage of images in which the target object was successfully localized.
%correct localizations of the object of interest in each image.
%In every image there is one object to detect, in Subsec.~\ref{subsec:ONHDetection} this is the optic nerve head, in Subsec.~\ref{subsec:FoveaDetection} this is the fovea.
In both optic nerve head (Subsec.~\ref{subsec:ONHDetection}) and fovea (Subsec.~\ref{subsec:FoveaDetection}) detection experiments, a successful detection is defined as such if the detected location $\mathbf{x}^*$ (Eq.~(\ref{eq:objectDetection})) lies within one optic disk radius distance to the actual location.
%Details hereof are given inS Subsec.~\ref{subsec:ONHDetection} and Subsec.~\ref{subsec:FoveaDetection} for the ONH and fovea application respectively.
For pupil detection both the left and right eye need to be detected and we therefore use the following normalized error metric
\begin{equation}
\label{eq:normalizedError}
e=\frac{\operatorname{max}(d_{left},d_{right})}{w},
\end{equation}
in which $w$ is the (ground truth) distance between the left and right eye, and $d_{left}$ and $d_{right}$ are respectively the distances of detection locations to the left and right eye. %In Fig.~\ref{fig:PupilTemplates}(a) a normalized error $e$ of 0.1 is illustrated with blue circles.

%\subsubsection{k-Fold Cross Validation}
\textbf{\emph{k-Fold Cross Validation.}}
%\textbf{k-Fold Cross Validation.}
For correct unbiased evaluation, none of the test images are used for training of the templates, nor are they used for parameter optimization. We perform $k$-fold cross validation: The complete dataset is randomly partitioned into $k$ subsets. Training (patch extraction, parameter optimization and template construction) is done using the data from $k-1$ subsets. Template matching is then performed on the remaining subset. This is done for all $k$ configurations with $k-1$ training subsets and one test subset, allowing us to compute the average performance (success rate) and standard deviation. We set $k=5$.

%\subsubsection{The Training Samples $\hat{f}_i$}
%\textbf{The Training Samples $\hat{f}_i$.}
%Positive training samples $\hat{f}_i$ are defined as patches of dimensions $N_x \times N_y$ centered around true location of the object in the image. For every image, a negative sample is defined as an image patch centered around random location in the image that does not lie within a certain distance to the true object location. For both the ONH and fovea detection application this distance is the radius of the optic nerve head. An example of a positive ONH patch is given in Fig.~\ref{fig:odOS}.

%ERROR BOUNDING BOX
\begin{comment}
\begin{figure}[b]
\begin{center}
\includegraphics[width=\linewidth]{Figs/FigPatches.png}
\end{center}
\caption{Example images patches $f_i$ used in template optimization for optic nerve head detection. Top row positive samples ($y_i = 1$), bottom row negative samples ($y_i = 0$).}
\label{fig:samples}
\end{figure}
\end{comment}

\subsection{Optic Nerve Head Detection in Retinal Images}
\label{subsec:ONHDetection}

\begin{figure*}
\begin{center}
\includegraphics[width=\linewidth]{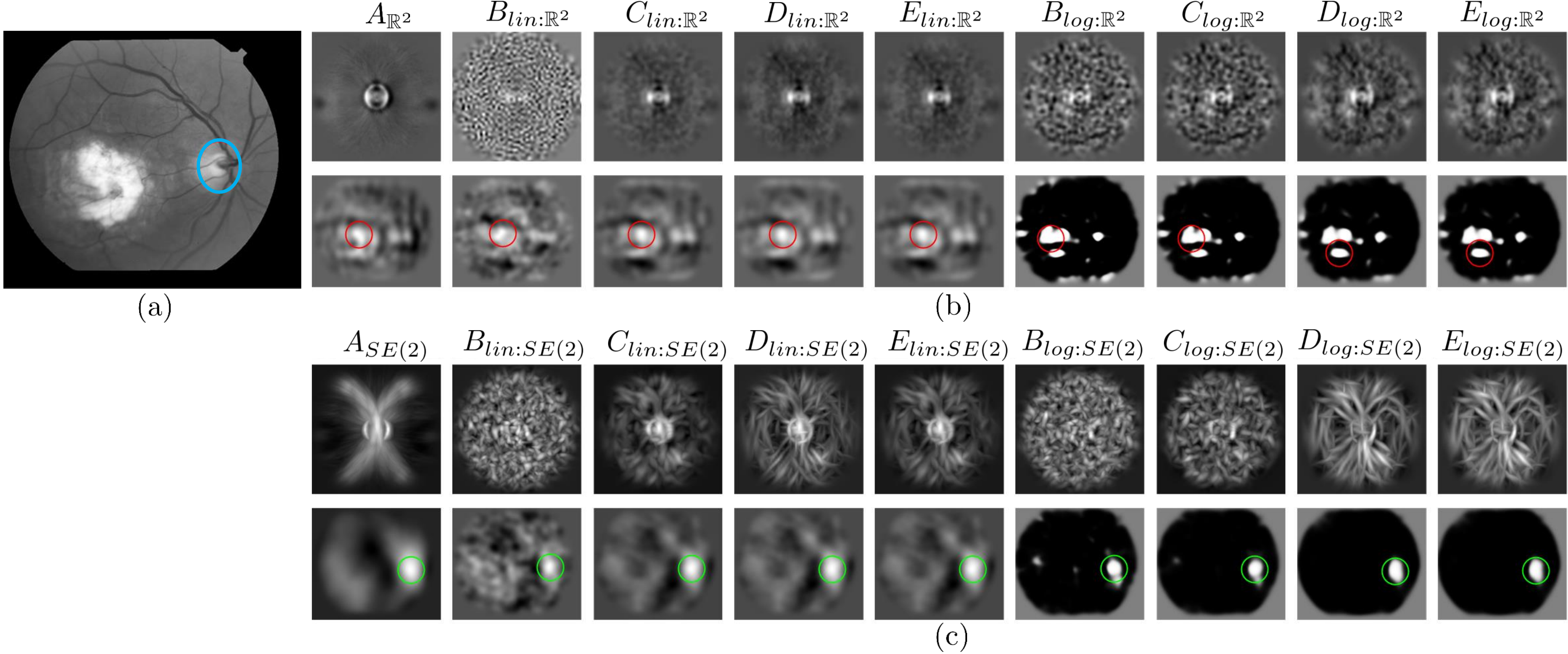}
\end{center}
\caption{Overview of trained templates for ONH detection, and their responses to a challenging retinal image. \textbf{(a)} The example input image with true ONH location in blue. \textbf{(b)} The $\mathbb{R}^2$-type templates (top row) and their responses to the input image (bottom row). \textbf{(c)} The maximum intensity projections (over $\theta$) of the $SE(2)$-type templates (top row) and their responses to the input image (bottom row). Detected ONH locations are indicated with colored circles (green = correct, red = incorrect).
}
\label{fig:ODTemplates}
\end{figure*}

%\settoheight{\mylength}{some text}
\newlength{\smallspacing}
\setlength{\smallspacing}{1.5mm}
\definecolor{rowcolor}{gray}{0.85}

\begin{table*}%[!htbp]
\centering
\caption{Average template matching results ($\pm$ standard deviation) for optic nerve head detection in 5-fold cross validation, number of failed detections in parentheses.}
\begin{tabular}{l|lllll|l}
\toprule
%\multicolumn{2}{c|}{}& \multicolumn{1}{c|}{SLO} & \multicolumn{4}{c|}{CF} & \multicolumn{1}{}{}\\
%\multicolumn{2}{c|}{} & \multicolumn{1}{c|}{} & \multicolumn{4}{c|}{} & \multicolumn{1}{}{}\\
\multicolumn{1}{l|}{Template} & ES (SLO) & TC & MESSIDOR & DRIVE & STARE & All Images\\
ID & 208 & 208 & 1200 & 40 & 81 & 1737\\

\midrule
\multicolumn{7}{c}{{$\mathbb{R}^2$ templates}}\\
\midrule
  $A_{\mathbb{R}^2}$ & 100.0\% {\tiny $\pm$ 0.00\%} (0)     & 99.49\% {\tiny $\pm$ 1.15\%} (1)     & 98.83\% {\tiny $\pm$ 0.56\%} (14)   & 96.36\% {\tiny $\pm$ 4.98\%} (2)        & 74.94\% {\tiny $\pm$ \;\;9.42\%} (20)    & \cellcolor{rowcolor}97.87\% {\tiny $\pm$ 0.52\%} (37) \vspace{\smallspacing}\\

  $B_{lin:\mathbb{R}^2}$ & 99.09\% {\tiny $\pm$ 2.03\%} (2)     & 20.35\% {\tiny $\pm$ 5.99\%} (165)     & \;\;9.67\% {\tiny $\pm$ 2.69\%} (1084)   & \;\;9.09\% {\tiny $\pm$ 12.86\%} \!\!(35)        & \;\;3.56\% {\tiny $\pm$ \;\;3.28\%} (78)    & 21.48\% {\tiny $\pm$ 2.16\%} (1364) \\

  $C_{lin:\mathbb{R}^2}$ & 99.55\% {\tiny $\pm$ 1.02\%} (1)     & 99.57\% {\tiny $\pm$ 0.97\%} (1)     & 98.33\% {\tiny $\pm$ 0.41\%} (20)   & 94.55\% {\tiny $\pm$ 8.13\%} (3)        & 66.96\% {\tiny $\pm$ \;\;16.65\%} \!\!(26)    & 97.07\% {\tiny $\pm$ 0.76\%} (51) \\

  $D_{lin:\mathbb{R}^2}$ & 99.55\% {\tiny $\pm$ 1.02\%} (1)     & 99.57\% {\tiny $\pm$ 0.97\%} (1)     & 98.42\% {\tiny $\pm$ 0.45\%} (19)   & 96.36\% {\tiny $\pm$ 4.98\%} (2)        & 67.53\% {\tiny $\pm$ \;\;17.80\%} \!\!(25)    & \cellcolor{rowcolor}97.24\% {\tiny $\pm$ 0.72\%} (48) \\

  $E_{lin:\mathbb{R}^2}$ & 99.55\% {\tiny $\pm$ 1.02\%} (1)     & 99.57\% {\tiny $\pm$ 0.97\%} (1)     & 98.33\% {\tiny $\pm$ 0.29\%} (20)   & 96.36\% {\tiny $\pm$ 4.98\%} (2)        & 66.90\% {\tiny $\pm$ \;\;19.25\%} \!\!(26)    & 97.12\% {\tiny $\pm$ 0.84\%} (50) \vspace{\smallspacing}\\

  $B_{log:\mathbb{R}^2}$ & \;\;4.36\% {\tiny $\pm$ 3.21\%} (199)     & \;\;4.59\% {\tiny $\pm$ 6.41\%} (199)     & \;\;3.17\% {\tiny $\pm$ 0.86\%} (1162)   & \;\;1.82\% {\tiny $\pm$ 4.07\%} (39)        & \;\;3.64\% {\tiny $\pm$ \;\;8.13\%} (79)    & \;\;3.40\% {\tiny $\pm$ 0.74\%} (1678) \\

  $C_{log:\mathbb{R}^2}$ & 68.69\% {\tiny $\pm$ 6.24\%} (65)     & 98.10\% {\tiny $\pm$ 2.00\%} (4)     & 97.75\% {\tiny $\pm$ 1.01\%} (27)   & 96.36\% {\tiny $\pm$ 4.98\%} (2)        & 66.94\% {\tiny $\pm$ \;\;16.43\%} \!\!(28)    & \cellcolor{rowcolor}92.74\% {\tiny $\pm$ 0.65\%} (126) \\

  $D_{log:\mathbb{R}^2}$ & 41.87\% {\tiny $\pm$ 6.81\%} (121)     & 97.60\% {\tiny $\pm$ 1.82\%} (5)     & 96.00\% {\tiny $\pm$ 1.59\%} (48)   & 91.01\% {\tiny $\pm$ 8.46\%} (4)        & 65.30\% {\tiny $\pm$ \;\;10.05\%} \!\!(28)    & 88.14\% {\tiny $\pm$ 1.21\%} (206) \\

  $E_{log:\mathbb{R}^2}$ & 58.68\% {\tiny $\pm$ 4.48\%} (86)     & 97.59\% {\tiny $\pm$ 2.48\%} (5)     & 97.33\% {\tiny $\pm$ 0.96\%} (32)   & 93.51\% {\tiny $\pm$ 9.00\%} (3)        & 67.88\% {\tiny $\pm$ \;\;12.61\%} \!\!(27)    & 91.20\% {\tiny $\pm$ 0.95\%} (153) \\

\midrule
\multicolumn{7}{c}{{$SE(2)$ templates}}\\
\midrule

  $A_{SE(2)}$ & 98.57\% {\tiny $\pm$ 2.16\%} (3)     & 98.95\% {\tiny $\pm$ 2.35\%} (2)     & 99.58\% {\tiny $\pm$ 0.30\%} (5)   & 98.18\% {\tiny $\pm$ 4.07\%} (1)        & 94.22\% {\tiny $\pm$ \;\;9.64\%} (5)    & \cellcolor{rowcolor}99.08\% {\tiny $\pm$ 0.75\%} (16)  \vspace{\smallspacing}\\

  $B_{lin:SE(2)}$ & 99.06\% {\tiny $\pm$ 1.29\%} (2)     & 94.75\% {\tiny $\pm$ 2.48\%} (11)     & 93.74\% {\tiny $\pm$ 1.80\%} (75)   & 92.05\% {\tiny $\pm$ 7.95\%} (4)        & 85.63\% {\tiny $\pm$ \;\;10.97\%} \!\!(12)    & 94.01\% {\tiny $\pm$ 0.89\%} (104) \\

  $C_{lin:SE(2)}$ & 99.06\% {\tiny $\pm$ 1.29\%} (2)     & 100.0\% {\tiny $\pm$ 0.00\%} (0)     & 100.0\% {\tiny $\pm$ 0.00\%} (0)   & 97.50\% {\tiny $\pm$ 5.59\%} (1)        & 94.00\% {\tiny $\pm$ \;\;6.17\%} (5)    & \cellcolor{rowcolor}99.54\% {\tiny $\pm$ 0.39\%} (8) \\

  $D_{lin:SE(2)}$ & 98.60\% {\tiny $\pm$ 2.05\%} (3)     & 100.0\% {\tiny $\pm$ 0.00\%} (0)     & 99.67\% {\tiny $\pm$ 0.46\%} (4)   & 100.0\% {\tiny $\pm$ 0.00\%} (0)        & 94.00\% {\tiny $\pm$ \;\;6.17\%} (5)    & 99.31\% {\tiny $\pm$ 0.44\%} (12) \\

  $E_{lin:SE(2)}$ & 98.60\% {\tiny $\pm$ 2.05\%} (3)     & 100.0\% {\tiny $\pm$ 0.00\%} (0)     & 99.67\% {\tiny $\pm$ 0.46\%} (4)   & 97.50\% {\tiny $\pm$ 5.59\%} (1)        & 95.11\% {\tiny $\pm$ \;\;5.48\%} (4)    & 99.31\% {\tiny $\pm$ 0.33\%} (12)  \vspace{\smallspacing}\\

  $B_{log:SE(2)}$ & 87.06\% {\tiny $\pm$ 4.20\%} (27)     & 77.68\% {\tiny $\pm$ 5.36\%} (46)     & 84.17\% {\tiny $\pm$ 2.25\%} (190)   & 80.19\% {\tiny $\pm$ 14.87\%} (9)        & 75.10\% {\tiny $\pm$ \;\;9.81\%} (21)    & 83.14\% {\tiny $\pm$ 1.78\%} (293) \\

  $C_{log:SE(2)}$ & 97.66\% {\tiny $\pm$ 2.79\%} (5)     & 99.52\% {\tiny $\pm$ 1.06\%} (1)     & 99.58\% {\tiny $\pm$ 0.42\%} (5)   & 98.18\% {\tiny $\pm$ 4.07\%} (1)        & 95.33\% {\tiny $\pm$ \;\;7.30\%} (4)    & \cellcolor{rowcolor}99.08\% {\tiny $\pm$ 0.13\%} (16) \\

  $D_{log:SE(2)}$ & 95.22\% {\tiny $\pm$ 3.78\%} (10)     & 98.50\% {\tiny $\pm$ 2.27\%} (3)     & 99.25\% {\tiny $\pm$ 0.19\%} (9)   & 98.18\% {\tiny $\pm$ 4.07\%} (1)        & 95.33\% {\tiny $\pm$ \;\;4.74\%} (4)    & 98.45\% {\tiny $\pm$ 0.38\%} (27) \\

  $E_{log:SE(2)}$ & 97.14\% {\tiny $\pm$ 2.61\%} (6)     & 99.52\% {\tiny $\pm$ 1.06\%} (1)     & 99.50\% {\tiny $\pm$ 0.35\%} (6)   & 98.18\% {\tiny $\pm$ 4.07\%} (1)        & 94.22\% {\tiny $\pm$ \;\;6.82\%} (5)    & 98.90\% {\tiny $\pm$ 0.48\%} (19) \\

\midrule
\multicolumn{7}{c}{{Template combinations (sorted on performance)}}\\
\midrule

  \tiny $A_{\mathbb{R}^2} \;\;\;\;\;\;\;\;\;\;\;\, + C_{log:SE(2)}$ & 100.0\% {\tiny $\pm$ 0.00\%} (0)     & 100.0\% {\tiny $\pm$ 0.00\%} (0)     & 99.92\% {\tiny $\pm$ 0.19\%} (1)   & 98.18\% {\tiny $\pm$ 4.07\%} (1)        & 98.67\% {\tiny $\pm$ \;\;2.98\%} (1)    & 99.83\% {\tiny $\pm$ 0.26\%} (3) \\

  \tiny $A_{\mathbb{R}^2} \;\;\;\;\;\;\;\;\;\;\;\, + E_{log:SE(2)}$ & 100.0\% {\tiny $\pm$ 0.00\%} (0)     & 100.0\% {\tiny $\pm$ 0.00\%} (0)     & 99.83\% {\tiny $\pm$ 0.23\%} (2)   & 98.18\% {\tiny $\pm$ 4.07\%} (1)        & 98.67\% {\tiny $\pm$ \;\;2.98\%} (1)    & 99.77\% {\tiny $\pm$ 0.24\%} (4) \\

  \tiny $A_{\mathbb{R}^2} \;\;\;\;\;\;\;\;\;\;\;\, + D_{log:SE(2)}$ & 100.0\% {\tiny $\pm$ 0.00\%} (0)     & 100.0\% {\tiny $\pm$ 0.00\%} (0)     & 99.83\% {\tiny $\pm$ 0.23\%} (2)   & 98.18\% {\tiny $\pm$ 4.07\%} (1)        & 98.67\% {\tiny $\pm$ \;\;2.98\%} (1)    & 99.77\% {\tiny $\pm$ 0.24\%} (4) \\

  \tiny $C_{lin:SE(2)} + E_{log:SE(2)}$ & 99.55\% {\tiny $\pm$ 1.02\%} (1)     & 100.0\% {\tiny $\pm$ 0.00\%} (0)     & 99.83\% {\tiny $\pm$ 0.23\%} (2)   & 100.0\% {\tiny $\pm$ 0.00\%} (0)        & 96.44\% {\tiny $\pm$ \;\;3.28\%} (3)    & 99.65\% {\tiny $\pm$ 0.13\%} (6) \\

  \tiny $C_{lin:SE(2)} + C_{log:SE(2)}$ & 99.55\% {\tiny $\pm$ 1.02\%} (1)     & 100.0\% {\tiny $\pm$ 0.00\%} (0)     & 99.92\% {\tiny $\pm$ 0.19\%} (1)   & 98.18\% {\tiny $\pm$ 4.07\%} (1)        & 96.44\% {\tiny $\pm$ \;\;3.28\%} (3)    & 99.65\% {\tiny $\pm$ 0.13\%} (6) \\

  \multicolumn{1}{l|}{{\;\;\;\;\;\;\;\;\;\;\;\;\;\;\;\;\,...}} & \multicolumn{5}{c|}{{...}} & \multicolumn{1}{c}{{...}}\\

  \hspace{-0.6em}$^*$\tiny $A_{SE(2)} \;\;\;\;\;\; +C_{lin:SE(2)}$ & 99.55\% {\tiny $\pm$ 1.02\%} (1)     & 100.0\% {\tiny $\pm$ 0.00\%} (0)     & 100.0\% {\tiny $\pm$ 0.00\%} (0)   & 98.18\% {\tiny $\pm$ 4.07\%} (1)        & 94.22\% {\tiny $\pm$ \;\;6.82\%} (5)    & 99.60\% {\tiny $\pm$ 0.26\%} (7) \\

  \multicolumn{1}{l|}{{\;\;\;\;\;\;\;\;\;\;\;\;\;\;\;\;\,...}} & \multicolumn{5}{c|}{{...}} & \multicolumn{1}{c}{{...}}\\

  \hspace{-0.5em}$^\dagger$\tiny $A_{\mathbb{R}^2} \;\;\;\;\;\;\;\;\;\;\;\; +A_{SE(2)}$ & 100.0\% {\tiny $\pm$ 0.00\%} (0)     & 100.0\% {\tiny $\pm$ 0.00\%} (0)     & 99.66\% {\tiny $\pm$ 0.35\%} (4)   & 98.18\% {\tiny $\pm$ 4.07\%} (1)        & 88.42\% {\tiny $\pm$ \;\;11.23\%} (9)    & 99.19\% {\tiny $\pm$ 0.63\%} (14) \\

  \multicolumn{1}{l|}{{\;\;\;\;\;\;\;\;\;\;\;\;\;\;\;\;\,...}} & \multicolumn{5}{c|}{{...}} & \multicolumn{1}{c}{{...}}\\
\bottomrule
\multicolumn{7}{l}{$^*$\emph{Best template combination that does not rely on logistic regression.} $^\dagger$\emph{Best template combination that does not rely on template optimization.}}
\end{tabular}
%\begin{flushleft}
%{$^*$\emph{Best template combination that does not rely on logistic regression.} $^\dagger$ \emph{Best template combination that does not rely on template optimization.}}
%\end{flushleft}
\label{tab:resultsONH}
\end{table*}

Our first application to retinal images is optic nerve head detection. The ONH is one of the key anatomical landmarks in the retina, and its location is often used as a reference point to define regions of interest for the analysis of the retina. The detection hereof is therefore an essential step in many automated retinal image analysis pipelines.
%Moreover, in other applications the ONH itself (its shape or appearance) is the main topic of research (e.g. in glaucoma research), in which case its robust identification is crucial.

The ONH has two main characteristics: 1) it often appears as a bright disk-like structure on color fundus (CF) images (dark on SLO images), and 2) it is the place from which blood vessels leave the retina. Traditionally, methods have mainly focused on the first characteristic \cite{LuLim2011,Aquino2012,Dashtbozorg2015}. However, in case of bad contrast of the optic disk, or in the presence of pathology (especially bright lesions, see e.g. Fig.~\ref{fig:ODTemplates}), these methods typically fail. Most of the recent ONH detection methods therefore also include the vessel patterns in the analysis; either via explicit vessel segmentation \cite{Sekhar2011,Marin2015}, vessel density measures \cite{Yu2012,Giachetti2013}, or via additional orientation pattern matching steps \cite{Youssif2008}. In our method, both the appearance and vessel characteristics are addressed in an efficient integrated template matching approach, resulting in state-of-the-art performance both in terms of success rates and computation time. We target the first characteristic with template matching on $\mathbb{R}^2$. The second is targeted with template matching on $SE(2)$.

\subsubsection{Processing Pipeline \& Data}
\textbf{\emph{Processing Pipeline.}}
\label{subsubsec:processingPipeline}
%\textbf{Processing Pipeline.}
%Our ONH detection pipeline is based on template matching (and construction) in the standard cross-correlation setting (Sec.~\ref{sec:templateMatchingR2} and Sec.~\ref{sec:templateMatchingSE2}), and we deal with local contrast and illumination variations using the normalization method from \cite{Foracchia2005}, rather than with normalized cross-correlations (Sec.~\ref{sec:normalizedCC}). The proposed cross-correlation based pipeline was $\pm$ 4 times faster than an alternative normalized cross-correlation based pipeline that we investigated. The results for both pipelines were similar, and we therefore chose to continue with the cross-correlation based pipeline, which we describe next.
%
%Full details on preprocessing are described in the supplementary materials Sec.~\ref{4.2}.
First, the images are rescaled to a working resolution of 40 $\mu m/pix$. In our experiments the average resolution per dataset was determined using the average optic disk diameter (which is on average $1.84 mm$). The images are normalized for contrast and illumination variations using the method from \cite{Foracchia2005}. Finally, in order to put more emphasis on contextual/shape information, rather than pixel intensities, we apply a soft binarization to the locally normalized (cf.~Eq.~(31) in Ch.~3 of the supplementary materials) image ${f}$  via the mapping $\operatorname{erf}(8 {f})$.

For the orientation score transform we use $N_\theta = 12$ uniformly sampled orientations from $0$ to $\pi$ and lift the image using cake wavelets \cite{Duits2007a,Bekkers2014}. For phase-invariant, nonlinear, left-invariant \cite{Duits2010}, and contractive \cite{Bruna2013} processing on SE(2), we work with the modulus of the complex valued orientation scores rather than with the complex-valued scores themselves (taking the modulus of quadrature filter responses is an effective technique for line detection, see e.g. Freeman et al. \cite{Freeman1991}).

Due to differences in image characteristics, training and matching is done separately for the SLO and the color fundus images. For SLO images we use the near infrared channel, for RGB fundus images we use the green channel.

Positive training samples $f_i$ are defined as $N_x \times N_y$ patches, with  $N_x=N_y=251$, centered around true ONH location in each image. For every image, a negative sample is defined as an image patch centered around random location in the image that does not lie within one optic disk radius distance to the true ONH location. An exemplary ONH patch is given in Fig.~\ref{fig:odOS}. For the B-spline expansion of the templates we set $N_k = N_l = 51$ and $N_m = 12$.

%\subsubsection{Data}
\textbf{\emph{Data.}}
%\textbf{Data.}
\label{subsubsec:dataONH}
In our experiments we made use of both publicly available data, and a private database. The private database consists of 208 SLO images taken with an EasyScan fundus camera (i-Optics B.V., the Netherlands) and 208 CF images taken with a Topcon NW200 (Topcon Corp., Japan). Both cameras were used to image both eyes of the same patient, taking an ONH centered image, and a fovea centered image per eye. The two sets of images are labeled as "ES" and "TC" respectively. The following (widely used) public databases are also used: MESSIDOR (\url{http://messidor.crihan.fr/index-en.php}), DRIVE (\url{http://www.isi.uu.nl/Research/Databases/DRIVE}) and STARE (\url{http://www.ces.clemson.edu/~ahoover/stare}), consisting of 1200, 40 and 81 images respectively. For each image, the circumference of the ONH was annotated, and parameterized by an ellipse. The annotations for the MESSIDOR dataset were kindly made available by the authors of \cite{Aquino2010} (\url{http://www.uhu.es/retinopathy}). The ONH contour in the remaining images were manually outlined by ourselves. The annotations are made available on our website. The images in the databases contain a mix of good quality healthy images, and challenging diabetic retinopathy cases. Especially MESSIDOR and STARE contain challenging images.

\subsubsection{Results and Discussion}
%\textbf{Results and Discussion.}
\emph{\textbf{The templates.}} The different templates for ONH detection are visualized in Fig.~\ref{fig:ODTemplates}. The $SE(2)$ templates are visualized using maximum intensity projections over $\theta$. In this figure we have also shown template responses to an example image. Visually one can clearly recognize the typical disk shape in the $\mathbb{R}^2$ templates, whereas the $SE(2)$ templates also seem to capture the typical pattern of outward radiating blood vessels (compare e.g. $A_{\mathbb{R}^2}$ with $A_{SE(2)}$). Indeed, when applied to a retinal image, where we took an example with an optic disk like pathology, we see that the $\mathbb{R}^2$ templates respond well to the disk shape, but also (more strongly) to the pathology. In contrast, the $SE(2)$ templates respond mainly to vessel pattern and ignore the pathology. We also see, as expected, a smoothing effect of gradient based regularization ($D$ and $E$) in comparison to standard $\mathbb{L}_2$-norm regularization ($C$) and no regularization ($B$). %Visually we can not recognize clear differences in performance between the different types of regularization.
Finally, in comparison to linear regression templates, the logistic regression templates have a more binary response due to the logistic sigmoid mapping.

\begin{table}
\caption{Comparison to state of the art: Optic nerve head detection success rates, the number of fails (in parentheses), and computation times.}
\centering
\begin{tabular}{llll|l}
\toprule
Method                                              & MESSIDOR    & DRIVE       & STARE       & Time (s)      \\
\midrule
Lu {\tiny \cite{Lu2011} }                           & 99.8\% (3)  &             & 98.8\% (1)  & \;\;\;\;5.0   \\
Lu {\tiny et al. \cite{LuLim2011} }                 &             & 97.5\% (1)  & 96.3\% (3)  & \;\;40.0      \\
Yu {\tiny et al. \cite{Yu2012} }                    & 99.1\% (11) &             &             & \;\;\;\;4.7   \\
Aquino {\tiny et al. \cite{Aquino2012} }            & 99.8\% (14) &             &             & \;\;\;\;1.7   \\
Giachetti {\tiny et al. \cite{Giachetti2013} }      & 99.7\% (4)  &             &             & \;\;\;\;5.0   \\
Ramakanth {\tiny et al. \cite{Ramakanth2014} }      & 99.4\% (7)  & 100\% (0)   & 93.83\% (5) & \;\;\;\;0.2   \\
Marin {\tiny et al. \cite{Marin2015}}               & 99.8\% (3)  &             &             & \;\;\;\;5.4$^\dagger$  \\
Dashtbozorg {\tiny et al. \cite{Dashtbozorg2015}}   & 99.8\% (3)  &             &             & \;\;10.6$^\dagger$  \vspace{\smallspacing}\\
Proposed                                            & 99.9\% (1)  & 97.8\%  (1) & 98.8\%  (1) & \;\;\;\;0.5   \\
\bottomrule
\multicolumn{5}{l}{$^\dagger$\emph{Timings include simultaneous disk segmentation.}}\\
\end{tabular}%
\label{tab:stateOfTheArtONH}
\end{table}

\emph{\textbf{Detection results.}} Table \ref{tab:resultsONH} gives a breakdown of the quantitative results for the different databases used in the experiments. The templates are grouped in $\mathbb{R}^2$ templates, $SE(2)$ templates, and combination of templates. Within these groups, they are further divided in average, linear regression, and logistic regression templates. The best overall performance within each group is highlighted in gray.

Overall, we see that the $SE(2)$ templates out-perform their $\mathbb{R}^2$ equivalents, and that combinations of the two types of templates give best results. The two types are nicely complementary to each other due to disk-like sensitivity of the $\mathbb{R}^2$ templates and the vessel pattern sensitivity of the $SE(2)$ templates. If one of the two ONH characteristics is less obvious (as is e.g. for the disk-shape in Fig.~\ref{fig:ODTemplates}), the other can still be detected. Also, the failures of $\mathbb{R}^2$ templates are mainly due to either distracting pathologies in the retina, or poor contrast of the optic disk. As reflected by the increased performance of $SE(2)$ templates over $\mathbb{R}^2$ templates, a more stable pattern seems to be the vessel pattern.

From Table \ref{tab:resultsONH} we also deduce that the individual performances of the linear regression templates outperform the logistic regression templates. Moreover, the average templates give best individual performance, which indicates that with our effective template matching framework good performance can already be achieved with basic templates.
%, which indicates that good performance can already be achieved with our basic template matching framework, and the construction of basic average templates.
However, we also see that low performing individual templates can prove useful when combining templates.
% it is not the best performing individual templates that work well in the combination of templates.
In fact, we see that combinations with all linear $\mathbb{R}^2$ templates are highly ranked, and for the $SE(2)$ templates it is mainly the logistic regression templates. This can be explained by the binary nature of the logistic templates: even when the maximum response of the templates is at an incorrect location, the difference with the correct location is often small. The $\mathbb{R}^2$ template then adds to the sensitivity and precision.
%, as these response are often more fine and detailed.
%We also observe that in the combination of templates smooth templates ($D$ and $E$) are preferred.
%For reference we also reported the best results obtained with untrained templates, which is $99.19\%$ (14 fails). With the best performing template combination we achieved a $99.83\%$ success rate (3 fails).% of the overall best template combination.
The best results obtained with untrained templates was a $99.19\%$ success rate (14 fails), and with the overall best template combination we obtained a $99.83\%$ success rate (3 fails).

\emph{\textbf{State of the art.}} In Table \ref{tab:stateOfTheArtONH} we compare our results on the publicly available benchmark databases MESSIDOR, DRIVE and STARE, with the most recent methods for ONH detection (sorted from oldest to newest from top to bottom). In this comparison, our best performing method ($A_{\mathbb{R}^2}+C_{log:SE(2)}$) performs better than or equally well as the best methods from literature. We have also listed the computation times, and see that our method is also ranked as one of the fastest methods for ONH detection. The average computation time, using our experimental implementation in Wolfram \emph{Mathematica} 10.4, was $0.5$ seconds per image on a computer with an Intel Core i703612QM CPU and 8GB memory.
%Here we note that in the retinal image datasets the maximum template response always occurs at rotation $\alpha = 0$ {\color{changescolor2} at the object location}, so for the sake of reduced computation time we have set $P^{SE(2)}(\mathbf{x}) := \tilde{P}^{SE(2)}(\mathbf{x},0)$ instead of (\ref{eq:linearFunctionalSE2MaxTheta}).
A full breakdown of timings of the processing pipeline is given in the supplementary materials Sec.~4.

\begin{figure*}
\begin{center}
\includegraphics[width=\linewidth]{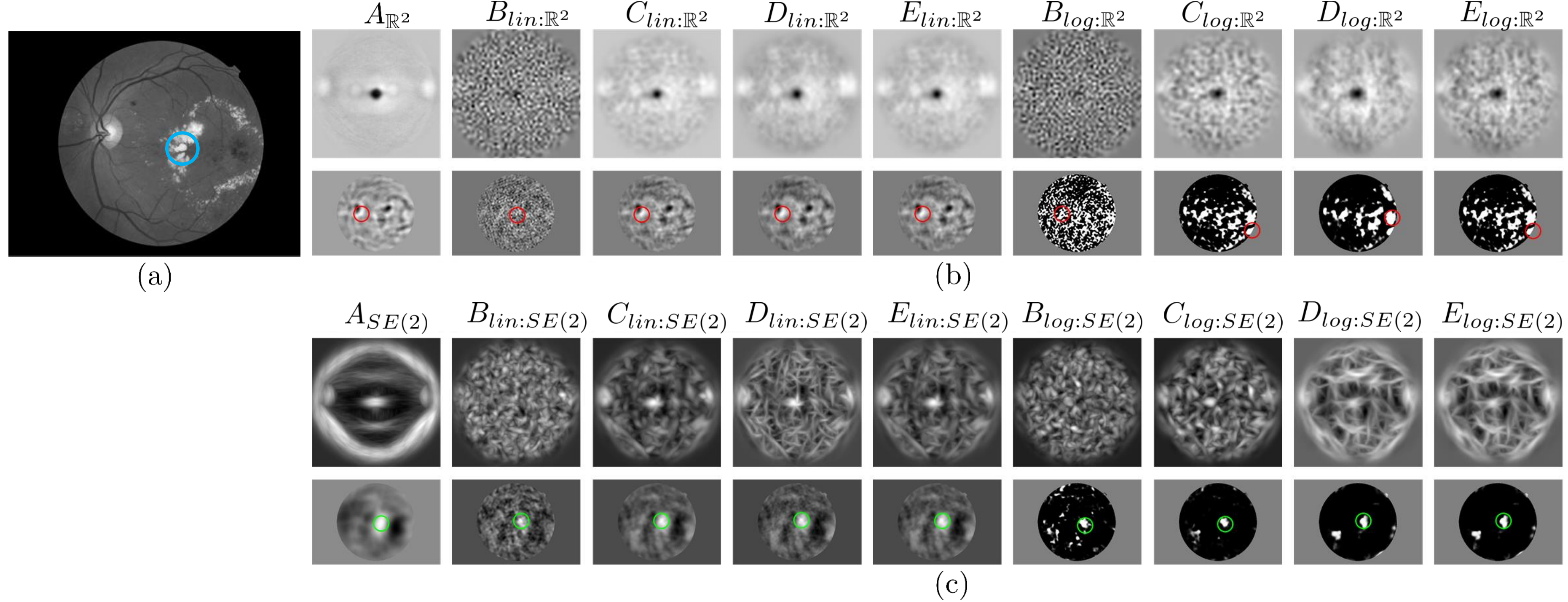}
\end{center}
\caption{Overview of trained templates for fovea detection, and their responses to a challenging retinal image.  {\textbf{(a)}} The example input image with true fovea location in blue.  \textbf{(b)} The $\mathbb{R}^2$-type templates (top row) and their responses to the input image (bottom row).  \textbf{(c)} The maximum intensity projections (over $\theta$) of the $SE(2)$-type templates (top row) and their responses to the input image (bottom row). Detected fovea locations are indicated with colored circles (green = correct, red = incorrect).
}
\label{fig:FoveaTemplates}
\end{figure*}

\subsection{Fovea Detection in Retinal Images}
\label{subsec:FoveaDetection}

Our second application to retinal images is for the detection of the fovea. The fovea is the location in the retina which is responsible for sharp central vision.
%, and its surrounding region has the highest concentration of cones (color-sensitive photoreceptors).
It is characterized by a small depression in thickness of the retina, and on healthy retinal images it often appears as a darkened area. Since the foveal area is responsible for detailed vision, this area is weighted most heavily in grading schemes that describe the severity of a disease.
%E.g., pathologies within this region (see e.g. Figs.~\ref{fig:ODTemplates} and \ref{fig:FoveaTemplates}) more severely affect vision than pathology present in the more peripheral areas of the retina.
Therefore, correct localization of the fovea is essential in automatic grading systems \cite{Abramoff2015}.

Methods for the detection of the fovea heavily rely on contextual features in the retina \cite{Aquino2014,Giachetti2013,GegundezArias2013,Yu2011,Niemeijer2009}, and take into account the prior knowledge that 1) the fovea is located approximately 2.5 optic disk diameters lateral to the ONH center, that 2) it lies within an avascular zone, and that 3) it is surrounded by the main vessel arcades. All of these methods restrict their search region for the fovea location to a region relative to the (automatically detected) ONH location. To the best of our knowledge, the proposed detection pipeline is the first that is completely independent of vessel segmentations and ONH detection. This is made possible due to the fact that anatomical reference patterns, in particular the vessel structures, are generically incorporated in the learned templates via data representations in orientation scores. %Similar to the ONH detection application (Subsec.~\ref{subsec:ONHDetection}), we combine $\mathbb{R}^2$ templates with $SE(2)$ templates to achieve state-of-the-art performance.
%The fact that this is achieved independent of the ONH detection shows the generic applicability of our template matching approach via smoothed (hypo-elliptic) regressions on $SE(2)$.

%\subsubsection{Processing Pipeline}
\subsubsection{Processing Pipeline \& Data}
\textbf{\emph{Processing Pipeline.}}
The proposed fovea detection pipeline is the same as for ONH detection, however, now the positive training samples $f_i$ are centered around the fovea.

%\subsubsection{Data}
\textbf{\emph{Data.}}
The proposed fovea detection method is validated on our (annotated) databases ``ES'' and ``TC'', each consisting of 208 SLO and 208 color fundus images respectively (cf. Subsec.\ref{subsubsec:dataONH}). We further test our method on the most used publicly available benchmark dataset MESSIDOR (1200 images). Success rates were computed based on the fovea annotations kindly made available by the authors of \cite{GegundezArias2013}.

\subsubsection{Results and Discussion}
%\textbf{Results and Discussion.}
\emph{\textbf{The templates.}} Akin to Fig.~\ref{fig:ODTemplates}, in Fig.~\ref{fig:FoveaTemplates} the trained fovea templates and their responses to an input image are visualized. The $\mathbb{R}^2$ templates seem to be more tuned towards the dark (isotropic) blob like appearance of the fovea, whereas in the $SE(2)$ templates one can also recognize the pattern of vessels surrounding the fovea (compare $A_{\mathbb{R}^2}$ with $A_{SE(2)}$). To illustrate the difference between these type of templates, we selected an image in which the fovea location is occluded with bright lesions. In this case the method has to rely on contextual information (e.g. the blood vessels). Indeed, we see that the $\mathbb{R}^2$ templates fail due to the absence of a clear foveal blob shape, and that the $SE(2)$ templates correctly identify the fovea location. The effect of regularization is also clearly visible; no regularization ($B$) results in noisy templates, standard $\mathbb{L}_2$ regularization ($C$) results in more stable templates, and smoothed regularization ($D$ and $E$) results in smooth templates. %In particular,
In templates $D_{SE(2)}$ and $E_{SE(2)}$ we see that more emphasis is put on line structures.% (due to the left-invariant hypo-elliptic smoothing).

\emph{\textbf{Detection results.}} A full overview of individual and combined template performance is discussed in the supplementary materials, here we only provide a summary. Again there is an improvement using $SE(2)$ templates over $\mathbb{R}^2$ templates, although the difference is smaller than in the ONH application.
%Whereas in the ONH detection application there was a clear improvement of using $SE(2)$ templates over $\mathbb{R}^2$ templates, the improvement in this application is less obvious.
Apparently both the dark blob-like appearance ($\mathbb{R}^2$ templates) and vessel patterns ($SE(2)$ templates) are equally reliable features of the fovea. A combination of templates leads to improved results and we conclude that the templates are again complementary to each other. Furthermore, again linear regression performs better than logistic regression.
In fovea detection
%We also see that in this application the logistic regression templates are less stable than the linear regression templates. In contrast to the ONH detection application, here
we do observe a large improvement of template training over basic averaging: 1529 of 1616 ($94.6\%$) successful detections with $C_{lin:SE(2)}$ versus 1488 ($92.1\%$) with $A_{SE(2)}$. The best performing $\mathbb{R}^2$ template was $A_{\mathbb{R}^2}$ ($65.6\%$), the best $SE(2)$ template was $C_{lin:SE(2)}$ ($94.6\%$). The best combination of templates was $C_{lin:\mathbb{R}^2}+C_{log:SE(2)}$ with 1605 ($99.3\%$) detections. When using non-optimized templates %, i.e. templates obtained via averaging,
1588 ($98.3\%$) successful detections were achieved (with $A_{\mathbb{R}^2}+A_{SE(2)}$).

\emph{\textbf{State of the art.}} In Table \ref{tab:stateOfTheArtFovea} we compared our results on the publicly available benchmark database MESSIDOR with the most recent methods for fovea detection (sorted from oldest to newest from top to bottom). In this comparison, our best performing method ($C_{lin:\mathbb{R}^2}+C_{log:SE(2)}$) quite significantly outperforms the best methods from literature. Furthermore, our detection pipeline is also the most efficient one; the computation time for fovea detection is the same as for ONH detection, which is $0.5$ seconds.

\begin{table}
\caption{Comparison to state of the art: Fovea detection success rates, the number of fails (in parentheses), and computation times.}
\centering
\begin{tabular}{ll|l}
\toprule
Method                                                            & MESSIDOR  & Time (s)   \\
\midrule
Niemeijer {\tiny et al. \cite{Niemeijer2009,GegundezArias2013} }  &   97.9\% \;\;(25)   & \;\;\;\;7.6$^{\dagger}$\\
Yu {\tiny et al. \cite{Yu2011} }                                  &   95.0\%$^*$ (60)   & \;\;\;\;3.9$^{\dagger}$\\
Gegundez-Arias {\tiny et al. \cite{GegundezArias2013}}            &   96.9\% \;\;(37)   & \;\;\;\;0.9   \\
Giachetti {\tiny et al. \cite{Giachetti2013}}                     &   99.1\% \;\;(11)   & \;\;\;\;5.0$^{\dagger}$  \\
Aquino {\tiny \cite{Aquino2014}}                                  &   98.2\% \;\;(21)   & \;\;10.9$^{\dagger}$  \vspace{\smallspacing}\\
Proposed                                                          &   99.7\% \;\;(3)    & \;\;\;\;0.5   \\
\bottomrule
\multicolumn{3}{l}{$^*$\emph{Success-criterion based on half optic radius.}}\\
%\multicolumn{3}{l}{$^*$\emph{Successful if distance to target is less than 0.5 optic}}\\
%\multicolumn{3}{l}{\;\;\emph{disk radius. }}\\
\multicolumn{3}{l}{$^\dagger$\emph{Timing includes ONH detection.}}\\
\end{tabular}%
\label{tab:stateOfTheArtFovea}
\end{table}

\begin{figure*}
\begin{center}
\includegraphics[width=\linewidth]{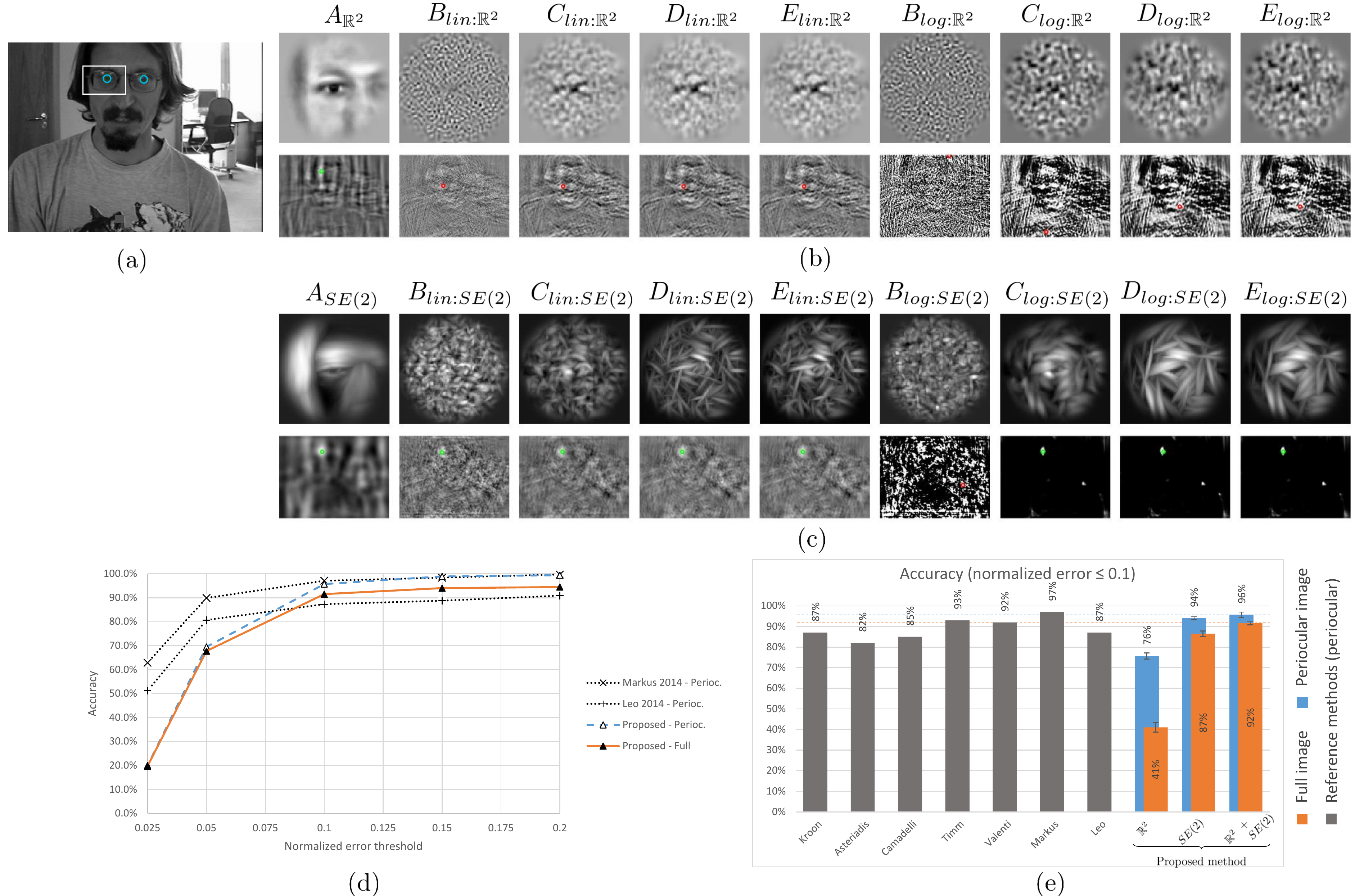}
\end{center}
\caption{Overview of trained templates for right-eye pupil detection, and their responses to a challenging image from the BioID database. \textbf{(a)} The example input image with true pupil locations (blue circle with a radius that corresponds to a normalized error threshold of 0.1, see Eq.~\ref{eq:normalizedError}. The white square indicates the periocular image region for the right eye. \textbf{(b)} The $\mathbb{R}^2$-type templates (top row) and their responses to the input image (bottom row). \textbf{(c)} The maximum intensity projections (over $\theta$) of the $SE(2)$-type templates (top row) and their responses to the input image (bottom row). Detected pupil locations are indicated with colored circles (green = correct, red = incorrect, based on a normalized error threshold of 0.1). \textbf{(d)} Accuracy curves generated by varying thresholds on the normalized error, in comparison with the two most recent methods from literature. \textbf{(e)} Accuracy (at a normalized error threshold of 0.1) comparison with pupil detection methods from literature.
}
\label{fig:PupilTemplates}
\end{figure*}

\subsection{Pupil Detection}
\label{subsec:PupilDetection}
Our third application is that of pupil localization in regular camera images, which is relevant in many applications as they provide important visual cues for face detection, face recognition, and understanding of facial expressions. In particular in gaze estimation the accurate localization of the pupil is essential. Eye detection and tracking is however challenging due to, amongst others: occlusion by the eyelids and variability in size, shape, reflectivity or head pose.

%\cite{Kroon2008,Asteriadis2009,Campadelli2009,Timm2011,Valenti2012,Markus2014,Leo2014}

Many pupil localization algorithms are designed to work on periocular images, these are close-up views of the eyes. Such images can be acquired by dedicated eye imaging devices, or by means of cropping a full facial image (see Fig.~\ref{fig:PupilTemplates}(a)).
%In addition to pupil detection in periocular images, we also consider the more difficult problem of detection in full images.
We will consider both the problem of detection pupils in periocular images and the more difficult problem of detection in full images.
%Here we also consider the more difficult problem of pupil detection in full images, rather than the periocular images.

We compare our method against the seven most recent pupil detection methods from literature, for a full overview see \cite{Leo2014} and \cite{Markus2014}. A method similar to our $\mathbb{R}^2$ approach in the sense that it is also based on 2D linear filtering is the method by Kroon et al. \cite{Kroon2008}. In their paper templates are obtained via linear discriminant analysis of pupil images. Asteriada et al. \cite{Asteriadis2009} detect the pupil by matching templates using features that are based on distances to the nearest strong (facial) edges in the image. Campadelli et al. \cite{Campadelli2009} use a supervised approach with a SVM classifier and Haar wavelet features. The method by Timm et al. \cite{Timm2011} is based on searching for gradient fields with a circular symmetry. Valenti et al. \cite{Valenti2012} use a similar approach but additionally include information of isophote curvature, with supervised refinement. Markus et al. \cite{Markus2014} employ a supervised approach using an ensemble of randomized regression trees. Leo et al. \cite{Leo2014} employ a completely unsupervised approach similar to those in \cite{Timm2011,Valenti2012}, but additionally include analysis of self-similarity.% \cite{Maver2010}.

A relevant remark is that all of the above mentioned methods rely on prior face detection, and restrict their search region to periocular images. Our method works completely stand alone, and can be used on full images.

\subsubsection{Processing Pipeline \& Data}
\textbf{\emph{Processing Pipeline.}}
Interestingly, we could again employ the same processing pipeline (including local normalization via \cite{Foracchia2005}) which was used for ONH and fovea detection. In our experiments we train templates for the left and right eye separately. %The patch sizes are $N_x x N_y] = [101,101]$ and we set the number of B-splines to $N_k=N_l=51$, and $N_m=12$.

%\subsubsection{Data}
\textbf{\emph{Data.}}
We validated our pupil detection approach on the publicly available BioID database (\url{http://www.bioid.com}), which is generally considered as one of the most challenging and realistic databases for pupil detection in facial images. The database consists of $1521$ frontal face grayscale images with significant variation in illumination, scale and pose.

\subsubsection{Results and Discussion}
%\textbf{Results and Discussion.}
\emph{\textbf{The templates.}} Fig.~\ref{fig:PupilTemplates}(b) and (c) show respectively the trained $\mathbb{R}^2$ and $SE(2)$ templates for pupil detection of the right eye, and their filtering response to the input image in Fig.~\ref{fig:PupilTemplates}(a). Here the trained $\mathbb{R}^2$ templates seemed to capture the pupil as a small blob in the center of the template, but apart from that no real structure can be observed. In the average template we do however clearly see structure in the form of an ``average face''. The $SE(2)$ templates reveal structures that resemble the eyelids in nearly all templates. The linear regression templates look sharper and seem to contain more detail than the average template, and the logistic regression templates seem to take a good compromise between course features and details.

\emph{\textbf{Detection results.}} We again refer to the supplementary materials for a full benchmarking analysis, in summary we observed the following. In terms of success rates we see a similar pattern as with the ONH and fovea application, however, here we see that the learned templates ($C$,$D$ and $E$) significantly outperform the average templates, and that logistic regression leads to better templates than using linear regression (94.0\% success rate for $C_{log:SE(2)}$ vs 87.2\% for $D_{lin:SE(2)}$). Overall, the $SE(2)$ templates outperform the $\mathbb{R}^2$ templates, linear regression templates outperform the average template, and logistic regression templates outperform linear regression templates.
The best $\mathbb{R}^2$ template was $D_{lin:\mathbb{R}^2}$ with 1151 of 1521 detections ($75.7\%$), the best $SE(2)$ template was $C_{log:SE(2)}$ ($94.0\%$). The best combination of templates was $D_{lin:\mathbb{R}^2}$ with $E_{lin:SE(2)}$ ($95.6\%$). Without template training (i.e., using average templates $A$) the performance was only $68.2\%$. Success rates %(under condition $e\leq0.1$, recall (\ref{eq:normalizedError}))
using the best template combination are given in Fig.~\ref{fig:PupilTemplates}(d) and (e). The processing time for detecting both pupils simultaneously was on average 0.4 seconds per image.%A successful (combined left and right) detection was considered as such if the normalized error $e$ was smaller than a certain threshold (varied along the horizontal axis in Fig.~\ref{fig:PupilTemplates}(d), and set at 0.1 in Fig.~\ref{fig:PupilTemplates}(e)).

\emph{\textbf{State of the art.}} In Fig.~\ref{fig:PupilTemplates}(d) we compared our approach to the two most recent pupil detection methods from literature for several normalized error thresholds. Here we see that with allowed errors of $0.1$ (blue circles Fig.~\ref{fig:PupilTemplates}(a)) and higher our method competes very well with the state of the art, despite the fact that our generic method is not adapted to the application. Further application specific tuning and preprocessing could be applied to improve precision (for $e \ll 0.1$), but this is beyond the scope of this article. Moreover, we see that our method can be used on full images instead of the periocular images without much loss in performance.
%On the other hand, for more precise localization criteria our method under-performs with respect to the reference methods. This could be possibly addressed by means of parameter tuning which we did not do in this application, and which is outside the scope of this paper. This indicates that our approach is perhaps less suitable for applications like gaze tracking (which requires very precise localization).
The fact that our method is still very accurate on full image processing
%, considering standard accuracy requirements ($e\leq0.1$, see also Fig.~\ref{fig:PupilTemplates}(a)),
shows that it can be used as a preprocessing step for other applications.

If Fig.~\ref{fig:PupilTemplates}(e) we compared our approach to the seven most recent methods from literature (sorted from old to new). Here we see that the only method outperforming our method, at standard accuracy requirements ($e\leq0.1$), is the method by Markus et al. \cite{Markus2014}. Even when considering processing of the full images the only other method that outperforms ours is the method by Timm et al. \cite{Timm2011}, whose performance is measured using periocular images.

\subsection{General Observations}
\label{subsec:GeneralDiscussion}
The application of our method to the three problems (ONH, fovea and pupil detection) showed the following:
\begin{enumerate}
\item State-of-the-art performance was achieved on three different applications, using a single (generic) detection framework and without application specific parameter adaptations.
\item Cross correlation based template matching via data representations on $SE(2)$ improves results over standard $\mathbb{R}^2$ filtering.
\item Trained templates, obtained using energy functionals of the form (\ref{eq:energyGeneric}), often perform better than basic average templates. In particular in pupil detection the optimization of templates proved to be essential.
%\item Our newly introduced logistic regression approach leads to improved results. In particular in the combination of the templates the soft-binarization due to the logistic sigmoid leads to more robust filtering. In particular in the pupil detection application logistic regression produced better results.
\item Our newly introduced logistic regression approach leads to improved results in pupil detection via single templates. When combining templates we observe only a small improvement of choosing logistic regression (instead of linear regression) for the application of ONH and fovea detection.%, which we expect to be due to the soft-binarization effect of the logistic sigmoid mapping.}
\item Regularization in both linear and logistic regression is important. Here both ridge and smoothing regularization priors have complementary benefits.
\item Our method does not rely on any other detection systems (such as ONH detection in the fovea application, or face detection in the pupil detection), and still performs well compared to methods that do.
\item Our method is fast and parallelizable as it is based on inner products, as such it could be efficiently implemented using convolutions.
\end{enumerate}

\color{black}

\section{Conclusion}
\label{sec:discussionAndConclusion}

%\subsection{Discussion}
%\label{subsec:discussion}
%\subsubsection{Parameters}
%\subsubsection{Speed}
%\subsubsection{Results}
%\subsubsection{Possible Extensions}

%Use of left-invariant $SE(2)$ regularization in other variational problems.

%\subsection{Conclusion}
%\label{subsec:conclusion}

%\emph{\textbf{Discussion}}

%\emph{\textbf{Conclusion.}} 
In this paper we have presented an efficient cross-correlation based template matching scheme for the detection of combined orientation and blob patterns. Furthermore, we have provided a generalized regression framework for the construction of templates. The method relies on data representations in orientation scores, which are functions on the Lie group $SE(2)$, and we have provided the tools for proper smoothing priors via resolvent hypo-elliptic diffusion processes on $SE(2)$ (solving time-integrated hypo-elliptic Brownian motions on $SE(2)$). The strength of the method was demonstrated with two applications in retinal image analysis (the detection of the optic nerve head (ONH), and the detection of the fovea) and additional experiments for pupil detection in regular camera images. In the retinal applications we achieved state-of-the-art results with an average detection rate of $99.83\%$ on $1737$ images for ONH detection, and $99.32\%$ on $1616$ images for fovea detection.
%, improving performances by $12$ other state-of-the-art methods. 
Also on pupil detection we obtained state-of-the-art performance with a $95.86\%$ success rate on 1521 images. We showed that the success of the method is due to the inclusion of both intensity and orientation features in template matching. The method is also computationally efficient as it is entirely based on a sequence of convolutions (which can be efficiently done using fast Fourier transforms). These convolutions are parallelizable, which can further speed up our already fast experimental \emph{Mathematica} implementations that are publicly available at \url{http://erikbekkers.bitbucket.org/TMSE2.html}. In future work we plan to investigate the applicability of smoothing on $SE(2)$ in variational settings, as this could also be used in (sparse) line enhancement and segmentation problems.

%%%%%%%%%%%%%%%%%%%%%%%%%%%%%%%%%%%%%%%%%%%%%
%%%%%%%%%%%%%%% Appendices
%%%%%%%%%%%%%%%%%%%%%%%%%%%%%%%%%%%%%%%%%%%%%

%\appendices
%\section{Explicit Expression of Matrix $R$}
%\label{app:rmatrix}
%\input{A1_Appendix}

% use section* for acknowledgment
\ifCLASSOPTIONcompsoc
  % The Computer Society usually uses the plural form
  \section*{Acknowledgments}
\else
  % regular IEEE prefers the singular form
  \section*{Acknowledgment}
\fi

The authors would like to thank the groups that kindly made available the benchmark datasets and annotations.
%In particular, we would like to thank Manuel E. Geg\'{u}ndez for the 64 MESSIDOR fovea annotations which were not publicly available.
The authors gratefully acknowledge Gonzalo Sanguinetti (TU/e) for fruitful discussions and feedback on this manuscript. The research leading to the results of this article has received funding from the European Research Council under the European Community's 7th Framework Programme (FP7/2007–2014)/ERC grant agreement No. 335555. This work is also part of the H\'{e} Programme of Innovation Cooperation, which is (partly) financed by the Netherlands Organisation for Scientific Research (NWO).

\clearpage
%\clearpage

\appendices

%%%%%%%%%%%%%%%%%%%%%%%%%%%%%%%%%%%%%%%%%%%%%
%%%%%%%%%%%%%%% Main article
%%%%%%%%%%%%%%%%%%%%%%%%%%%%%%%%%%%%%%%%%%%%%

\section{Probabilistic Interpretation of the Smoothing Prior in $SE(2)$}%e^{\tau \Delta_{SE(2)}}: Brownian motion for Contour Enhancement}
\label{sec:stochasticProcess}

In this section we relate the $SE(2)$ smoothing prior to time resolvent hypo-elliptic\footnote{This diffusion process on $SE(2)$ is called \emph{hypo-elliptic} as its generator equals $(\partial_\xi)^2 + D_{\theta\theta}(\partial_\theta)^2$ and diffuses only in 2 directions in a 3D space. This boils down to a sub-Riemannian manifold structure \cite{Citti2006,ZhangDuits2014}. Smoothing in the missing $(\partial_\eta)$ direction is achieved via the commutator: $[ \partial_\theta, \cos \theta \partial_x + \sin \theta \partial_y ] = -\sin \theta \partial_x + \cos \theta \partial_y$.} diffusion processes on $SE(2)$. First we aim to familiarize the reader with the concept of resolvent diffusions on $\mathbb{R}^2$ in Subsec.~\ref{subsec:resolventDiffusions}. Then we pose in Subsec.~\ref{subsec:toyProblem} a new problem (the single patch problem), which is a special case of our $SE(2)$ linear regression problem, that we use to link the left-invariant regularizer to tue resolvents of hypo-elliptic diffusions on $SE(2)$.

\subsection{Resolvent Diffusion Processes}
\label{subsec:resolventDiffusions}
A classic approach to noise suppression in images is via diffusion regularizations with PDE's of the form \cite{DuitsBurgeth2007}
\begin{equation}
\label{eq:pdeDiffusion}
\left\{
\begin{array}{cl}
\tfrac{\partial}{\partial \tau}u &= \Delta u,\\
u |_{\tau=0} &= u_0,
\end{array}
\right.
\end{equation}
where $\Delta$ denotes the Laplace operator. Solving (\ref{eq:pdeDiffusion}) for any diffusion time $\tau>0$ gives a smoothed version of the input $u_0$. The time-resolvent process of the PDE is defined by the Laplace transform with respect to $\tau$; time $\tau$ is integrated out using a memoryless negative exponential distribution $P(\mathcal{T}=\tau)=\alpha e^{-\alpha \tau}$. Then, the time integrated solutions
$$
t(\mathbf{x}) = \alpha \int_0^\infty u(\mathbf{x},\tau) e^{-\alpha \tau} {\rm d}\tau,
$$
with decay parameter $\alpha$, are in fact the solutions %\cite{DuitsBurgeth2007}
\begin{equation}
t = \underset{t \in \mathbb{L}_2(\mathbb{R}^2)}{\operatorname{argmin}} \left[ \lVert t - t_0 \rVert^2_{\mathbb{L}_2(\mathbb{R}^2)} + \lambda \int_{\mathbb{R}^2} \lVert \nabla t (\tilde{\mathbf{x}}) \rVert^2 \, {\rm d}\tilde{\mathbf{x}} \right],
\end{equation}
with $\lambda = \alpha^{-1}$, and corresponding Euler-Lagrange equation
\begin{equation}
(I - \lambda \Delta) t = t_0 \;\;\; \Leftrightarrow \;\;\; t = \lambda^{-1} \left( \frac{1}{\lambda} - \Delta \right)^{-1} t_0,
\end{equation}
to which we refer as the ``resolvent'' equation \cite{Yosida1995}, as it involves operator $(\alpha I - \Delta)^{-1}$, $\alpha = \lambda^{-1}$. In the next subsections, we follow a similar procedure with $SE(2)$ instead of $\mathbb{R}^2$, and show how the smoothing regularizer in Eq.~(28) and (30) of the main article relates to Laplace transforms of hypo-elliptic diffusions on the group $SE(2)$ \cite{ZhangDuits2014,Duits2010}.

\subsection{The Fundamental Single Patch Problem}
\label{subsec:toyProblem}%\ref{eq:energySE2LinearRegression} and \ref{eq:energySE2LogisticRegression}
In order to grasp what the (anisotropic regularization term) in Eq. (28) and (30) of the main article actually means in terms of stochastic interpretation/probabilistic line propagation, let us consider the following single patch problem and optimize
\begin{multline}\label{min}
\mathcal{E}_{sp}(T) =
\left| \left( G_s *_{\mathbb{R}^2} T(\cdot,\cdot,\theta_0) \right)(\mathbf{x}_0) - 1 \right|^2 \\
+ \lambda \int_{\mathbb{R}^2} \int_0^{2\pi} \lVert \nabla T(\tilde{\mathbf{x}},\tilde{\theta}) \rVert^2_{D} {\rm d}\tilde{\mathbf{x}}{\rm d}\tilde{\theta}
+ \mu \lVert T \rVert^2_{\mathbb{L}_2(SE(2))},
\end{multline}
with $(\mathbf{x}_0,\theta_0)=g_0:=(x_0,y_0,\theta_0) \in SE(2)$ the fixed center of the template, and with spatial Gaussian kernel
$$
G_s(\mathbf{x}) = \frac{1}{4 \pi s} e^{-\frac{\lVert \mathbf{x} \rVert^2}{4 s}}.
$$
%, $s=\frac{1}{2}\sigma^2$.
Regarding this problem, we note the following:
\begin{itemize}
\item In the original problem (28) of the main article we take $N=1$, with
\begin{equation}
\label{eq:uf1}
U_{f_1}(x,y,\theta)=G_s(x-x_0,y-y_0) \, \delta_{\theta_0}(\theta)
\end{equation}
representing a local spatially smoothed spike in $SE(2)$, and set $y_1 = 1$. The general single patch case (for arbitrary $U_{f_1}$) can be deduced by superposition of such impulse responses.
\item We use $\mu>0$ to suppress the output elsewhere.
\item We use $0 < s \ll 1$. This minimum scale due to sampling removes the singularity at $(\mathbf{0},0)$ from the kernel that solves (\ref{min}), as proven in \cite{ZhangDuits2014}.
\end{itemize}

\begin{mythm}
\label{thm:1}
%Let $\hat{T}$ be the solution to the single patch problem (\ref{min}). Up to scalar multiplication this solution coincides with the time integrated hypo-elliptic Brownian motion kernel on $SE(2)$ depicted in Fig.~\ref{fig:stochasticEnhancement}.
The solution to the single patch problem (\ref{min}) coincides up to scalar multiplication with the time integrated hypo-elliptic Brownian motion kernel on $SE(2)$ (depicted in Fig.~\ref{fig:stochasticEnhancement}).
\end{mythm}
\begin{proof}[\textbf{Proof}]
We optimize $\mathcal{E}_{sp}(T)$ over the set $\mathcal{S}(SE(2))$ of
all functions $T:SE(2) \rightarrow \mathbb{R}$ that are bounded and on $SE(2)$, infinitely differentiable on $SE(2) \setminus \{g_0\}$, and rapidly decreasing in spatial direction, and $2\pi$ periodic in $\theta$.
%A priori, it is not immediately clear that such optimization problem is well-posed under such constraints. However, one can perform the optimization over a wider Sobolev type of function space and then  one obtains (via H\"{o}rmander condition \cite{Hoermander1967}) a minimizer that indeed satisfies the constraints.
We omit topological details on function spaces and H\"{o}rmanders condition \cite{Hoermander1967}.
Instead, we directly proceed with applying the Euler-Lagrange technique to the single patch problem:
%\begin{equation} \label{EL}
%(S^*_s S_s + \lambda R + \mu I) \hat{T} = S^*_s y_1 = S^*_s 1,
%\end{equation}
\begin{multline} \label{EL}
\forall_{\delta \in \mathcal{S}(SE(2))}: \underset{\epsilon \downarrow 0}{\operatorname{lim}} \left\{ \frac{\mathcal{E}_{sp}(T + \epsilon \delta) - \mathcal{E}_{sp}(T)}{\epsilon} \right\} = 0
\Leftrightarrow \\
(S^*_s S_s + \lambda R + \mu I) T = S^*_s y_1 = S^*_s 1,
\end{multline}
with linear functional (distribution) $\mathcal{S}_s$ given by
$$ (S_s T) = (G_{s} *_{\mathbb{R}^2} T(\cdot,\theta_0))(\mathbf{x}_0),$$
and with regularization operator $R$ given by
$$
R = -\Delta_{SE(2)} := - (D_{\theta\theta} \partial_\theta^2 + D_{\xi\xi}\partial_\xi^2 + D_{\eta\eta}\partial_\eta^2) \geq 0.
$$
Note that
$\lim \limits_{s \to 0} S_{s} = \delta_{(\mathbf{x}_0,\theta_0)}$ in distributional sense, and that the constraint $s>0$ is crucial for solutions $T$ to be bounded at $(\mathbf{x}_0,\theta_0)$. By definition the adjoint operator $S_s^*$ is given by
$$
%\begin{aligned}
\begin{array}{rl}
(S^*_s y, T)_{\tiny\mathbb{L}_2(SE(2))} \!\!\!\!\!\! & = (y, S_s T)
       = y \int_{\mathbb{R}^2} G_s(\mathbf{x}-\mathbf{x}_0) T(\mathbf{x},\theta_0)\, {\rm d}\mathbf{x}\\
      & = y \int\limits_0^{2\pi} \int\limits_{\mathbb{R}^2} G_s(\mathbf{x}-\mathbf{x}_0)\delta_{\theta_0}(\theta) T(\mathbf{x},\theta)\, {\rm d}\mathbf{x}{\rm d}\theta,\\
      & = (y \;G_s(\cdot-\mathbf{x}_0) \delta_{\theta_0}(\cdot),T )_{\tiny\mathbb{L}_2(SE(2))}
\end{array}
%\end{aligned}
$$
and thereby we deduce that
$$
\begin{aligned}
(S^*_s y)(\mathbf{x},\theta) &= y \;G_s(\mathbf{x}-\mathbf{x}_0) \delta_{\theta_0}(\theta), \\
\ \ \ \   S^*_s (S_s T) &= T_0^{s}\; G_s(\mathbf{x}-\mathbf{x}_0) \delta_{\theta_0}(\theta),
\end{aligned}
$$
with $\infty > T_0^{s}:=(G_{s} *_{\mathbb{R}^2} T(\cdot,\theta_0))(\mathbf{x}_0)>  1$ for $0 < s \ll 1$.
The  Euler-Lagrange equation (\ref{EL}) becomes
$$
(-\lambda \Delta_{SE(2)} + \mu I ) T = (1 - T^s_0)G_{s}(\mathbf{x}-\mathbf{x}_0) \delta_{\theta_0}(\theta).
$$
Now, when setting $T_{new} = \frac{T}{1-T_0^s}$ we arrive at the hypo-elliptic resolvent equation on $SE(2)$:
\begin{equation}
\label{aap}
\begin{aligned}
(- \lambda \Delta_{SE(2)} + \mu I) T_{new} = (G_{s} *_{\mathbb{R}^2} \delta_{\mathbf{x}_0}) \delta_{\theta_0} \;\;\;\;\;\;
 \Leftrightarrow \\%\;\;\;\;\;\;\;\;\;\;\;\;\;\;\;\;\;\;\;\;\;\;\;\;\\
T_{new} =\left( - \lambda \Delta_{SE(2)} + \mu I \right)^{-1} e^{s \Delta_{\mathbb{R}^2}} \delta_{g_0} \>\\
			  = e^{s \Delta_{\mathbb{R}^2}} \left( - \lambda \Delta_{SE(2)} + \mu I \right)^{-1} \delta_{g_0}
\end{aligned}
\end{equation}
where we write $e^{s \Delta_{\mathbb{R}^2}} f= G_{s} *_{\mathbb{R}^{2}} f$ for the diffusion operator, to stress
the vanishing commutators
\[
[e^{s \Delta_{\mathbb{R}^2}},\Delta_{SE(2)}]=e^{s \Delta_{\mathbb{R}^2}}\Delta_{SE(2)}-\Delta_{SE(2)} e^{s \Delta_{\mathbb{R}^2}}=0,
\]
which directly follows from $[\Delta_{\mathbb{R}^2},\Delta_{SE(2)}]=0$.
In fact, from these vanishing commutators one can deduce that, thanks to the isotropy of Gaussian kernel, blurring with inner-scale $s>0$ can be done either before applying the resolvent operator or after (as seen in (\ref{aap})).

The solutions $T_{new}$ are precisely the probabilistic kernels $R_{\alpha,s}:SE(2) \to \mathbb{R}$ for time integrated contour enhancements studied in \cite{ZhangDuits2014,Duits2010}.
In fact we see that
\[
T_{new}(g)=\mu^{-1} R_{\alpha,s}(g_0^{-1}g),
\]
where $R_{\alpha,s}= (I -\alpha^{-1} \Delta_{SE(2)})^{-1} e^{s \Delta_{\mathbb{R}^{2}}} \delta_{(\mathbf{0},0)}$ (i.e., the impuls response of the resolvent operator) denotes the time-integration of the hypo-elliptic diffusion kernel
$K_{\tau,s}= e^{\tau \Delta_{SE(2)}}e^{s \Delta_{\mathbb{R}^2}} \delta_{(\mathbf{0},0)}$:
\[
R_{\alpha,s}(g)= \alpha \int_{0}^{\infty} K_{\tau,s}(g) \, e^{-\alpha \tau}\, {\rm d}\tau,
\]
for which 3 different exact analytic formulas are derived in \cite{Duits2010}. The kernel $R_{\alpha,s}(\mathbf{x},\theta)$ denotes the probability density of finding a random brush stroke (regardless its traveling time) at location $\mathbf{x}$ with orientation $\theta$ given that a `drunkman's pencil' starts at $g=(\mathbf{0},0)$ at time zero. Here the traveling time $\tau$ of the random pencil is assumed to be negatively exponentially distributed with expectation $\alpha^{-1}$.
\end{proof}

%i.e. $P(\mathcal{T}=\tau)=\alpha e^{-\alpha \tau}$,
%see Fig.~\ref{fig:stochasticEnhancement}.
%\textbf{Conclusion.} The single patch problem equals the time integrated Brownian motion kernel on $SE(2)$. In fact, our single patch method provides a new finite element type of implementation as we will explain next.
\begin{figure}
\centerline{
\includegraphics[width=\hsize]{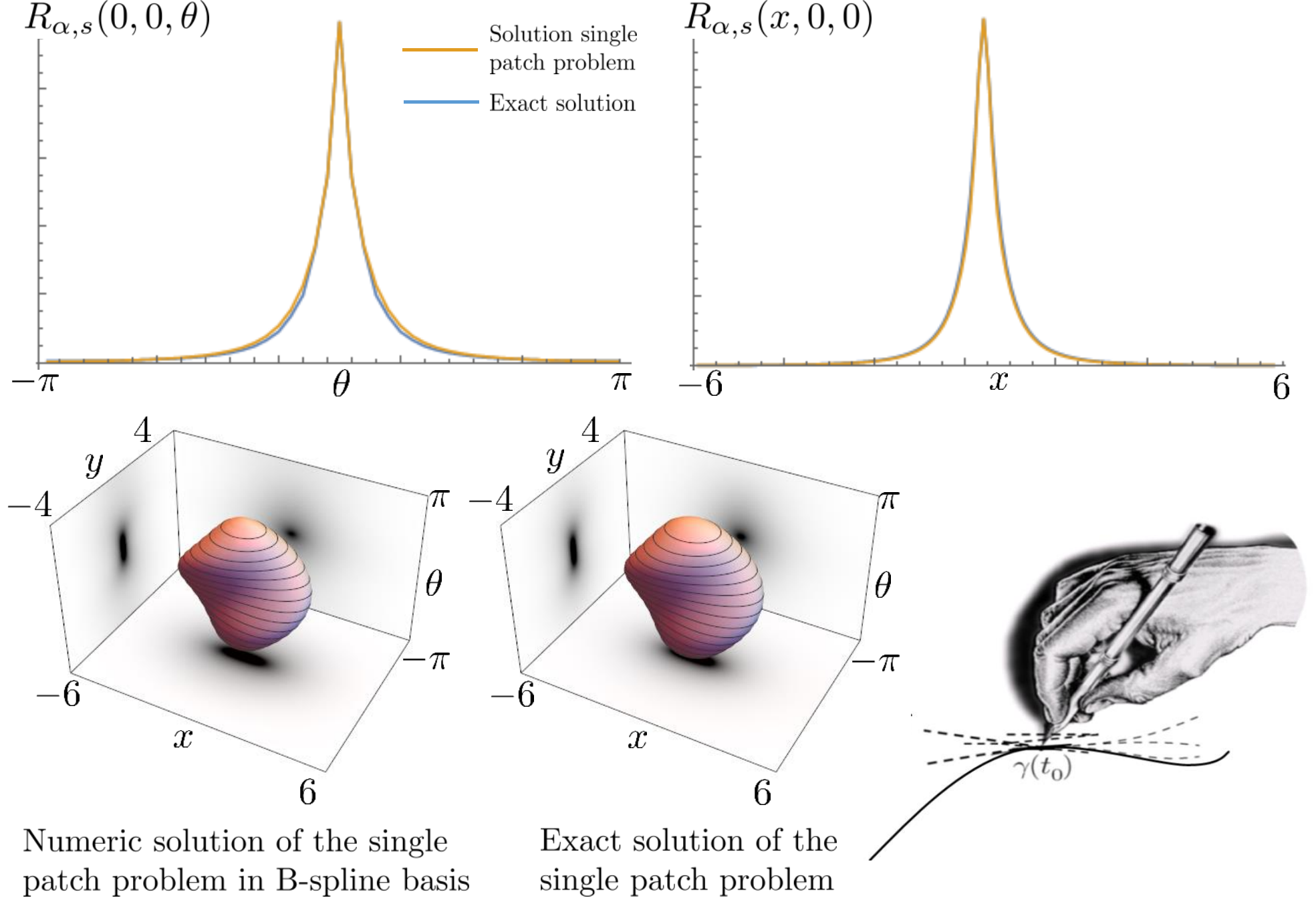}
}
\caption{
Top row: Comparison of kernel $R_{\alpha,s}(x,y,\theta)$ along respectively the $\theta$ and $x$ axis. Bottom row:
Isosurface of the kernel computed by solving the fundamental single patch problem (\ref{min}), the exact solution, and an illustration of the drunkman's pencil. For Monte Carlo simulations of the drunkman's pencil see the supplementary materials.% of this paper.
\label{fig:stochasticEnhancement}}
\end{figure}

\subsection{Expansion in B-splines}
\label{subsec:expansionBSplines}
Now we consider the B-spline expansions (Eq.~(34) in the main article) and apply our optimization algorithm (cf. Subsec.~2.4 of the main article) to the single patch problem
(\ref{min}), with $(\mathbf{x}_0,\theta_0) = (\mathbf{0},0)$. Here we no longer need a smoothing with a continuous Gaussian $G_s$, as expansion in the B-spline basis already includes regularization. Now we set for the smooth spike $U_{f_1}(x,y,\theta) = B^n\!\left( \frac{x}{s_k} \right) B^n\!\left( \frac{y}{s_l} \right)B^n\!\left( \frac{\theta \!\!\!\! \mod 2\pi}{s_m} \right)$, and we thus approximate spikes by the same B-spline basis in which we expressed our templates.
%Choosing the discrete $\delta_{(\mathbf{x}_0,\theta_0)}$ for $U_{f_1}$ yields the fundamental B-spline as the continuous impulse response under $n$-th order polynomial spline interpolation. However, the fundamental B-spline of order $n$ is oscillating and looks like a sinc function, \cite[Fig.~2]{Unser93}).
We accept extra regularization (like we did with the Gaussian in the previous section) and choose to represent a spike by a normal B-spline.
%\footnote{Note also that as a result of the central limit theorem, each B-spline becomes more and more like the Gaussian with increasing spline degree $n$. I.e. $B^n(x) \approx G_{s_n}(x)$, with $s_n = \frac{1}{24}(n+1)$.}.
After all, via the central limit theorem B-splines converge to Gaussians when increasing $n$.
We also considered to instead use the fundamental B-Spline \cite[Fig.~2]{Unser93}, which is sharper but suffers from oscillations, yielding less favorable results.

% instead of the fundamental B-spline. %As $n$ increases, $B^n$ converges to a Gaussian via the central limit theorem.
%\end{remark}

In our normal B-spline setting, %We set $U_{f_1}= \delta_{e}$ and thereby do not include spatial Gaussian smoothing.
%Then we obtain the matrix equations
this choice of smooth spike representation (cf. Eq.~(14) in the main article) leads to the following
%matrix
equations
$$
(S^\dagger S  + \lambda R + \mu I ) T = S^\dagger 1,
$$
with $S$ the $[1 \times N_k N_l N_m]$-matrix whose components are given by $M(0,0,0)\, B_{s_k s_l s_m}(k,l,m)$.
Akin to the previous derivations (\ref{aap}) this matrix-equation can be rewritten as
%$$
%\left( R + \frac{\mu}{\lambda} I \right) \hat{T}_{new} = \lambda^{-1}\, S^\dagger \mathbf{1}.
%$$
$$
\left( \lambda R + \mu I \right) T_{new} = S^\dagger 1.
$$
%\begin{remark}
%The B-spline expansions (\ref{splineexp}) boil down to a convolution with a single B-spline function $B_{s_k s_l s_m}$ centered at the origin. In view of the central limit theorem, each $B$-spline basis function component $B^{n}(x)$ starts to look more and more like a Gaussian when the order $n$ of the spline increases. I.e. $B^n(x) \approx G_{s_n}(x)$, with $s_n = \frac{1}{24}(n+1)$. Also the spatial part of the basis functions $B_{s_k s_l s_m}$ become more and more isotropic when $n$ increases. So the column vector $S^\dagger 1$ whose components equal $M(e)\, B_{s_k s_l s_m}(k,l,m)$ represents a fat discrete spike likewise the Gaussian kernel in the previous subsection.
%\end{remark}
In particular our $B$-spline basis algorithm is a new algorithm that can be used for the resolvent (hypo-)elliptic diffusion process on $SE(2)$.
The benefit over Fourier based algorithms is the local support of the basis functions, which allow for sparse representations.

In Fig.~\ref{fig:stochasticEnhancement} we compare the impulse response for Tikhonov regularization via our B-spline expansion algorithm with the Brownian motion prior on $SE(2)$
%(i.e. solving the fundamental single patch problem)
(using a fine B-spline basis) to
the exact solutions derived in \cite{ZhangDuits2014,Duits2010}.
The strong accuracy of our algorithm shows that even in the discrete B-spline setting the probabilistic interpretation (Thm.~\ref{thm:1}) of our prior in $SE(2)$-template matching holds.

%\textbf{Conclusion.} We have established a probabilistic interpretation of the choice of prior: We work with smoothing splines, cf.~\!\cite{Unser93,deBoor1978,Green1993,Unser1999}, where we rely on a (hypo-elliptic and time integrated) Brownian motion prior on $SE(2)$ for regularization.

%To check the accuracy of our algorithm in the linear regression case, we compare the impulse response for Tikhonov regularization via our B-spline expansion algorithm with the Brownian motion prior on $SE(2)$
%(i.e. solving the fundamental single patch problem)
%(using a fine B-spline basis) to
%the exact solutions derived in \cite{ZhangDuits2014,Duits2010}.
%The strong accuracy of our algorithm is depicted in Figure~\ref{fig:stochasticEnhancement}, and shows that the probabilistic interpretation of our prior in $SE(2)$-template matching holds.
%For completeness, in Figure~\ref{fig:tobeinserted} (part B), we have also included comparisons towards the common exact solutions for Tikhonov regularization on the $\mathbb{R}^2$-case. This supports the fact that in $\mathbb{R}^2$-template matching, our prior is a (time integrated) isotropic Brownian motion prior on $\mathbb{R}^2$.

%Now we consider
%$$
%\underset{c^{k,l,m}}{\operatorname{min}} \;\; \mathcal{E}_{sp}\left(\sum_k \sum_l \sum_m c^{k,l,m} B_k \otimes B_l \otimes B_m \right).
%$$
%\begin{equation}
%\begin{array}{l}
%\end{array}
%\end{equation}

%{\color{red} REMCO, DIT STUKJE IS VOOR JOU :)}

\subsection{The Drunkman's Pencil}
Similar to the diffusions on $\mathbb{R}^2$, given by (\ref{eq:pdeDiffusion}), the hypo-elliptic diffusion process on $SE(2)$ is described by the following PDE:
\begin{equation}
\label{eq:pdeDiffusionSE2}
\left\{
\begin{array}{cl}
\tfrac{\partial}{\partial \tau}W &= (D_{\xi\xi} \partial_{\xi}^2 + D_{\theta\theta} \partial_{\theta}^2) W,\\
W |_{\tau=0} &= W_0,
\end{array}
\right.
\end{equation}
initialized with $W_0 \in \mathbb{L}_2(\mathbb{R}^2)$ at time $\tau = 0$. The PDE can be used to obtain the solutions of our single patch problem by initializing $W_0$ with a smooth spike such as we did in Subsec.~\ref{subsec:expansionBSplines}, e.g. taking $W_0 =  U_{f_1}(x,y,\theta) = B^n\!\left( \frac{x}{s_k} \right) B^n\!\left( \frac{y}{s_l} \right)B^n\!\left( \frac{\theta \!\!\!\! \mod 2\pi}{s_m} \right)$.

The PDE in (\ref{eq:pdeDiffusionSE2}) is the forward Kolmogorov equation \cite{hsustochastic} of the following stochastic process \cite{ZhangDuits2014}:
\begin{equation}
\label{eq:stochasticProcess}
\left\{
\begin{array}{l}
\mathbf{x}(\tau) = \mathbf{x}(0) + \\
\;\;\;\;\;\;\;\;\;\;\;\;\;\;\;\;\;\; \sqrt{2 D_{\xi\xi}} \; \epsilon_\xi \int_0^\tau (\cos \theta(\tau) \mathbf{e}_x + \sin \theta(\tau) \mathbf{e}_y ) \frac{1}{2\sqrt{\tau}} \rm d\tau\\
\theta(\tau) = \theta(0) + \sqrt{\tau} \sqrt{2 D_{\theta\theta}} \; \epsilon_\theta, \;\;\;\;\;\;\;\;\;\; \epsilon_\xi, \epsilon_\theta ~ \mathcal{N}(0,1),
\end{array}
\right.
\end{equation}
where $\epsilon_\xi$ and $\epsilon_\theta$ are sampled from a normal distribution with expectation $0$ and unit standard deviation. The stochastic process in (\ref{eq:stochasticProcess}) can be interpreted as the motion of a drunkman's pencil: it randomly moves forward and backwards, and randomly changes its orientation along the way. The resolvent hypo-elliptic diffusion kernels $R_{\alpha,s}(g)$ (solutions to the fundamental single patch problem, up to scalar multiplication) can then also be obtained via Monte Carlo simulations, where the stochastic process is sampled many times with a negatively exponentially distributed traveling time ($P(\mathcal{T}=\tau)=\alpha e^{-\alpha \tau}$) in order to be able to estimate the probability density kernel $R_{\alpha,s}(g)$. This process is illustrated in Fig.~\ref{fig:contourEnhancement}.

\begin{figure}
\centerline{
\includegraphics[width=\hsize]{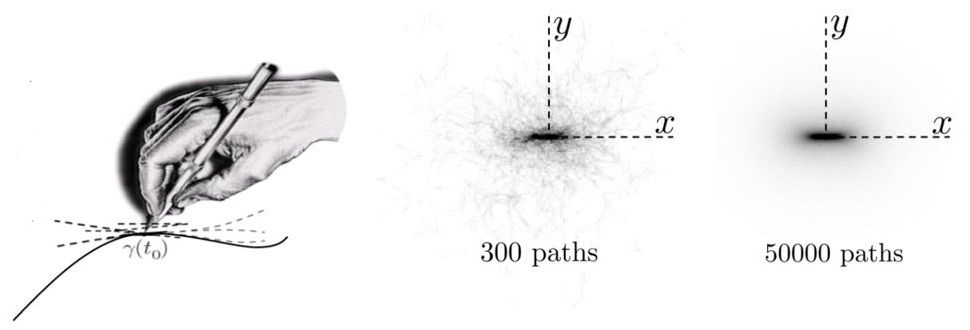}
}
\caption{
Stochastic random process for contour enhancement.
\label{fig:contourEnhancement}}
\end{figure}

\section{The Smoothing Regularization Matrix R}
\label{sec:regMatrix}
When expanding the templates $t$ and $T$ in a finite B-Spline basis (Sec.~2 and 3 of the main article), the energy functionals (7), (11), (28) and (30) of the main article can be expressed in matrix vector form. The following theorems summarize how to compute the matrix $R$, which encodes the smoothing prior, for respectively the $\mathbb{R}^2$ and $SE(2)$ case.
\begin{mylem}
\label{lem:regtermR2}
The discrete smoothing regularization-term of energy functional (7) of the main article can be expressed directly in the B-Spline coefficients $\mathbf{c}$ as follows
\begin{equation}
\iint_{\mathbb{R}^2} \lVert \nabla t(x,y) \rVert^2 {\rm d}x {\rm d}y
 = \mathbf{c}^\dagger R \mathbf{c},
\end{equation}
with $\mathbf{c}$ given by Eq.~(16) of the main article, and with
\begin{equation}
\label{eq:RR2}
R = R_x^k \otimes R_x^l + R_y^k \otimes R_y^l,
\end{equation}
a $[N_k N_l \times N_k N_l]$ matrix. The elements of the matrices in (\ref{eq:RR2}) are given by
\begin{equation}
\label{eq:elementsR2}
\begin{array}{rl}
R_x^k(k,k') &= -\frac{1}{s_k}\frac{ \partial^2 B^{2n+1}}{\partial x^2}(k'-k)\\
R_x^l(l,l') &= s_l B^{2n+1}(l'-l),\\
R_y^k(k,k') &= s_k B^{2n+1}(k'-k),\\
R_y^l(l,l') &= -\frac{1}{s_l}\frac{ \partial^2 B^{2n+1}}{\partial y^2}(l'-l).
\end{array}
\end{equation}
\end{mylem}
\begin{proof}
For the sake of readability we divide the regularization-term in two parts:
\begin{equation}
\label{eq:regularizationTermDiscrete}
\begin{array}{rl}
\iint_{\mathbb{R}^2} \lVert \nabla t(x,y) \rVert^2 {\rm d}x {\rm d}y
& = \iint_{\mathbb{R}^2} \left| \frac{\partial t}{\partial x}(x,y)\right|^2 \\
& \hspace{6em} + \left| \frac{\partial t}{\partial y}(x,y)\right|^2 {\rm d}x {\rm d}y,
\end{array}
\end{equation}
where
$$
\begin{array}{rl}
\mathcal{R}_x &= \iint_{\mathbb{R}^2} \left| \frac{\partial t}{\partial x}(x,y)\right|^2 {\rm d}x {\rm d}y,\;\;\text{and}\\
\mathcal{R}_y &= \iint_{\mathbb{R}^2} \left| \frac{\partial t}{\partial y}(x,y)\right|^2 {\rm d}x {\rm d}y.
\end{array}
$$
We first derive the matrix-vector representation of $\mathcal{R}_x$ as follows:
\begin{equation}
\begin{array}{rl}
\mathcal{R}_x &= \iint_{\mathbb{R}^2} \left| \frac{\partial t}{\partial x}(x,y)\right|^2 {\rm d}x {\rm d}y\\
&= \sum \limits_{k,k'=1}^{N_k} \sum \limits_{l,l'=1}^{N_l} \iint_{\mathbb{R}^2}
\overline{c_{k,l} \frac{ \partial B^n}{\partial x}(\frac{x}{s_k}-k) B^n(\frac{y}{s_l}-l)}\\
&\hspace{9em}c_{k,l} \frac{ \partial B^n}{\partial x}(\frac{x}{s_k}-k') B^n(\frac{y}{s_l}-l') {\rm d}x {\rm d}y\\
&= \sum \limits_{k,k'=1}^{N_k} \sum \limits_{l,l'=1}^{N_l}
\overline{c_{k,l}} c_{k',l'}\\
&\hspace{6em}\left[\int\limits_{-\infty}^\infty
\frac{ \partial B^n}{\partial x}(\frac{x}{s_k}-k) \frac{ \partial B^n}{\partial x}(\frac{x}{s_k}-k')
{\rm d}x\right]\\
&\hspace{9em}\left[\int\limits_{-\infty}^\infty
B^n(\frac{y}{s_l}-l) B^n(\frac{y}{s_l}-l')
{\rm d}y\right]\\
&\overset{1}{=} \sum \limits_{k,k'=1}^{N_k} \sum \limits_{l,l'=1}^{N_l}
\overline{c_{k,l}} c_{k',l'}
\left[\frac{1}{s_k}\left(\frac{ \partial B^n}{\partial x} * \frac{ \partial B^n}{\partial x}\right)(k'-k)\right]\\
&\hspace{14em}\left[s_l \left(B^n * B^n\right)(l'-l)\right]\\
&\overset{2}{=} \sum \limits_{k,k'=1}^{M} \sum \limits_{l,l'=1}^{N}
\overline{c_{k,l}} c_{k',l'}
\left[\frac{1}{s_k} \frac{ \partial^2 B^{2n+1}}{\partial x^2}(k'-k)\right]\\
&\hspace{14em}\left[s_l B^{2n+1}(l'-l)\right]\\
&= \mathbf{c}^\dagger (R_x^k \otimes R_x^l) \mathbf{c}.
\end{array}
\end{equation}
Here the following properties are used:
\begin{enumerate}
  \item The integrals of shifted B-splines can be expressed as convolutions:
  \begin{multline*}
    \label{eq:convolutionTerms}
    \int_{-\infty}^{\infty} \frac{ \partial B^n}{\partial x}\left(\frac{x}{s_k}-k\right) \frac{ \partial B^n }{ \partial x}\left(\frac{x}{s_k} - k'\right) {\rm d}x \\
    =
    - \frac{1}{s_k} \int_{-\infty}^{\infty} \frac{ \partial B^n}{\partial u}(u)\frac{ \partial B^n }{\partial u}( (k'-k) - u) {\rm d}u\\
    =
    - \frac{1}{s_k} \left(\frac{ \partial B^n}{\partial u} * \frac{ \partial B^n}{\partial u}\right)(k'-k).
  \end{multline*}
  This is easily verified by substitution of integration variable ($u = -\frac{x}{s_k}+k$) and noting that $B^n(x) = B^n(-x)$ and $\frac{\partial B^n}{\partial x}(x) = -  \frac{\partial B^n}{\partial x}(-x)$.
  \item Convolution of two B-splines $B^n$ of order $n$ results in a B-Spline $B^{2n+1}$ of order $2n+1$:
  \begin{equation*}
    \label{eq:splineConvolution}
    B^n * B^n = B^{2n + 1}.
  \end{equation*}
\end{enumerate}
The elements of the matrices $R_y^k$ and $R_y^l$ are derived in a similar manner.
\end{proof}
As a result of Lemma~\ref{lem:regtermR2} we can state the following.
\begin{mycor} \label{cor:1}
Let $V=\operatorname{span}\{ B^n_{k,l} \}$, with $k = 1,\ldots, N_k, l=1,\ldots, N_l$, and shifted B-splines $B^n_{k,l}$ of order $n$ (cf. Subsec. 2.4 of the main article). Let the energy function $E_{lin}^B: \mathbb{R}^{N_k N_l} \to \mathbb{R}^+$ be given by Eq.~(13) of the main article. Then the optimal continuous template of the constrained optimization problem (cf. Eq.~(7) of the main article)
\[
t^*=\underset{t \in V}{\operatorname{argmin}}\;\; E_{lin}(t)
\]
has coefficients $\mathbf{c}^*$ w.r.t. the B-spline basis for $V$,
that are the unique solution of
\[
\nabla_{\mathbf{c}} {E}^B(\mathbf{c}^*)=\mathbf{0},
\]
which boils down to Eq.~(14) of the main article.
\end{mycor}

\begin{mylem}
\label{lem:regtermSE2}
The discrete regularization-term of energy functional (28) of the main article can be expressed directly in the B-Spline coefficients:
\begin{equation}
\iiint_{SE(2)} \lVert \nabla T \rVert_D \; {\rm d}x{\rm d}y{\rm d}\theta
=
\mathbf{c}^T (D_{\xi\xi} R_\xi + D_{\eta\eta} R_\eta + D_{\theta\theta} R_\theta) \mathbf{c}.
\end{equation}
Matrix $R_\xi$ is given by
\begin{multline}
\label{eq:RmatrixXi}
R_\xi =
\left(R_\xi^{Ix} \otimes R_\xi^{Iy} \otimes R_\xi^{I\theta}\right)
+ \left(R_\xi^{IIx} \otimes R_\xi^{IIy} \otimes R_\xi^{II\theta}\right)\\
+ \left(R_\xi^{IIIx} \otimes R_\xi^{IIIy} \otimes R_\xi^{III\theta}\right)
+ \left(R_\xi^{IVx} \otimes R_\xi^{IVy} \otimes R_\xi^{IV\theta}\right)
\end{multline}
with the elements of the matrices used in the Kronecker products given by
\begin{equation}
\label{eq:RmatricesXi}
\begin{array}{l}
R_\xi^{Ix}(k,k')\;\;\;  = - \frac{1}{s_k}  \frac{\partial^2 B^{2n+1}}{\partial x^2}  (k'-k),\\
R_\xi^{Iy}(l,l')\;\;\;\;\;  = s_l  B^{2n+1}  (l'-l),\\
R_\xi^{I\theta}(m,m')  = \int\limits_0^\pi \cos^2(\theta)B^n(\frac{\theta}{s_m}-m)B^n(\frac{\theta}{s_m}-m'){\rm d}\theta,
\end{array}
\end{equation}
\begin{equation}
\begin{array}{l}
R_\xi^{IIx}(k,k') \;\;\, = - R_\xi^{IIIx}(k,k') \;\;\; = \frac{\partial B^{2n+1}}{\partial x}  (k'-k),\\
R_\xi^{IIy}(l,l') \;\;\;\; = - R_\xi^{IIIy}(l,l') \;\;\;\;\; = -\frac{\partial B^{2n+1}}{\partial y}  (l'-l),\\
R_\xi^{II\theta}(m,m')  = \;\;\; R_\xi^{III\theta}(m,m')  = \\
\hspace{6em} \int\limits_0^\pi \cos(\theta)\sin(\theta)B^n(\frac{\theta}{s_m}-m)B^n(\frac{\theta}{s_m}-m'){\rm d}\theta,
\end{array}
\end{equation}
\begin{equation}
\begin{array}{rl}
R_\xi^{IVx}(k,k') \;\;\,\,& = s_k B^{2n+1}(k'-k),\\
R_\xi^{IVy}(l,l') \;\;\;\;\;& = - \frac{1}{s_l} \frac{\partial^2 B^{2n+1}}{\partial y^2}  (l'-l),\\
R_\xi^{IV\theta}(m,m') & = \int\limits_0^\pi \sin^2(\theta)B^n(\frac{\theta}{s_m}-m)B^n(\frac{\theta}{s_m}-m'){\rm d}\theta.
\end{array}
\end{equation}
Matrix $R_\eta$ is given by
\begin{multline}
R_\eta =
\left(R_\xi^{IIx} \otimes R_\xi^{IIy} \otimes R_\xi^{IV\theta}\right)
- \left(R_\xi^{IIx} \otimes R_\xi^{IIy} \otimes R_\xi^{II\theta}\right)\\
- \left(R_\xi^{IIIx} \otimes R_\xi^{IIIy} \otimes R_\xi^{III\theta}\right)
+ \left(R_\xi^{IVx} \otimes R_\xi^{IVy} \otimes R_\xi^{I\theta}\right).
\end{multline}
Matrix $R_\theta$ is given by
\begin{equation}
R_\theta =
\left(R_\theta^{x} \otimes R_\theta^{y} \otimes R_\theta^{\theta}\right),
\end{equation}
with the elements of the matrices given by
\begin{equation}
\label{eq:RmatricesTheta}
\begin{array}{rl}
R_\theta^{x}(k,k')\;\;\; & = s_k  B^{2n+1} (k'-k),\\
R_\theta^{y}(l,l')\;\;\;\;\; & = s_l  B^{2n+1}  (l'-l),\\
R_\theta^{\theta}(m,m') & = - \frac{1}{s_m}  \frac{\partial^2 B^{2n+1}}{\partial \theta^2}  (m'-m).
\end{array}
\end{equation}
\end{mylem}
\begin{proof}
The proof of Lemma~\ref{lem:regtermSE2} follows the same steps as in the proof of Lemma~\ref{lem:regtermR2}, only here left-invariant derivatives are used (cf. Eq.~(24) of the main article). The four separate terms $I-IV$ of Eq.~(\ref{eq:RmatrixXi}) arise from the left invariant derivative $\partial_\xi$: $\left|\frac{\partial T}{\partial \xi}\right|^2 = \left|\cos( \theta )\frac{\partial T}{\partial x} + \sin( \theta )\frac{\partial T}{\partial y}\right|^2$.
\end{proof}
Lemma~\ref{lem:regtermSE2} has the following consequence.
\begin{mycor} \label{cor:2}
Let $V=\operatorname{span}\{ B^n_{k,l,m} \}$, with $k = 1,\ldots, N_k, l=1,\ldots, N_l, m=1,\ldots, N_m$, and shifted B-splines $B^n_{k,l,m}$ of order $n$ (cf. Subsec. 3.5 of the main article). Let the energy function $\mathcal{E}_{lin}^B: \mathbb{R}^{N_k N_l N_m} \to \mathbb{R}^+$ be given by
%$\mathcal{E}^B(\mathbf{c})$ be given by
\[
\mathcal{E}_{lin}^B(\mathbf{c}) = \frac{1}{N} \|S \mathbf{c} - \mathbf{y}\|^2 + \mathbf{c}^\dagger (\lambda R+\mu I) \mathbf{c}
\]
With $S$ and $\mathbf{y}$ given by (33) of the main article and with $R=D_{\xi\xi} R_{\xi} + D_{\eta \eta}R_{\eta} + D_{\theta\theta} R_{\theta}$ given in Lemma~\ref{lem:regtermSE2}.
Then the optimal continuous template of the constrained optimization problem (cf. Eq.~(28) of the main article)
\[
T^*=\underset{T \in V}{\operatorname{argmin}}\;\; \mathcal{E}_{lin}(T)
\]
has coefficients $\mathbf{c}^*$ w.r.t. the B-spline basis for $V$
that are the unique solution of
\[
\nabla_{\mathbf{c}} \mathcal{E}^B(\mathbf{c}^*)=\mathbf{0} .
\]
This boils down to Eq.~(14) of the main article, but then on $\mathbb{R}^{N_k N_l N_m}$ with matrices $R$ and $S$ given above.
\end{mycor} 

\section{Normalized Cross Correlation}
\label{sec:normalizedCrossCorrelation}
In most applications it is necessary to make the detection system invariant to local contrast and luminosity changes. In our template matching framework this can either be achieved via certain pre-processing steps that filter out these variations, or by means of normalized cross-correlation. In normalized cross-correlation, both the template as well as the image are (locally) normalized to zero mean and unit standard deviation (with respect to the inner product used in the cross-correlations).
%For completeness we propose in the section a basic extension of the cross-correlation based template matching to \emph{normalized} cross-correlation. In normalized cross-correlation both the template as well as the image are (locally) normalized to zero mean and unit standard deviation, and can be necessary to compensate for local contrast and luminosity variations. As a result, the output of normalized cross-correlation is always between -1 and 1. Alternatively, one can correct for contrast and luminosity variations via a pre-processing step. Both approaches are studies in this article.
In this section, we explain the necessary adaptations to extend the standard cross-correlation based framework to normalized cross-correlations.

\subsection{Normalized Cross-Correlation in $\mathbb{R}^2$}
\label{subsec:normedCCR2}
In the usual cross-correlation based template matching approach, as described in Sec.~2 of the main article, we rely on the standard $\mathbb{L}_2(\mathbb{R}^2)$ inner product (Eq. (6) of the main article). In normalized cross-correlation it is however convenient to extend this inner product to include a windowing function $m$ which indicates the relevant region (support) of the template. As such, the inner product with respect to windowing function $m$ is given by
\begin{equation}
(t,f)_{\mathbb{L}_2 (\mathbb{R}^2, m d \tilde{\mathbf{x}})} := \int_{\mathbb{R}^2} \overline{t(\tilde{\mathbf{x}})} f(\tilde{\mathbf{x}}) m(\tilde{\mathbf{x}}) {\rm d}\tilde{\mathbf{x}},
\end{equation}
with associated norm $\lVert \cdot \rVert_{\mathbb{L}_2(\mathbb{R}^2, m d \tilde{\mathbf{x}})} = \sqrt{ (\cdot , \cdot )_{\mathbb{L}_2(\mathbb{R}^2, m d \tilde{\mathbf{x}})} }$. The windowing function has to be a smooth function $m:\mathbb{R}^2 \rightarrow \mathbb{R}^+$ with $\int_{\mathbb{R}^2} m(\tilde{\mathbf{x}}){\rm d}\tilde{\mathbf{x}} = 1$. In this work, the use of a window $m$ is also convenient to deal with boundary conditions in the optimization problems for template construction.
We define
\begin{equation}
\label{eq:mass}
m(\mathbf{x}) :=  \varsigma \; e^{-\frac{\lVert \mathbf{x} \rVert^2}{s}} \sum\limits_{i=0}^n \frac{(\lVert \mathbf{x} \rVert^2/s)^{i}}{i!},
\end{equation}
which smoothly approximates the indicator function $1_{[0,r]}(\lVert \mathbf{x} \rVert)$, covering a disk with radius $r$, when setting $s = \frac{2 r^2}{1+2n}$, see e.g. \cite[Fig.~2]{Bekkers2014}. The constant $\varsigma$ normalizes the function such that $\int_{\mathbb{R}^2} m(\tilde{\mathbf{x}}){\rm d}\tilde{\mathbf{x}} = 1$.

In normalized cross-correlation the image is locally normalized (at position $\mathbf{x}$) to zero mean and unit standard deviation, which is done as follows
\begin{equation}
\label{eq:normalization}
\hat{f}_\mathbf{x}(\tilde{\mathbf{x}}) := \frac{f(\tilde{\mathbf{x}}) - \langle f \rangle_{\mathcal{T}_\mathbf{x} m}}{ \lVert f(\tilde{\mathbf{x}}) - \langle f \rangle_{\mathcal{T}_\mathbf{x} m} \rVert_{\mathbb{L}_2(\mathbb{R}^2, \mathcal{T}_\mathbf{x} m {\rm d}\tilde{\mathbf{x}})}},
\end{equation}
with local mean $\langle f \rangle_{m} = ( 1 , f )_{\mathbb{L}_2(\mathbb{R}^2, m {\rm d}\tilde{\mathbf{x}})}$. Template $\hat{t}$ can be obtained via normalization of a given template $t$ via
\begin{equation}
\label{eq:templateNormalizationR2}
\hat{t}(\tilde{\mathbf{x}}) := \frac{t(\tilde{\mathbf{x}}) - \langle t \rangle_{m}}{ \lVert t(\tilde{\mathbf{x}}) - \langle t \rangle_{m} \rVert_{\mathbb{L}_2(\mathbb{R}^2, m {\rm d}\tilde{\mathbf{x}})}}.
\end{equation}
Template matching is then done in the usual way (via (4) of the main article), however now $\hat{t}$ and $\hat{f}_{\mathbf{x}}$ are used instead of $t$ and $f$. In fact, the entire $\mathbb{R}^2$ cross-correlation template matching, and template optimization framework is extended to normalized cross-correlation by substituting all instances of $t$ with $\hat{t}$, $f$ with $\hat{f}_\mathbf{x}$, and $(\cdot , \cdot)_{\mathbb{L}_2(\mathbb{R}^2)}$ with $(\cdot,\cdot)_{\mathbb{L}_2 (\mathbb{R}^2, m d \tilde{\mathbf{x}})}$ in Sec.~2 of the main article. However, since templates $\hat{t}$ are directly constructed via the minimization of energy functionals, we will not explictely normalize the templates, unless they are obtained by other methods. E.g., Eq.~(\ref{eq:templateNormalizationR2}) is used in the main article to construct basic templates obtained by averaging positive object patches (Subsec.~4.1 of the main article).
%Due to the normalization in (\ref{eq:normalization}), the range of $P^{\mathbb{R}^2}_{lin}$ in (\ref{eq:linearFunctional}) is restricted to $[-1,1]$.

\subsection{Normalized Cross-Correlation in $SE(2)$}
\label{subsec:normedCCSE2}
Similar to the $\mathbb{R}^2$ case, templates and orientation scores are locally normalized to zero mean and unit standard deviation, however, now with respect to the $\mathbb{L}_2(SE(2), M d\tilde{g})$ inner product, which is given by
\begin{multline}
(T,U_f)_{\mathbb{L}_2 (SE(2),M {\rm d}\tilde{g})} := \\
\int_{\mathbb{R}^2}\int_{0}^{2\pi} \overline{T(\tilde{\mathbf{x}},\tilde{\theta})} U_f(\tilde{\mathbf{x}},\tilde{\theta}) M(\tilde{\mathbf{x}},\tilde{\theta}) {\rm d}\tilde{\mathbf{x}}{\rm d}\tilde{\theta},
\end{multline}
with norm $\lVert \cdot \rVert_{\mathbb{L}_2(SE(2),M {\rm d}\tilde{g})} = \sqrt{ (\cdot , \cdot )_{\mathbb{L}_2(SE(2),M {\rm d}\tilde{g})} }$. Also here windowing function $M$ indicates the support of the template, and has the property $\int_{\mathbb{R}^2}\int_{0}^{2\pi} M(\tilde{\mathbf{x}},\tilde{\theta}) {\rm d}\tilde{\mathbf{x}}{\rm d}\tilde{\theta} = 1$. We define
\begin{equation}
M(\mathbf{x},\theta) := \frac{1}{2\pi} m(\mathbf{x}),
\end{equation}
independent of $\theta$ and with $m(\mathbf{x})$ given by (\ref{eq:mass}).

The (locally at $g$) normalized orientation score and template $T$ are then given by
\begin{align}
\hat{U}_{f,g}(\tilde{\mathbf{x}},\tilde{\theta}) &:= \frac{U_f(\tilde{\mathbf{x}},\tilde{\theta}) - \langle U_f \rangle_{\mathcal{L}_{g} M}}{ \lVert U_f(\tilde{\mathbf{x}},\tilde{\theta}) - \langle U_f \rangle_{\mathcal{L}_g M} \rVert_{\mathbb{L}_2(SE(2), \mathcal{L}_g M d\tilde{g})}},\label{eq:normalizationSE2a}\\
\hat{T}(\tilde{\mathbf{x}},\tilde{\theta}) &:= \frac{T(\tilde{\mathbf{x}},\tilde{\theta}) - \langle T \rangle_{M}}{ \lVert T(\tilde{\mathbf{x}},\tilde{\theta}) - \langle T \rangle_{M} \rVert_{\mathbb{L}_2(SE(2), M {\rm d}\tilde{g})}},\label{eq:normalizationSE2b}
\end{align}
with mean $\langle U_f \rangle_{M} = ( 1 , U_f )_{\mathbb{L}_2( SE(2), M {\rm d}\tilde{g}) )}$.

\subsection{Efficient local normalization of $\hat{f}_{\mathbf{x}}$ and $\hat{U}_{f,g}$.}
\label{subsec:efficientNormalization}
%\textbf{Efficient local normalization of $\hat{f}_{\mathbf{x}}$ and $\hat{U}_{f,g}$.}
Since the normalized image $\hat{f}_{\mathbf{x}}$ depends on the location $\mathbf{x}$ it needs to be calculated for every translation of the template, which makes normalized cross-correlation computationally expensive. Therefore, (\ref{eq:normalization}) can be approximated by assuming that the local average is approximately constant in the area covered by $m$. That is, assuming $\langle f \rangle_{\mathcal{T}_{\mathbf{x}} m}(\tilde{\mathbf{x}}) \approx \langle f \rangle_{\mathcal{T}_{\tilde{\mathbf{x}}} m}(\tilde{\mathbf{x}}) = (m \star f)(\tilde{\mathbf{x}})$  for \mbox{$\lVert \tilde{\mathbf{x}}-\mathbf{x} \rVert < r$}, with $r$ the radius that determines the extent of $m$, (\ref{eq:normalization}) is approximated as follows:
\begin{equation}
\label{eq:fnormedapprox}
\hat{f}_{\mathbf{x}}(\tilde{\mathbf{x}}) \approx \cfrac{f(\tilde{\mathbf{x}}) - (m \star f)(\tilde{\mathbf{x}})}{ \sqrt{ (m \star (f - (m \star f))^2)(\tilde{\mathbf{x}}) } }.
\end{equation}
Similarly, in the $SE(2)$-case (\ref{eq:normalizationSE2a}) can be approximated via
\begin{equation}
\hat{U}_{f,g}(\tilde{\mathbf{x}},\tilde{\theta}) \approx \frac{U_f(\tilde{\mathbf{x}},\tilde{\theta}) - (M \star_{SE(2)} U_f)(\mathbf{x},\tilde{\theta}) }{ \sqrt{  (M \star_{SE(2)} (U_f - (M \star_{SE(2)} U_f))^2 )(\mathbf{x},\tilde{\theta})  }  }.
\end{equation}

\subsection{Including a Region of Interest Mask}
\label{subsec:roi}
Depending on the application, large portions of the image might be masked out. This for example the case in retinal images (see circular masks in Fig.~\ref{fig:resultsOverviewImageONH}). To deal with this, template matching is only performed inside the region of interest defined by a mask image $m^{roi}:\mathbb{R}^2 \rightarrow \{0,1\}$. Including such a mask is important in normalized template matching, and can be done by replacing the standard inner products by
%\begin{subequations}
\begin{align}
\label{eq:productwithmask}
(t,f)^{roi}_{\mathbb{L}_2(\mathbb{R}^2,m,d\tilde{\mathbf{x}})} &:= \frac{ (t,f m^{roi})_{\mathbb{L}_2(\mathbb{R}^2,m,d\tilde{\mathbf{x}})} }{ (1, m^{roi})_{\mathbb{L}_2(\mathbb{R}^2,m,d\tilde{\mathbf{x}})} },\\
(T,U_f)^{roi}_{\mathbb{L}_2(SE(2),M,d\tilde{g})} &:= \frac{ (T,U_f M^{roi})_{\mathbb{L}_2(SE(2),M,d\tilde{g})} }{ (1, M^{roi})_{\mathbb{L}_2(SE(2),M,d\tilde{g})} },
\end{align}
%\end{subequations}
%where $(1, m^{roi})_{\mathbb{L}_2(\mathbb{R}^2,m,d\tilde{\mathbf{x}})}$ is the fraction of pixels covered by the template that is inside the mask $m^{roi}$,
with $M^{roi}(\mathbf{x},\theta) = m^{roi}(\mathbf{x})$.

\section{Additional Details on the Detection Problems}
\label{sec:additionalDetails}
In this section we describe additional details about the implementation and results of the three detection problems discussed in the main article.

\begin{figure*}
\begin{center}
\includegraphics[width=\linewidth]{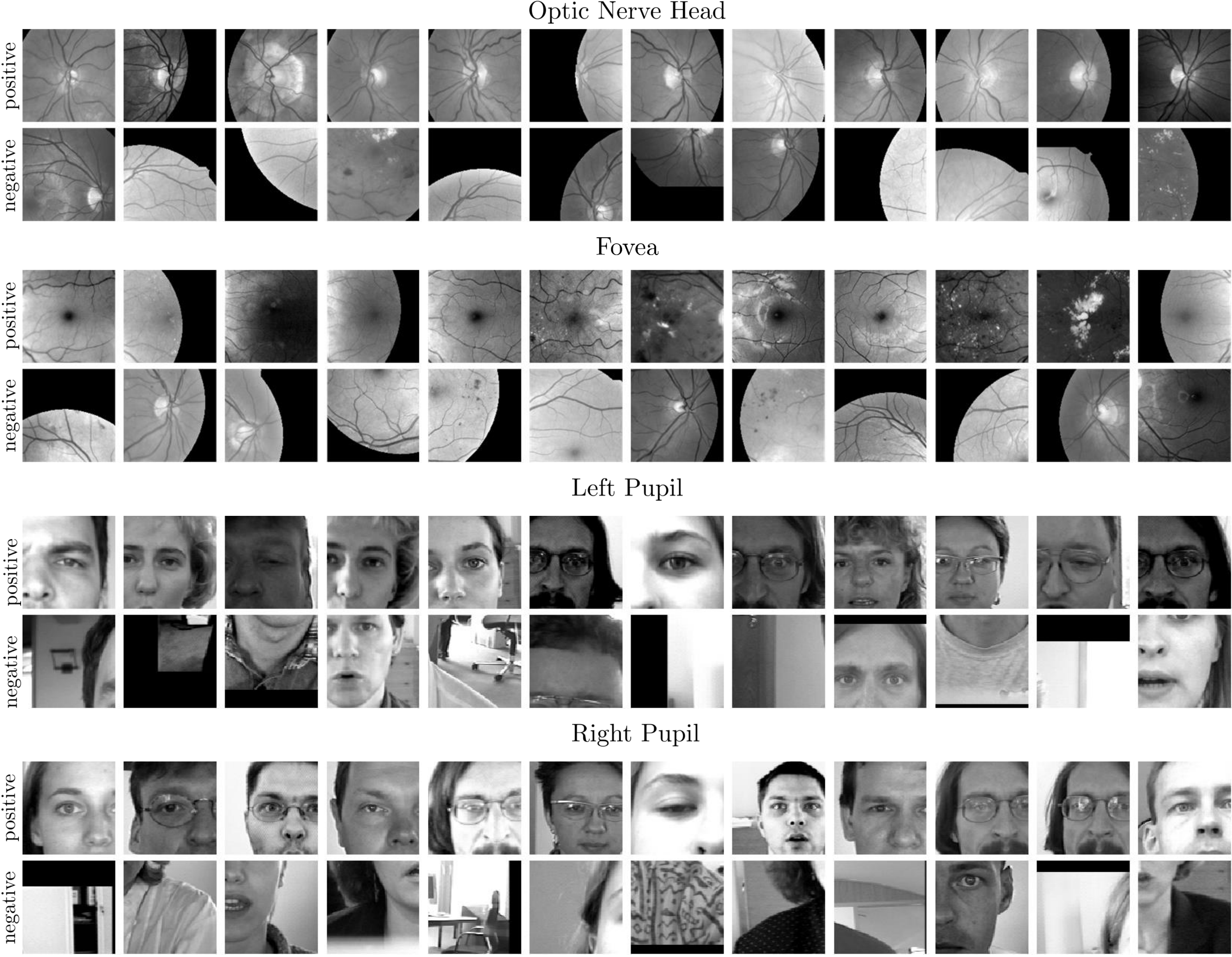}
\end{center}
\caption{A selection of positive and negative image patches $f_i$ used in the training of templates.
}
\label{fig:TrainingSamples}
\end{figure*}

\begin{table*}%[!htbp]
\centering
\caption{Average processing times. For optic nerve head detection (ONH) the average is taken over 1529 images of the TC, MESSIDOR, DRIVE and STARE database. For fovea detection the average is taken over 1408 images of the TC and MESSIDOR database. For pupil detection the average is taken over 1521 images of the BioID database.}
%\caption{Results of (combinations) of templates for optic nerve head (ONH) detection. Columns ES, TC, MESSIDOR, DRIVE, STARE and All images give the ONH detection results per databases (number of fails in parentheses).}
%\input{Tables/Detection.tex}
\begin{tabular}{lll|ll|ll}
%\toprule
\cmidrule[1.5pt]{1-7}
& \multicolumn{2}{c|}{ONH} & \multicolumn{2}{c|}{Fovea} & \multicolumn{2}{c}{\;\;Pupil (left \& right)\;\;}\\
& \multicolumn{1}{l}{$\mathbb{R}^2$} & \multicolumn{1}{l|}{$SE(2)$} & \multicolumn{1}{l}{$\mathbb{R}^2$} & \multicolumn{1}{l|}{$SE(2)$} & \multicolumn{1}{l}{$\mathbb{R}^2$} & \multicolumn{1}{l}{$SE(2)$} \\
\cmidrule{2-7}
& \multicolumn{6}{c}{{Timings (ms)}}\\
\cmidrule{2-7}

  \multicolumn{1}{l}{1. Rescaling}          & 106\hspace{3em}   & 106\hspace{3em}     & 111\hspace{3em}    & 111\hspace{3em}      & 0\hspace{3em}    & 0\hspace{3em} \\
  \multicolumn{1}{l}{2. $\mathbb{R}^2$-Processing}  & 66    & 66      & 64    & 64      & 71    & 71 \\
  \multicolumn{1}{l}{3. OS Transform}         & 0     & 108     & 0     & 108     & 0    & 121 \\
  \multicolumn{1}{l}{4. $SE(2)$-Processing}     & 0     & 5       & 0     & 5       & 0   & 6 \\
  \multicolumn{1}{l}{5. Template Matching}      & 20    & 195     & 19    & 190     & 26    & 116 \vspace{\smallspacing}\\
  \multicolumn{1}{l}{Total}             & 192   & 479     & 195   & 477     & 97    & 313 \\

\cmidrule{2-7}
& \multicolumn{6}{c}{{Combined Total Timings (ms) - $\mathbb{R}^2$ and $SE(2)$}}\\
\cmidrule{2-7}

\multicolumn{1}{c}{ }& \multicolumn{2}{c|}{497} & \multicolumn{2}{c|}{501} & \multicolumn{2}{c}{420} \\

\cmidrule{2-7}
& \multicolumn{6}{c}{{Combined Total Timings (ms) - Fovea and ONH}}\\
\cmidrule{2-7}

\multicolumn{1}{c}{ }& \multicolumn{4}{c|}{730} & \\

\cmidrule[1.5pt]{1-7}
\end{tabular}
%\begin{flushleft}
%
%$^*$\emph{$\lambda$ and or $D_{\theta\theta}$ are optimized based on the Bhattacharyya distance} $d_B$.\hspace{9.5em}
%\end{flushleft}
\label{tab:timings}
\end{table*}

\begin{table*}%[!htbp]
\centering
\caption{Success rates for optic nerve head detection ($\pm$ standard deviation, number of fails in parenthesis) with varying accuracy requirements in 5-fold cross validation. Maximum distance to ground truth location is expressed in optic disk radius $R$.}
\begin{tabular}{l|lllll}
\toprule
 & \multicolumn{5}{c}{Maximum distance to ground truth}\\
Database (\# of images)  & \;\;\;\;\;\;\;\;\;\;\;\;R/8\;\;\;\;\;\;\;\;\;\;\;\; & \;\;\;\;\;\;\;\;\;\;\;\;R/4\;\;\;\;\;\;\;\;\;\;\;\; & \;\;\;\;\;\;\;\;\;\;\;\;R/2\;\;\;\;\;\;\;\;\;\;\;\; & \;\;\;\;\;\;\;\;\;\;\;\;R\;\;\;\;\;\;\;\;\;\;\;\; & \;\;\;\;\;\;\;\;\;\;\;\;2R\;\;\;\;\;\;\;\;\;\;\;\;\\
% & & & & (baseline) &\\

\midrule

ES (SLO)\;\;\;\;\;\;\;(208) & 98.05\% {\tiny $\pm$ 2.04\%} (4) & 100.0\% {\tiny $\pm$ 0.00\%} (0) & 100.0\% {\tiny $\pm$ 0.00\%} (0) & 100.0\% {\tiny $\pm$ 0.00\%} (0) & 100.0\% {\tiny $\pm$ 0.00\%} (0)\\
TC\;\;\;\;\;\;\;\;\;\;\;\;\;\;\;\;\;(208) & 84.19\% {\tiny $\pm$ 4.34\%} (33) & 94.54\% {\tiny $\pm$ 3.51\%} (11) & 99.52\% {\tiny $\pm$ 1.06\%} (1) & 100.0\% {\tiny $\pm$ 0.00\%} (0) & 100.0\% {\tiny $\pm$ 0.00\%} (0)\\
MESSIDOR \;\;(1200) & 73.07\% {\tiny $\pm$ 3.69\%} (323) & 94.41\% {\tiny $\pm$ 1.47\%} (67) & 99.50\% {\tiny $\pm$ 0.46\%} (6) & 99.92\% {\tiny $\pm$ 0.19\%} (1) & 100.0\% {\tiny $\pm$ 0.00\%} (0)\\
DRIVE \;\;\;\;\;\;\;\;\;\,(40) & 70.84\% {\tiny $\pm$ 26.0\%} (13) & 91.69\% {\tiny $\pm$ 12.3\%} (4) & 98.18\% {\tiny $\pm$ 4.07\%} (1) & 98.18\% {\tiny $\pm$ 4.07\%} (1) & 100.0\% {\tiny $\pm$ 0.00\%} (0)\\
STARE \;\;\;\;\;\;\;\;\;\,(81) & 48.12\% {\tiny $\pm$ 10.27\%} (42) & 74.94\% {\tiny $\pm$ 6.52\%} (20) & 89.39\% {\tiny $\pm$ 8.16\%} (9) & 98.67\% {\tiny $\pm$ 2.98\%} (1) & 98.67\% {\tiny $\pm$ 2.98\%} (1)\\
&&&&&\\
All Images\;\;\;\;(1737) & 76.11\% {\tiny $\pm$ 2.58\%} (415) & 94.13\% {\tiny $\pm$ 0.79\%} (102) & 99.02\% {\tiny $\pm$ 0.26\%} (17) & 99.83\% {\tiny $\pm$ 0.26\%} (3) & 99.94\% {\tiny $\pm$ 0.13\%} (1)\\

\bottomrule
\end{tabular}
\label{tab:resultsONHAccuracy}
\end{table*}

\subsection{Training Samples}
In all three applications training samples were used to compute the templates. Positive training samples were centered around the object of interest. Negative training samples were centered around random locations in the image, but not within a certain distances to the true positive object location. In the retinal applications this distance was one optic disk radius, in the pupil detection application this was a normalized distance of 0.1 (cf. Eq.(39) of the main article). An selection of the 2D image pathes that were used in the experiments are shown in Fig.~\ref{fig:TrainingSamples}.

\subsection{Processing Pipeline, Settings and Timings}
\label{subsec:implementationdetails}
\subsubsection{Processing Pipeline}
\label{subsec:ProcessingPipeline}
In all three application the same processing pipeline was used. The pipeline can be divided into the following 5 steps:
\begin{enumerate}
\item \emph{Resizing}. Each input image is resized to a certain operating resolution and cropped to remove large regions with value 0 (outside the field of view mask in retinal images, see e.g. Fig.~\ref{fig:resultsOverviewImageONH}). The retinal images are resized such that the pixel size was approximately $40 \mu m/pix$. In the pupil detection application no rescaling or cropping was done.
\item \emph{$\mathbb{R}^2$-Processing}. In all three applications we applied a local intensity and contrast normalization step using an adaptation of \cite{Foracchia2005} which we explain below. The locally normalized image $\hat{f}$ is then mapped through an error function via $\operatorname{erf}(8 \hat{f})$ to dampen outliers.
\item \emph{Orientation score transform}. The processed image is then taken as input for an orientation score transform using Eq.~(23) of the main article. For the oriented wavelets we used cake wavelets \cite{Duits2007a,Bekkers2014} of size $[51 \times 51]$ and with angular resolution $s_\theta = \pi/12$, and with sampling $\theta$ from $0$ to $\pi$.
\item \emph{$SE(2)$-Processing}. For phase-invariant, nonlinear, left-invariant \cite{Duits2010}, and contractive \cite{Bruna2013} processing on SE(2), we work with the modulus of the complex valued orientation scores rather than with the complex-valued scores themselves (taking the modulus of quadrature filter responses is an effective technique for line detection, see e.g. Freeman et al. \cite{Freeman1991}).
\item \emph{Template Matching}. Finally we perform template matching using respectively Eqs.~(3),(4) and (5) of the main article for the $\mathbb{R}^2$ case and Eqs.~(3),(25) and (26) of the main article for the $SE(2)$ case.
\end{enumerate}

Regarding the image resolutions (step 1) we note that the average image size after rescaling was $[300\times 300]$. The average image resolutions in each database were as follows:
\begin{itemize}
\item \emph{ES (SLO)} contained images of average resolution $13.9\mu m /pix$.
\item \emph{TC} contained images of average resolution $9.4 \mu m /pix$.
\item \emph{MESSIDOR} contained images of 3 cameras with average resolutions $13.6\mu m /pix$, $9.1\mu m /pix$ and $8.6\mu m /pix$.
\item \emph{DRIVE} contained images of average resolution $21.9\mu m /pix$.
\item \emph{STARE} contained images of average resolution $17.6\mu m /pix$.
\end{itemize}

Regarding local image normalization (step 2) we note the following. Local image normalization was done using an adaptation of \cite{Foracchia2005}. The method first computes a local average and standard deviation of pixel intensities, and the image is locally normalized to zero mean and unit standard deviation. This is done via Eq.~(\ref{eq:fnormedapprox}). Then a background mask is construct by setting pixels with a larger distance than 1 standard deviation to the average (Mahalanobis distance) to 0, and other pixels to 1. This mask is then used to ignore outliers in a second computation of the local average and standard deviation. The final normalized image is again computed via Eq.~(\ref{eq:fnormedapprox}) but now with the inclusion of the background mask, see Eq.~(\ref{eq:productwithmask}).

\subsubsection{Template Settings}
In the retinal applications we used $\mathbb{R}^2$ templates of size $[N_x \times N_y] = [251 \times 251]$ which were covered by a grid of B-spline basis functions of size $[N_k \times N_l] = [51 \times 51]$, the  $SE(2)$ templates were of size $[N_x \times N_y \times N_\theta] = [251 \times 251 \times 12]$ and were covered by a grid of B-spline basis functions of size $[N_k \times N_l \times N_m] = [51 \times 51 \times 12]$.

In the pupil detection application we used $\mathbb{R}^2$ templates of size $[N_x \times N_y] = [101 \times 101]$ which were also covered by a grid of B-spline basis functions of size $[N_k \times N_l] = [51 \times 51]$, the  $SE(2)$ templates were of size $[N_x \times N_y \times N_\theta] = [101 \times 101 \times 12]$ and were also covered by a grid of B-spline basis functions of size $[N_k \times N_l \times N_m] = [51 \times 51 \times 12]$.

The regularization parameters ($\lambda$, $\mu$ and $D_{\theta\theta}$) for the different template types were automatically optimized using generalized cross validation.

\subsubsection{Timings}
We computed the average time for detecting one (or two) object(s) in an image and tabulated the results in Tab.~\ref{tab:timings}. Here we sub-divided the timings into the 5 processing steps explained in Subsec.~\ref{subsec:ProcessingPipeline}. The average (full) processing time on the retinal images was in both applications approximately $500ms$. When both the ONH and fovea are detected by the same processing pipeline the processing took $730ms$. For pupil detection the average time to detect $\emph{both}$ the left and right pupil on the \emph{full} images was $420ms$.

The retinal images were on average of size $[1230 \times 1792]$, and $[300\times300]$ after cropping and resizing. The images in the pupil detection application were not resized or cropped and were of size $[286 \times 384]$.

All experiments were performed using Wolfram \emph{Mathematica} 10.4, on a computer with an Intel Core i703612QM CPU and 8GB memory.
%Here we note that in the retinal image datasets the maximum template response always occurs at rotation $\alpha = 0$, so for the sake of reduced computation time we have set $P^{SE(2)}(\mathbf{x}) := \tilde{P}^{SE(2)}(\mathbf{x},0)$ instead of Eq.~(25) of the main article. In the pupil detection application we also make this assumption, however, we remark that detection performance could slightly be improved with the inclusion of a certain search range for $\alpha$ (see also Subsec.~\ref{subsec:pupilResults}).

\subsection{Detection Results}
In this section we provide the results for the three separate applications. A general discussion of these results can be found in the main article.

\subsubsection{Optic Nerve Head Detection}
A Table of detection performance for each type of template is provided in Tab.~1 of the main article. In Fig.~\ref{fig:resultsOverviewImageONH} we show the 3 failed cases for ONH detection, and a selection of correct ONH localizations in difficult images. In Table \ref{tab:resultsONHAccuracy} we show detection results for varying accuracy criteria. Note that detection results are typically reported for the accuracy requirement of 1 optic disk radius with the target (see also state-of-the-art comparison in Table~2 of the main article).

\begin{figure*}
\begin{center}
\includegraphics[width=\linewidth]{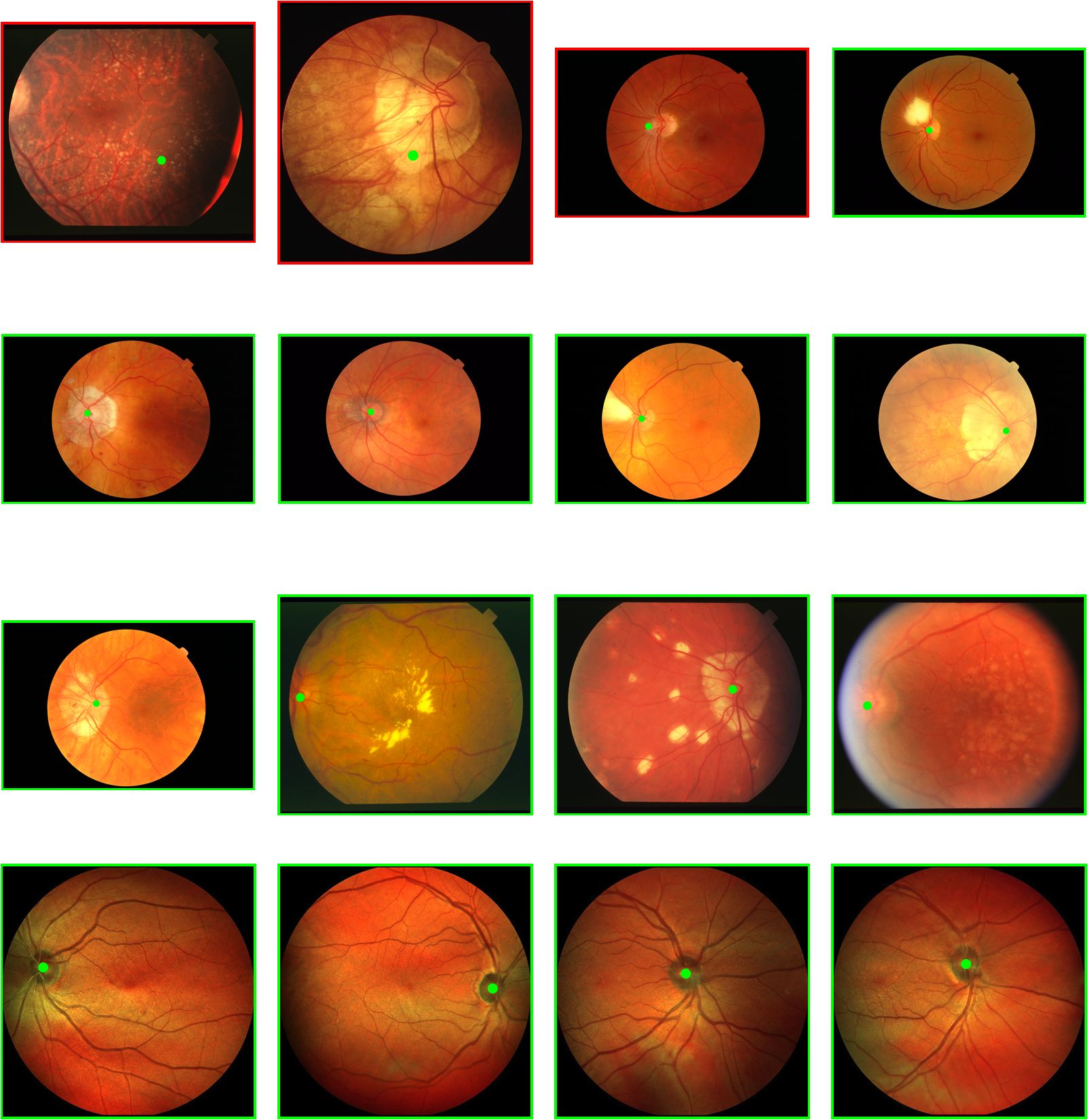}
\end{center}
\caption{Detection results of our best method for optic nerve head detection in retinal images. Successful detection are indicated with a green frame around the image, failed detections are indicated with a red frame. In the ONH detection application there were only 3 fails in a set of 1737 images.
}
\label{fig:resultsOverviewImageONH}
\end{figure*}

\subsubsection{Fovea Detection}
A Table of detection performance for each type of template is provided in Tab.~\ref{tab:resultsFovea}. In Fig.~\ref{fig:resultsOverviewImageFovea} we show next to a selection of successful detections the only 5 failed cases on images from conventional fundus (CF) cameras (TC, MESSIDOR, DRIVE, STARE), and 3 of the failed detections in images coming from an scanning laser ophthalmoscopy (SLO) camera.

As can also be read from Tab.~\ref{tab:resultsFovea}, we found that fovea detection in SLO images was significantly more difficult than fovea detection in CF images. The reason for this is that on SLO images the clear dark blob-like shape is not always present on these images. Compare for example the positive fovea patches from Fig.~\ref{fig:TrainingSamples} (where one generally sees a dark blob at the center) with the fovea locations in the bottom row of images in Figs.~\ref{fig:resultsOverviewImageONH} and \ref{fig:resultsOverviewImageFovea}.

Additionally, the ES (SLO) and CF databases are also more difficult than the MESSIDOR database for fovea detection, as these two databases contain a mix of both fovea centered and ONH centered images. The MESSIDOR database contains only fovea centered images, in which case the fovea is always located around the center of the image. Therefore, even though MESSIDOR is one of the most used databases, it might not be the most representative database for fovea detection benchmarking.

We show detection performance for a range of accuracy requirements in Table~\ref{tab:resultsFoveaAccuracy} for the different databases used in our experiments, and in Table~\ref{tab:resultsFoveaAccuracyStateOfArt} a comparison to the state of the art. There we see that for the stricter requirement of detection within half an optic disk radius our method still outperforms the state of the art. We also see that with further decreasing the acceptance distance ($R/4$ or lower) none of the methods provided acceptable results.

\begin{table*}%[!htbp]
\centering
\caption{Average template matching results ($\pm$ standard deviation) for fovea detection in 5-fold cross validation, number of failed detections in parentheses.}
%\caption{Results of (combinations) of templates for optic nerve head (ONH) detection. Columns ES, TC, MESSIDOR, DRIVE, STARE and All images give the ONH detection results per databases (number of fails in parentheses).}
%\input{Tables/Detection.tex}
\begin{tabular}{l|lll|l}
\toprule
%\multicolumn{2}{c|}{}& \multicolumn{1}{c|}{SLO} & \multicolumn{4}{c|}{CF} & \multicolumn{1}{}{}\\
%\multicolumn{2}{c|}{} & \multicolumn{1}{c|}{} & \multicolumn{4}{c|}{} & \multicolumn{1}{}{}\\
\multicolumn{1}{l|}{Template} & ES (SLO) & TC & MESSIDOR & All Images\\
ID & 208 & 208 & 1200 & 1616\\

\midrule
\multicolumn{5}{c}{{$\mathbb{R}^2$ templates}}\\
\midrule

  $A_{\mathbb{R}^2}$ & 76.36\% {\tiny $\pm$ 6.79\%} (49)     & 98.24\% {\tiny $\pm$ 2.74\%} (3)     & 98.41\% {\tiny $\pm$ 0.22\%} (19)   & \cellcolor{rowcolor}95.60\% {\tiny $\pm$ 0.98\%} (71) \vspace{\smallspacing}\\

  $B_{lin:\mathbb{R}^2}$ & 23.50\% {\tiny $\pm$ 3.81\%} (159)     & 31.66\% {\tiny $\pm$ 9.03\%} (142)     & 51.19\% {\tiny $\pm$ 5.97\%} (587)   & 45.07\% {\tiny $\pm$ 3.33\%} (888) \\

  $C_{lin:\mathbb{R}^2}$ & 45.65\% {\tiny $\pm$ 8.61\%} (113)     & 98.24\% {\tiny $\pm$ 2.74\%} (3)     & 98.59\% {\tiny $\pm$ 0.36\%} (17)   & 91.77\% {\tiny $\pm$ 1.26\%} (133) \\

  $D_{lin:\mathbb{R}^2}$ & 44.21\% {\tiny $\pm$ 4.62\%} (116)     & 99.49\% {\tiny $\pm$ 1.14\%} (1)     & 98.84\% {\tiny $\pm$ 0.31\%} (14)   & 91.90\% {\tiny $\pm$ 0.59\%} (131) \\

  $E_{lin:\mathbb{R}^2}$ & 46.10\% {\tiny $\pm$ 8.11\%} (112)     & 98.86\% {\tiny $\pm$ 1.57\%} (2)     & 98.67\% {\tiny $\pm$ 0.34\%} (16)   & 91.95\% {\tiny $\pm$ 1.18\%} \cellcolor{rowcolor}(130) \vspace{\smallspacing}\\

  $B_{log:\mathbb{R}^2}$ & 1.43\% {\tiny $\pm$ 1.31\%} (205)     & 10.27\% {\tiny $\pm$ 5.09\%} (185)     & 20.07\% {\tiny $\pm$ 3.00\%} (959)   & 16.53\% {\tiny $\pm$ 2.52\%} (1349) \\

  $C_{log:\mathbb{R}^2}$ & 9.59\% {\tiny $\pm$ 3.74\%} (188)     & 70.30\% {\tiny $\pm$ 8.57\%} (61)     & 77.61\% {\tiny $\pm$ 4.64\%} (267)   & 68.06\% {\tiny $\pm$ 3.53\%} (516) \\

  $D_{log:\mathbb{R}^2}$ & 11.48\% {\tiny $\pm$ 4.70\%} (184)     & 83.47\% {\tiny $\pm$ 7.80\%} (32)     & 88.22\% {\tiny $\pm$ 2.81\%} (141)   & 77.90\% {\tiny $\pm$ 2.00\%} \cellcolor{rowcolor}(357) \\

  $E_{log:\mathbb{R}^2}$ & 2.86\% {\tiny $\pm$ 2.62\%} (202)     & 79.68\% {\tiny $\pm$ 7.92\%} (40)     & 84.79\% {\tiny $\pm$ 5.16\%} (181)   & 73.82\% {\tiny $\pm$ 2.62\%} (423) \\

\midrule
\multicolumn{5}{c}{{$SE(2)$ templates}}\\
\midrule

  $A_{SE(2)}$ & 67.81\% {\tiny $\pm$ 4.69\%} (67)     & 79.13\% {\tiny $\pm$ 9.11\%} (40)     & 98.25\% {\tiny $\pm$ 0.68\%} (21)   & \cellcolor{rowcolor}92.08\% {\tiny $\pm$ 0.84\%} (128) \vspace{\smallspacing}\\

  $B_{lin:SE(2)}$ & 83.19\% {\tiny $\pm$ 2.76\%} (35)     & 71.53\% {\tiny $\pm$ 7.36\%} (58)     & 91.31\% {\tiny $\pm$ 0.68\%} (104)   & 87.81\% {\tiny $\pm$ 1.25\%} (197) \\

  $C_{lin:SE(2)}$ & 83.65\% {\tiny $\pm$ 3.18\%} (34)     & 84.13\% {\tiny $\pm$ 6.25\%} (32)     & 98.23\% {\tiny $\pm$ 1.04\%} (21)   & \cellcolor{rowcolor}94.62\% {\tiny $\pm$ 0.36\%} (87) \\

  $D_{lin:SE(2)}$ & 73.57\% {\tiny $\pm$ 4.71\%} (55)     & 83.69\% {\tiny $\pm$ 6.83\%} (33)     & 97.88\% {\tiny $\pm$ 1.17\%} (25)   & 93.01\% {\tiny $\pm$ 1.09\%} (113) \\

  $E_{lin:SE(2)}$ & 77.83\% {\tiny $\pm$ 4.29\%} (46)     & 84.88\% {\tiny $\pm$ 6.69\%} (30)     & 98.22\% {\tiny $\pm$ 1.23\%} (21)   & 94.00\% {\tiny $\pm$ 0.93\%} (97) \vspace{\smallspacing}\\

  $B_{log:SE(2)}$ & 75.49\% {\tiny $\pm$ 5.73\%} (51)     & 60.80\% {\tiny $\pm$ 5.68\%} (80)     & 92.79\% {\tiny $\pm$ 1.98\%} (86)   & 86.56\% {\tiny $\pm$ 2.20\%} (217) \\

  $C_{log:SE(2)}$ & 79.33\% {\tiny $\pm$ 6.57\%} (43)     & 70.87\% {\tiny $\pm$ 10.28\%} (59)     & 96.90\% {\tiny $\pm$ 0.71\%} (37)   & 91.39\% {\tiny $\pm$ 1.36\%} \cellcolor{rowcolor}(139) \\

  $D_{log:SE(2)}$ & 62.09\% {\tiny $\pm$ 6.66\%} (79)     & 72.57\% {\tiny $\pm$ 8.59\%} (54)     & 96.64\% {\tiny $\pm$ 1.05\%} (40)   & 89.30\% {\tiny $\pm$ 0.63\%} (173) \\

  $E_{log:SE(2)}$ & 68.34\% {\tiny $\pm$ 8.59\%} (66)     & 72.20\% {\tiny $\pm$ 8.53\%} (55)     & 96.57\% {\tiny $\pm$ 0.96\%} (41)   & 89.98\% {\tiny $\pm$ 1.25\%} (162) \\

\midrule
\multicolumn{5}{c}{{Template combinations (sorted on performance)}}\\
\midrule

  $C_{lin:\mathbb{R}^2}+C_{log:SE(2)}$ & 97.17\% {\tiny $\pm$ 3.01\%} (6)     & 99.17\% {\tiny $\pm$ 1.13\%} (2)     & 99.74\% {\tiny $\pm$ 0.38\%} (3)   & \cellcolor{rowcolor}99.32\% {\tiny $\pm$ 0.26\%} (11) \\%\vspace{\smallspacing}\\

   \hspace{-0.8em}$^*$ $A_{\mathbb{R}^2} \;\;\;\;\; +C_{lin:SE(2)}$ & 98.08\% {\tiny $\pm$ 2.03\%} (4)     & 98.07\% {\tiny $\pm$ 1.95\%} (4)     & 99.68\% {\tiny $\pm$ 0.33\%} (4)   & 99.26\% {\tiny $\pm$ 0.47\%} (12) \\%\vspace{\smallspacing}\\

  $E_{lin:\mathbb{R}^2}+C_{log:SE(2)}$ & 96.20\% {\tiny $\pm$ 3.15\%} (8)     & 99.17\% {\tiny $\pm$ 1.13\%} (2)     & 99.75\% {\tiny $\pm$ 0.23\%} (3)   & 99.20\% {\tiny $\pm$ 0.35\%} (13) \\

  $E_{lin:\mathbb{R}^2}+C_{lin:SE(2)}$ & 96.65\% {\tiny $\pm$ 2.13\%} (7)     & 99.17\% {\tiny $\pm$ 1.13\%} (2)     & 99.66\% {\tiny $\pm$ 0.36\%} (4)   & 99.19\% {\tiny $\pm$ 0.42\%} (13) \\%\vspace{\smallspacing}\\

  $C_{lin:\mathbb{R}^2}+C_{lin:SE(2)}$ & 97.14\% {\tiny $\pm$ 1.97\%} (6)     & 98.78\% {\tiny $\pm$ 1.78\%} (3)     & 99.58\% {\tiny $\pm$ 0.31\%} (5)   & 99.13\% {\tiny $\pm$ 0.40\%} (14) \\

  $A_{\mathbb{R}^2} \;\;\;\;\; +E_{lin:SE(2)}$ & 97.59\% {\tiny $\pm$ 1.73\%} (5)     & 98.07\% {\tiny $\pm$ 1.95\%} (4)     & 99.59\% {\tiny $\pm$ 0.28\%} (5)   & 99.13\% {\tiny $\pm$ 0.25\%} (14) \\%\vspace{\smallspacing}\\

  $E_{lin:\mathbb{R}^2}+E_{lin:SE(2)}$ & 96.16\% {\tiny $\pm$ 2.76\%} (8)     & 99.17\% {\tiny $\pm$ 1.13\%} (2)     & 99.58\% {\tiny $\pm$ 0.31\%} (5)   & 99.07\% {\tiny $\pm$ 0.38\%} (15) \\%\vspace{\smallspacing}\\

  $E_{lin:\mathbb{R}^2}+D_{lin:SE(2)}$ & 95.71\% {\tiny $\pm$ 3.07\%} (9)     & 99.17\% {\tiny $\pm$ 1.13\%} (2)     & 99.58\% {\tiny $\pm$ 0.31\%} (5)   & 99.01\% {\tiny $\pm$ 0.40\%} (16) \\

  $C_{lin:\mathbb{R}^2}+E_{lin:SE(2)}$ & 96.16\% {\tiny $\pm$ 2.76\%} (8)     & 98.78\% {\tiny $\pm$ 1.78\%} (3)     & 99.58\% {\tiny $\pm$ 0.31\%} (5)   & 99.01\% {\tiny $\pm$ 0.51\%} (16) \\

  $A_{\mathbb{R}^2} \;\;\;\;\; +C_{log:SE(2)}$ & 96.65\% {\tiny $\pm$ 2.13\%} (7)     & 98.07\% {\tiny $\pm$ 1.95\%} (4)     & 99.58\% {\tiny $\pm$ 0.42\%} (5)   & 99.01\% {\tiny $\pm$ 0.26\%} (16) \\

  \multicolumn{1}{l|}{{\;\;\;\;\;\;\;\;\;\;\;\;\;...}} & \multicolumn{3}{c|}{{...}} & \multicolumn{1}{c}{{...}}\\

  \hspace{-0.8em}$^\dagger$ $A_{\mathbb{R}^2} \;\;\;\;\; +A_{SE(2)}$ & 92.85 \% {\tiny $\pm$ 4.68\%} (15)     & 95.84\% {\tiny $\pm$ 2.58\%} (8)     & 99.58\% {\tiny $\pm$ 0.30\%} (5)   & 98.27\% {\tiny $\pm$ 0.70\%} (28) \\

  \multicolumn{1}{l|}{{\;\;\;\;\;\;\;\;\;\;\;\;\;...}} & \multicolumn{3}{c|}{{...}} & \multicolumn{1}{c}{{...}}\\
\bottomrule
\multicolumn{5}{l}{$^*$\emph{Best template combination that does not rely on logistic regression.}}\\
\multicolumn{5}{l}{$^\dagger$\emph{Best template combination that does not rely on template optimization.}}
\end{tabular}
%\begin{flushleft}
%
%$^*$\emph{$\lambda$ and or $D_{\theta\theta}$ are optimized based on the Bhattacharyya distance} $d_B$.\hspace{9.5em}
%\end{flushleft}
\label{tab:resultsFovea}
\end{table*}

\begin{table*}%[!htbp]
\centering
\caption{Success rates for fovea detection ($\pm$ standard deviation, number of fails in parenthesis) with varying accuracy requirements in 5-fold cross validation. Maximum distance to ground truth location is expressed in optic disk radius $R$.}
\begin{tabular}{l|lllll}
\toprule
 & \multicolumn{5}{c}{Maximum distance to ground truth}\\
Database (\# of images)  & \;\;\;\;\;\;\;\;\;\;\;\;R/8\;\;\;\;\;\;\;\;\;\;\;\; & \;\;\;\;\;\;\;\;\;\;\;\;R/4\;\;\;\;\;\;\;\;\;\;\;\; & \;\;\;\;\;\;\;\;\;\;\;\;R/2\;\;\;\;\;\;\;\;\;\;\;\; & \;\;\;\;\;\;\;\;\;\;\;\;R\;\;\;\;\;\;\;\;\;\;\;\; & \;\;\;\;\;\;\;\;\;\;\;\;2R\;\;\;\;\;\;\;\;\;\;\;\;\\

\midrule

ES (SLO)\;\;\;\;\;\;\;(208) & 66.91\% {\tiny $\pm$ 4.64\%} (69) & 92.85\% {\tiny $\pm$ 3.16\%} (15) & 94.74\% {\tiny $\pm$ 1.93\%} (11) & 97.17\% {\tiny $\pm$ 3.01\%} (6) & 97.66\% {\tiny $\pm$ 3.28\%} (5)\\
TC\;\;\;\;\;\;\;\;\;\;\;\;\;\;\;\;\;(208) & 49.51\% {\tiny $\pm$ 4.07\%} (106) & 80.33\% {\tiny $\pm$ 3.22\%} (40) & 95.41\% {\tiny $\pm$ 1.77\%} (9) & 99.17\% {\tiny $\pm$ 1.13\%} (2) & 99.61\% {\tiny $\pm$ 0.88\%} (1)\\
MESSIDOR \;\;(1200) & 61.81\% {\tiny $\pm$ 2.64\%} (459) & 90.56\% {\tiny $\pm$ 1.31\%} (113) & 98.07\% {\tiny $\pm$ 0.87\%} (23) & 99.74\% {\tiny $\pm$ 0.38\%} (3) & 100.0\% {\tiny $\pm$ 0.00\%} (0)\\
&&&&&\\
All Images\;\;\;\;(1616) & 60.78\% {\tiny $\pm$ 1.84\%} (634) & 89.60\% {\tiny $\pm$ 0.80\%} (168) & 97.34\% {\tiny $\pm$ 0.65\%} (43) & 99.32\% {\tiny $\pm$ 0.26\%} (11) & 99.63\% {\tiny $\pm$ 0.40\%} (6)\\

\bottomrule
\end{tabular}
\label{tab:resultsFoveaAccuracy}
\end{table*}

\begin{table*}%[!htbp]
\centering
\caption{Success rates for fovea detection (number of fails in parenthesis) with varying accuracy requirements; a comparison to literature using the MESSIDOR database. Maximum distance to ground truth location is expressed in optic disk radius $R$.}
\begin{tabular}{l|lllll}
\toprule
 & \multicolumn{5}{c}{Maximum distance to ground truth}\\
Method  & \;\;\;\;\;\;\;\;\;\;\;\;R/8\;\;\;\;\;\;\;\;\;\;\;\; & \;\;\;\;\;\;\;\;\;\;\;\;R/4\;\;\;\;\;\;\;\;\;\;\;\; & \;\;\;\;\;\;\;\;\;\;\;\;R/2\;\;\;\;\;\;\;\;\;\;\;\; & \;\;\;\;\;\;\;\;\;\;\;\;R\;\;\;\;\;\;\;\;\;\;\;\; & \;\;\;\;\;\;\;\;\;\;\;\;2R\;\;\;\;\;\;\;\;\;\;\;\;\\

\midrule

Niemeijer {\tiny et al. \cite{Niemeijer2009,GegundezArias2013} }& 75.67\% (292)& 93.50\% (78) & 96.83\% (38) & 97.92\% (25) & - \\
Yu et al. {\tiny et al. \cite{Yu2011}} & - & - & 95.00\% (60) & - & - \\
Gegundez-Arias {\tiny et al. \cite{GegundezArias2013}}   & 80.42\% (235) & 93.90\% (73)& 96.08\% (47)& 96.92\% (37)& 97.83\% (26)\\
Giachetti {\tiny et al. \cite{Giachetti2013}} & - & - & - & 99.10\% (11) & - \\
Aquino {\tiny \cite{Aquino2014}} & - & - & - & 98.20\% (21) & - \\
&&&&&\\
Proposed & 61.81\% (459) & 90.56\% (113) & 98.07\% (23) & 99.74\% (3) & 100.0\% (0)\\

\bottomrule
\end{tabular}
\label{tab:resultsFoveaAccuracyStateOfArt}
\end{table*}

\begin{figure*}
\begin{center}
\includegraphics[width=\linewidth]{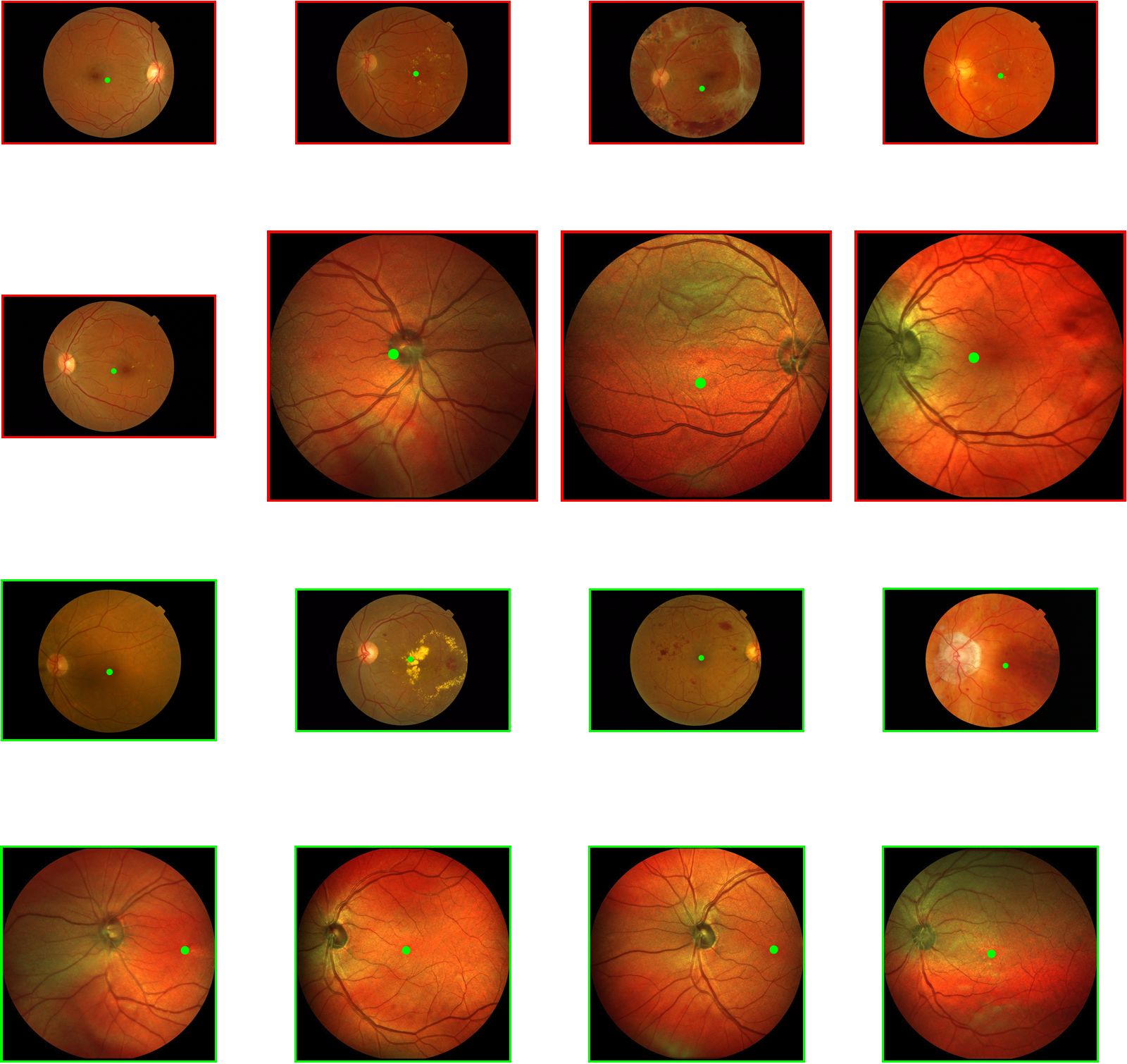}
\end{center}
\caption{Detection results of our best method for fovea detection in retinal images. Successful detection are indicated with a green frame around the image, failed detections are indicated with a red frame. In the fovea detection application there were only 5 fails in a set of 1408 conventional fundus (CF) camera images. Out of the 208 scanning laser ophthalmoscopy (SLO) images there were 6 fails, 3 of them are shown in this figure.
}
\label{fig:resultsOverviewImageFovea}
\end{figure*}

%\clearpage
%\clearpage
\subsubsection{Pupil Detection}
\label{subsec:pupilResults}

A Table of detection performance for each type of template is provided in Tab.~\ref{tab:resultsPupil}. In Fig.~\ref{fig:resultsOverviewImagePupil} we show a selection of failed and successful detections. By inspection of the failed cases we found that a main source of failed detections was due to rotations of the head. As stated in the previous section \ref{subsec:implementationdetails} we did not employ a rotation invariant detection scheme. Doing so might improve the results. Other failed detections could be attributed to closed eyes, reflection of glasses, distracting background objects and different scales (object distance to camera).

\begin{table}
\centering
\caption{Average template matching results ($\pm$ standard deviation) for pupil detection in 5-fold cross validation, number of failed detections in parentheses. A successful detection has a normalized error $e\le 0.1$.}
\begin{tabular}{l|ll}
\toprule
\multicolumn{1}{l|}{Template} & BioID (Full image)\;\;\; & BioID (Periocular image)\\
ID & 1521 & 1521  \\

\midrule
\multicolumn{3}{c}{{$\mathbb{R}^2$ templates}}\\
\midrule

  $A_{\mathbb{R}^2}$ & \cellcolor{rowcolor}41.03\% {\tiny $\pm$ 1.45\%} (897)     & \cellcolor{rowcolor}59.70\% {\tiny $\pm$ 1.52\%} (613)  \vspace{\smallspacing}\\

  $B_{lin:\mathbb{R}^2}$ & 0.00\% \hspace{0.5em}{\tiny $\pm$ 0.00\%} (1521)     & 3.62\% \hspace{0.5em}{\tiny $\pm$ 1.09\%} (1466)  \\

  $C_{lin:\mathbb{R}^2}$ & \cellcolor{rowcolor}12.95\% {\tiny $\pm$ 2.22\%} (1324)     & 67.26\% {\tiny $\pm$ 2.55\%} (498)  \\

  $D_{lin:\mathbb{R}^2}$ & 8.28\% \hspace{0.5em}{\tiny $\pm$ 1.80\%} (1395)     & \cellcolor{rowcolor}75.68\% {\tiny $\pm$ 2.33\%} (370)  \\

  $E_{lin:\mathbb{R}^2}$ & 11.51\% {\tiny $\pm$ 2.25\%} (1346)     & 71.47\% {\tiny $\pm$ 2.76\%} (434)  \vspace{\smallspacing}\\

  $B_{log:\mathbb{R}^2}$ & 0.00\%\hspace{0.5em} {\tiny $\pm$ 0.00\%} (1521)     & 0.00\% \hspace{0.5em}{\tiny $\pm$ 0.00\%} (1521)  \\

  $C_{log:\mathbb{R}^2}$ & \cellcolor{rowcolor}12.89\% {\tiny $\pm$ 2.06\%} (1325)     & \cellcolor{rowcolor}39.91\% {\tiny $\pm$ 3.37\%} (914)  \\

  $D_{log:\mathbb{R}^2}$ & 1.84\% \hspace{0.5em}{\tiny $\pm$ 0.95\%} (1493)     & 22.09\% {\tiny $\pm$ 2.37\%} (1185)  \\

  $E_{log:\mathbb{R}^2}$ & 10.39\% {\tiny $\pm$ 2.26\%} (1363)     & 37.21\% {\tiny $\pm$ 4.37\%} (955)  \\

\midrule
\multicolumn{3}{c}{{$SE(2)$ templates}}\\
\midrule

  $A_{SE(2)}$ & \cellcolor{rowcolor}57.72\% {\tiny $\pm$ 1.68\%} (643)     & \cellcolor{rowcolor}75.34\% {\tiny $\pm$ 1.31\%} (375)  \vspace{\smallspacing}\\

  $B_{lin:SE(2)}$ & 8.74\% \hspace{0.5em}{\tiny $\pm$ 2.00\%} (1388)     & 41.81\% {\tiny $\pm$ 5.04\%} (885)  \\

  $C_{lin:SE(2)}$ & 84.61\% {\tiny $\pm$ 4.19\%} (234)     & 86.78\% {\tiny $\pm$ 3.68\%} (201)  \\

  $D_{lin:SE(2)}$ & \cellcolor{rowcolor}85.53\% {\tiny $\pm$ 3.44\%} (220)     & \cellcolor{rowcolor}87.18\% {\tiny $\pm$ 3.71\%} (195)  \\

  $E_{lin:SE(2)}$ & 85.47\% {\tiny $\pm$ 3.82\%} (221)     & 87.11\% {\tiny $\pm$ 3.87\%} (196)  \vspace{\smallspacing}\\

  $B_{log:SE(2)}$ & 0.00\% \hspace{0.5em}{\tiny $\pm$ 0.00\%} (1521)     & 0.13\%\hspace{0.5em} {\tiny $\pm$ 0.29\%} (1519)  \\

  $C_{log:SE(2)}$ & 86.52\% \cellcolor{rowcolor}{\tiny $\pm$ 0.77\%} (205)     & \cellcolor{rowcolor}93.95\% {\tiny $\pm$ 1.33\%} (92)  \\

  $D_{log:SE(2)}$ & 75.21\% {\tiny $\pm$ 2.18\%} (377)     & 89.48\% {\tiny $\pm$ 2.27\%} (160)  \\

  $E_{log:SE(2)}$ & 83.30\% {\tiny $\pm$ 1.68\%} (254)     & 92.77\% {\tiny $\pm$ 1.02\%} (110)  \\

\midrule
\multicolumn{3}{c}{{Template combinations (sorted on performance full image)}}\\
\midrule

  \hspace{-0.6em}$^*$$C_{lin:\mathbb{R}^2} + E_{lin:SE(2)}$ & \cellcolor{rowcolor}93.49\% {\tiny $\pm$ 1.49\%} (99)     & 95.60\% {\tiny $\pm$ 1.46\%} (67)  \\

  $C_{lin:\mathbb{R}^2} + D_{lin:SE(2)}$ & 93.16\% {\tiny $\pm$ 1.54\%} (104)     & 95.00\% {\tiny $\pm$ 1.15\%} (76)  \\

  $E_{lin:\mathbb{R}^2} + E_{lin:SE(2)}$ & 93.10\% {\tiny $\pm$ 1.04\%} (105)     & 95.59\% {\tiny $\pm$ 0.89\%} (67)  \\

  $E_{lin:\mathbb{R}^2} + D_{lin:SE(2)}$ & 92.97\% {\tiny $\pm$ 1.62\%} (107)     & 95.27\% {\tiny $\pm$ 1.31\%} (72)  \\

  $C_{lin:\mathbb{R}^2} + C_{lin:SE(2)}$ & 92.70\% {\tiny $\pm$ 1.41\%} (111)     & 95.33\% {\tiny $\pm$ 0.97\%} (71)  \\

  $E_{lin:\mathbb{R}^2} + C_{lin:SE(2)}$ & 92.64\% {\tiny $\pm$ 0.94\%} (112)     & 95.33\% {\tiny $\pm$ 0.94\%} (71)  \\

  $D_{lin:\mathbb{R}^2} + D_{lin:SE(2)}$ & 92.51\% {\tiny $\pm$ 0.96\%} (114)     & 95.79\% {\tiny $\pm$ 0.82\%} (64)  \\

  $D_{lin:\mathbb{R}^2} + E_{lin:SE(2)}$ & 92.24\% {\tiny $\pm$ 1.23\%} (118)     & 95.86\% {\tiny $\pm$ 0.89\%} (63)  \\

  $E_{log:\mathbb{R}^2} + D_{lin:SE(2)}$ & 92.11\% {\tiny $\pm$ 2.26\%} (120)     & 93.23\% {\tiny $\pm$ 1.93\%} (103)  \\

  $D_{lin:\mathbb{R}^2} + C_{log:SE(2)}$ & 92.05\% {\tiny $\pm$ 1.52\%} (121)     & 95.14\% {\tiny $\pm$ 0.78\%} (74)  \\

  \multicolumn{1}{l|}{{\;\;\;\;\;\;\;\;\;\;\;\;\;...}} & \multicolumn{2}{c}{{...\;\;\;\;\;\;\;\;\;\;}}\\

  %\hspace{-0.6em}$^\dagger$ $A_{\mathbb{R}^2} \;\;\;\;\,+ A_{SE(2)}$ & 61.34\% {\tiny $\pm$ 1.54\%} (588)     & 68.18\% {\tiny $\pm$ 1.25\%} (484)  \\

\midrule
\multicolumn{3}{c}{{Template combinations (sorted on performance periocular image)}}\\
\midrule

  \hspace{-0.6em}$^*$$D_{lin:\mathbb{R}^2} + E_{lin:SE(2)}$ & 92.24\% {\tiny $\pm$ 1.23\%} (118)     &\cellcolor{rowcolor}95.86\% {\tiny $\pm$ 0.89\%} (63)  \\

  $D_{lin:\mathbb{R}^2} + D_{lin:SE(2)}$ & 92.51\% {\tiny $\pm$ 0.96\%} (114)     & 95.79\% {\tiny $\pm$ 0.82\%} (64)  \\

  $D_{lin:\mathbb{R}^2} + C_{lin:SE(2)}$ & 91.52\% {\tiny $\pm$ 1.25\%} (129)     & 95.73\% {\tiny $\pm$ 0.77\%} (65)  \\

  $E_{lin:\mathbb{R}^2} + E_{lin:SE(2)}$ & 93.10\% {\tiny $\pm$ 1.04\%} (105)     & 95.59\% {\tiny $\pm$ 0.89\%} (67)  \\

  $C_{lin:\mathbb{R}^2} + E_{lin:SE(2)}$ & 93.49\% {\tiny $\pm$ 1.49\%} (99)     & 95.60\% {\tiny $\pm$ 1.46\%} (67)  \\

  $E_{lin:\mathbb{R}^2} + C_{lin:SE(2)}$ & 92.64\% {\tiny $\pm$ 0.94\%} (112)     & 95.33\% {\tiny $\pm$ 0.94\%} (71)  \\

  $C_{lin:\mathbb{R}^2} + C_{lin:SE(2)}$ & 92.70\% {\tiny $\pm$ 1.41\%} (111)     & 95.33\% {\tiny $\pm$ 0.97\%} (71)  \\

  $E_{lin:\mathbb{R}^2} + D_{lin:SE(2)}$ & 92.97\% {\tiny $\pm$ 1.62\%} (107)     & 95.27\% {\tiny $\pm$ 1.31\%} (72)  \\

  $D_{lin:\mathbb{R}^2} + E_{log:SE(2)}$ & 91.72\% {\tiny $\pm$ 1.23\%} (126)     & 95.27\% {\tiny $\pm$ 0.79\%} (72)  \\

  $D_{lin:\mathbb{R}^2} + C_{log:SE(2)}$ & 92.05\% {\tiny $\pm$ 1.52\%} (121)     & 95.14\% {\tiny $\pm$ 0.78\%} (74)  \\

  \multicolumn{1}{l|}{{\;\;\;\;\;\;\;\;\;\;\;\;\;...}} & \multicolumn{2}{c}{{...\;\;\;\;\;\;\;\;\;\;}}\\

  \hspace{-0.6em}$^\dagger$ $A_{\mathbb{R}^2} \;\;\;\;\,+ A_{SE(2)}$ & 61.34\% {\tiny $\pm$ 1.54\%} (588)     & 68.18\% {\tiny $\pm$ 1.25\%} (484)  \\

  \multicolumn{1}{l|}{{\;\;\;\;\;\;\;\;\;\;\;\;\;...}} & \multicolumn{2}{c}{{...\;\;\;\;\;\;\;\;\;\;}}\\

\bottomrule

  \multicolumn{3}{l}{$^*$\emph{Best template combination that does not rely on logistic regression.}}\\
  \multicolumn{3}{l}{$^\dagger$\emph{Best template combination that does not rely on template optimization.}}
\end{tabular}
\label{tab:resultsPupil}
\end{table}

\begin{figure*}
\begin{center}
\includegraphics[width=\linewidth]{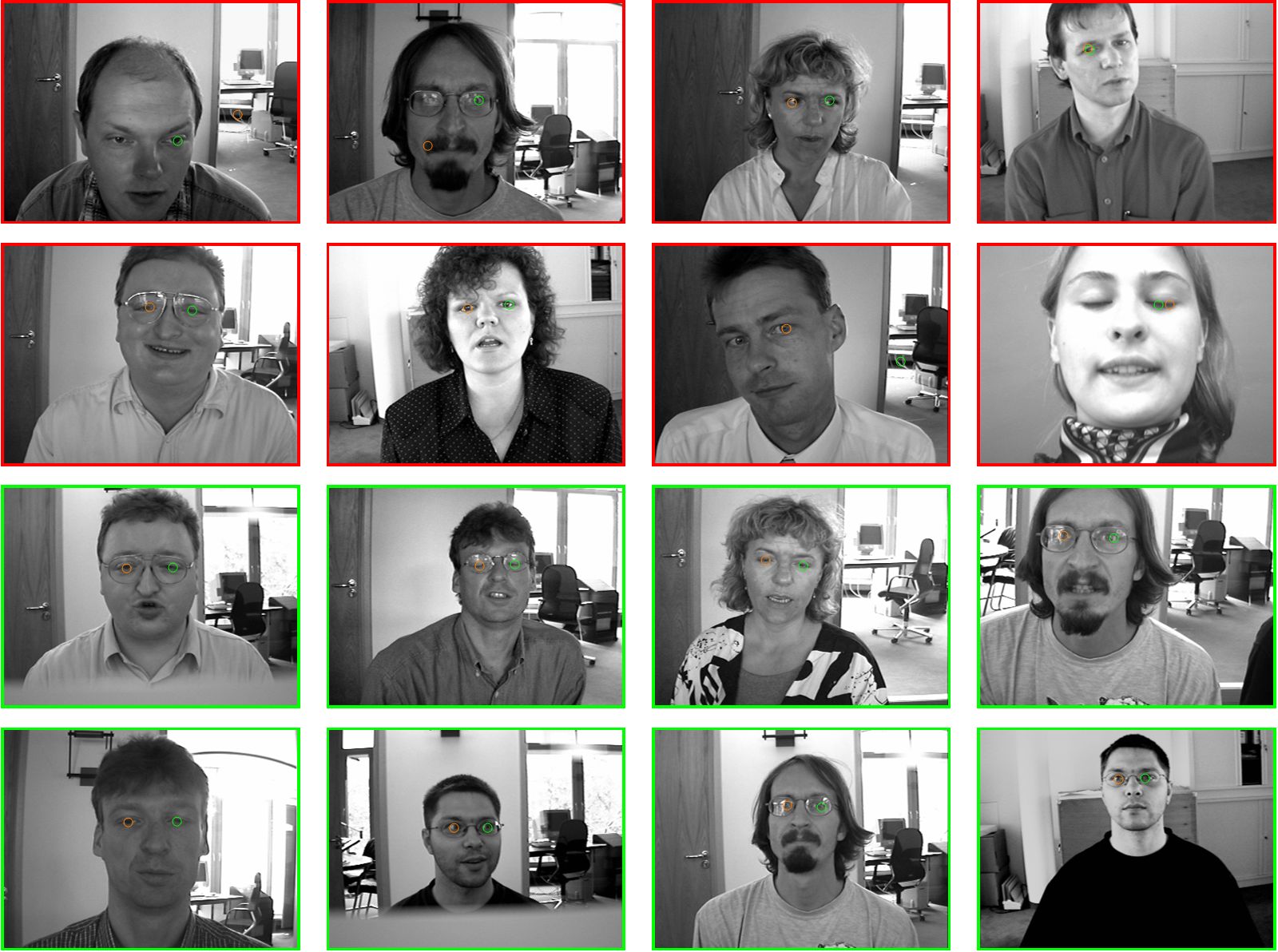}
\end{center}
\caption{Detection results of our best method for pupil detection. Successful detection are indicated with a green frame around the image, failed detections are indicated with a red frame.
}
\label{fig:resultsOverviewImagePupil}
\end{figure*}

\section{Rotation-Scale Invariant Matching}
\subsection{A Basic Extension}
\label{subsec:basicExtension}
The extension to rotation and scale invariant object localization of the 2D cross-correlation based template matching approach, described in Eqs.~(3)-(5) of the main article, is as follows. For the linear potential function (Eq.~(4) of the main article) we can define
%The cross-correlation based tempalte matching approach described by Eqs.~(??)-(??) of the main article can be extended to rotation and scale invariant in a straightforward faobject localization as follows
%A straight forward extension of the template matching method described in the main article to scale and rotation invariant matching can be obtained in the 2D case as follows
%By not only considering translations by $\mathbf{x}$ of template $t$ via the translation operator $\mathcal{T}_\mathbf{x}$, but also including scaling by a factor $a$ via a scaling operator $\mathcal{S}_a$ and rotation by $\alpha$ via a rotation $\mathcal{R}_\alpha$ as follows:
\begin{equation}
\label{eq:linearFunctionalInv}
P_{lin,inv}^{\mathbb{R}^2}(\mathbf{x}) := \underset{\begin{array}{c}a\in[a_-,a_+],\\\alpha \in [0,2\pi)\end{array}}{\operatorname{max}}( \mathcal{T}_\mathbf{x} \mathcal{S}_a \mathcal{R}_\alpha \;t , f)_{\mathbb{L}_2(\mathbb{R}^2)},
\end{equation}
and for the logistic regression case (Eq.~(5) of the main article) we define
\begin{equation}
\label{eq:logisticFunctionalInv}
%\begin{aligned}
P_{log,inv}^{\mathbb{R}^2}(\mathbf{x})  := \underset{\begin{array}{c}a\in[a_-,a_+],\\\alpha \in [0,2\pi)\end{array}}{\operatorname{max}} \sigma \left( ( \mathcal{T}_\mathbf{x} \mathcal{S}_a \mathcal{R}_\alpha \;t , f)_{\mathbb{L}_2(\mathbb{R}^2)} \right),
%\end{aligned}
\end{equation}
with $\sigma$ the logistic sigmoid function defined in Eq.~(5) of the main article, and with rotation operator $\mathcal{R}_\alpha$ and scaling operator $\mathcal{S}_a$ defined by
\begin{align}
%\begin{array}{rl}
(\mathcal{R}_\alpha t)(\mathbf{x}) &= t(\mathbf{R}_\alpha^{-1} \mathbf{x}),\\
(\mathcal{S}_a t)(\mathbf{x})  &= a^{-1} t(a \mathbf{x}),
%\end{array}
\end{align}
with rotation matrix $\mathbf{R}_\alpha$ representing a counter clockwise rotation of angle $\alpha$.
%, and with scaling operator $\mathcal{S}_a$ defined by
%\begin{equation}
%(\mathcal{S}_a t)(\mathbf{x})  = a^{-1} t(a \mathbf{x}).
%\end{equation}
By taking the maximum over scales $a$ (in a suitable range $[a_-,a_+]$) and rotations $\alpha$, the response of the best matching template is obtained at each location $\mathbf{x}$, and invariance is obtained with respect to scaling and rotation of the object of interest.

The rotation/scale invariant extension of the $SE(2)$ cross-correlation template matching case (Eqs.~(25)-(26) of the main article) is equally straightforward; for the linear potential we define
\begin{equation}
\label{eq:linearFunctionalInvSE2}
{P}_{lin,inv}^{SE(2)}(\mathbf{x}) :=   \underset{\begin{array}{c}a\in[a_-,a_+],\\\alpha \in [0,2\pi)\end{array}}{\operatorname{max}}( \mathcal{T}_{\mathbf{x}} \mathcal{S}_a \mathcal{R}_\alpha \; T , U_{f})_{\mathbb{L}_2(SE(2))},
\end{equation}
and for the logistic potential we define
\begin{equation}
\label{eq:logisticFunctionalInvSE2}
{P}_{log,inv}^{SE(2)}(\mathbf{x}) :=  \underset{\begin{array}{c}a\in[a_-,a_+],\\\alpha \in [0,2\pi)\end{array}}{\operatorname{max}} \sigma \left( ( \mathcal{T}_{\mathbf{x}} \mathcal{S}_a \mathcal{R}_\alpha \; T , U_{f})_{\mathbb{L}_2(SE(2))} \right),
\end{equation}
with for orientation score objects $T,U_f \in \mathbb{L}_2(SE(2))$ the rotation and scaling operators defined respectively by
\begin{align}
(\mathcal{R}_\alpha T)(\mathbf{x},\theta) &= T(\mathbf{R}_\alpha^{-1} \mathbf{x}, \theta - \alpha),\\
(\mathcal{S}_a T)(\mathbf{x},\theta)  &= a^{-1} T(a \mathbf{x},\theta).
\end{align}

It depends on the addressed template matching problem whether or not such invariance is desirable or not. In many applications the object is to be found in a human environment context, in which some objects tend to appear in specific orientations or at typical scales, and in which case rotation/scale invariance might not be desirable. E.g. the sizes of anatomical structures in the retina are relatively constant among different subjects (constant scale) and retinal images are typically taken at a fixed orientation (constant rotation). In the pupil detection problem the subjects typically appear in upright position behind the camera (constant rotation), and within a reasonable distance to the camera (constant scale). In the next Subsec.~\ref{subsubsec:detectionResults} we indeed show that in the applications considered in this manuscript rotation/scale invariance is not necessarily a desired property, and that computation time linearly increases with the number of rotations/scalings tested for (cf. Subsec.~\ref{subsubsec:timings}).

%In the applications discussed in the main article the objects of interest always appear under more or less the same orientation and at the same scale. Introducing additional degrees of freedom, i.e. rotation and scale invariance, is then not a wise thing to do. In fact, as we will see in the next Subsec. ?? we could not obtain improved results by introducing scale and rotation invariance.

%For a generic object recognition task, global rotation or scale invariance are not necessarily desired properties. Indeed, datasets often contain objects in a human environment context, in which some objects tend to appear in specific orientations (e.g. eye-browses are often horizontal above the eye, vascular trees propogate more in horizontal directions than vertical directions etc.). Discarding those hints by building a rotation invariant typically (likewise the scattering approach in \cite[ch:6.3.1]{siffre}) has an adversary effect on performances in the training phase, while increasing the computational load. See supplementary material where this can be observed.

\subsection{Results with Rotation and Scale Invariance}
Here we perform rotation and scale invariant template matching via the extension described in Subsec.~\ref{subsec:basicExtension}. We selected the best template combination for each specific application and compared non-invariant template matching (as described in the main article) to rotation and/or scale invariant template matching (Subsec.~\ref{subsec:basicExtension}). The best template combination for ONH detection was $A_{\mathbb{R}^2}+C_{log:SE(2)}$, for fovea detection this was $C_{lin:\mathbb{R}^2}+C_{log:SE(2)}$, and for pupil detection this was $D_{lin:\mathbb{R}^2}+E_{lin:SE(2)}$.

For the retinal applications we only tested for rotation invariance with
$$
\alpha \in \{-\frac{\pi}{6}, -\frac{\pi}{8}, -\frac{\pi}{12}, -\frac{\pi}{24}, 0, \frac{\pi}{24} , \frac{\pi}{12}, \frac{\pi}{8}, \frac{\pi}{6}\},
$$
and did not included scale invariance since each retinal image was already rescaled to a standardized resolution (see Subsec.~\ref{subsec:ProcessingPipeline}). In pupil detection we tested for a range of scalings with
$$
a \in \{0.7,0.8,0.9,1.0,1.1,1.2,1.3\}
$$
to deal with varying pupil sizes caused by varying distances to the camera; and we tested for a range of rotations with
$$
\alpha \in \{-\frac{\pi}{4}, -\frac{\pi}{8}, -\frac{\pi}{16}, 0, \frac{\pi}{16} , \frac{\pi}{8}, \frac{\pi}{4}\}
$$
to deal with rotations of the head.

\subsubsection{Detection Results}
\label{subsubsec:detectionResults}
The detection results are shown in Table.~\ref{tab:resultsInvariant}. Here we can see that in all three applications the inclusion of a rotation/scale invariant matching scheme results in a slight decrease in performance.
This can be explained by the fact that variations in scale an rotation within the databases are small, and that the trained templates can already deal robustly with these variations (due to the presence of such variations in the training set). By introducing rotation/scale invariance one then only increases the likelihood of false positive detections.

\subsubsection{Computation Time}
\label{subsubsec:timings}
The effect on computation time of rotation/scale invariant matching is shown in Fig.~\ref{fig:RotScalTimings}. Here one sees that computation time linearly increases with the number of template rotations and scalings tested for. This timings-experiment is performed on the pupil detection application, and the shown timings are only of \emph{step 5} of the full detection pipeline (see Subsec.~\ref{subsec:ProcessingPipeline} and Table 1) as this is the only step that is affected by the rotation/scale invariant extension.

\begin{table}
\centering
\caption{Average template matching results ($\pm$ standard deviation, number of fails between parenthesis) for optic nerve head (ONH), fovea, and pupil detection in 5-fold cross validation.
%, with or without rotation/scale invariant matching.
}
\begin{tabular}{ll}
\toprule
Method & Success rate\\

\midrule
\midrule
\multicolumn{2}{l}{\;\;\;\;\;\;\;ONH Detection \;\;\; (1737 images)}\\
\midrule

  No invariance &  99.83\% {\tiny $\pm$ 0.26\%} (3)\\
  Rotation invariance &  99.60\% {\tiny $\pm$ 0.16\%} (7)\\

\midrule
\multicolumn{2}{l}{\;\;\;\;\;\;\;Fovea Detection \;\; (1616 images)}\\
\midrule

  No invariance &  99.32\% {\tiny $\pm$ 0.26\%} (11)\\
  Rotation invariance &  97.10\% {\tiny $\pm$ 0.65\%} (47)\\

\midrule
\multicolumn{2}{l}{\;\;\;\;\;\;\;Pupil Detection \;\;\; (1521 images)}\\
\midrule

  No invariance &  95.86\% {\tiny $\pm$ 0.89\%} (63)\\
  Rotation invariance &  94.48\% {\tiny $\pm$ 1.62\%} (84)\\
  Scale invariance &  95.33\% {\tiny $\pm$ 1.46\%} (71)\\
  Rotation + scale invariance &  94.28\% {\tiny $\pm$ 2.10\%} (87)\\

\bottomrule
\end{tabular}
\label{tab:resultsInvariant}
\end{table}

\begin{figure}
\centerline{
\includegraphics[width=\hsize]{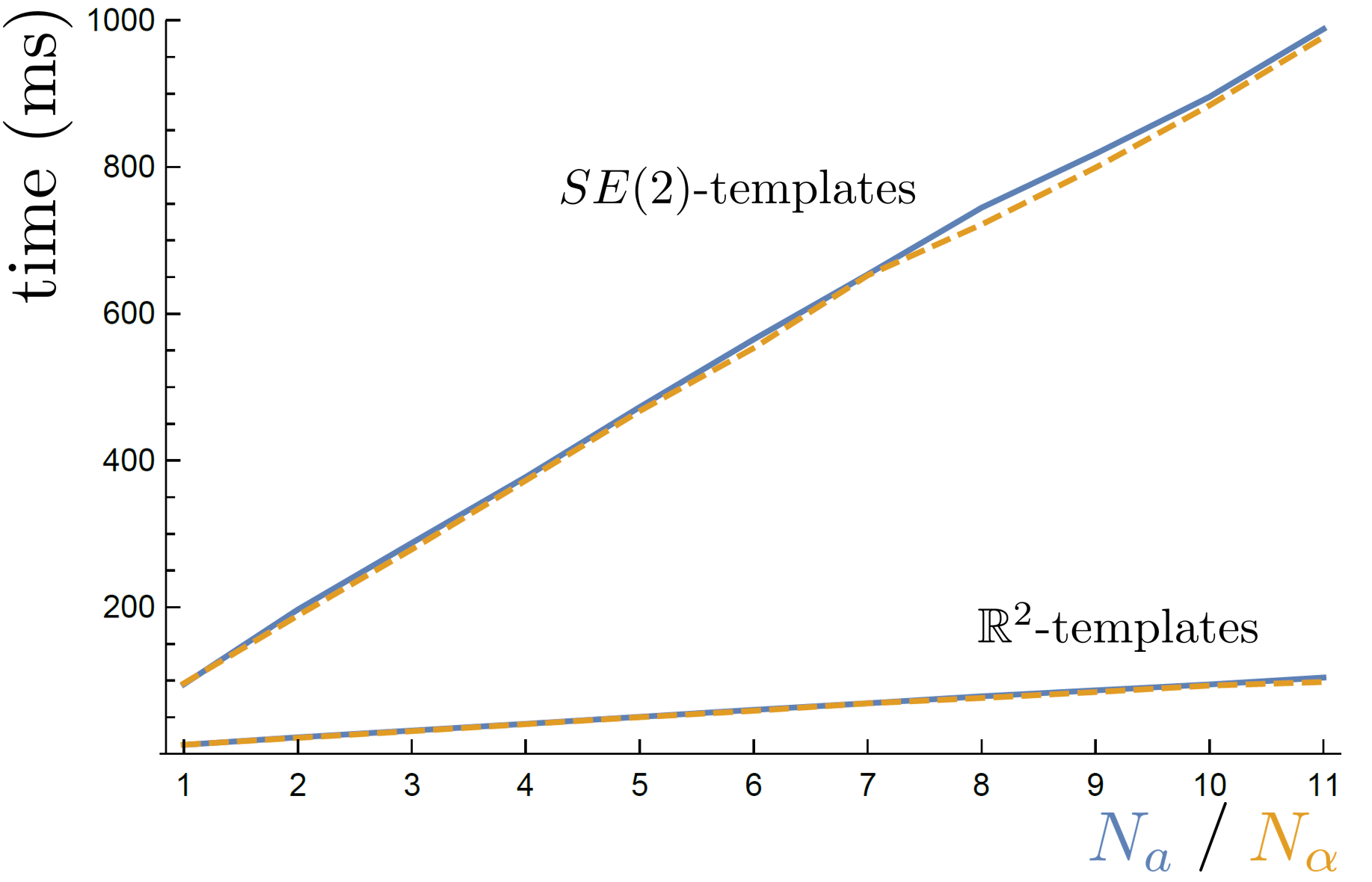}
}
\caption{
Average computation times for the detection of one pupil/eye per image, using $SE(2)$ or $\mathbb{R}^2$ templates, testing for different template orientations or scalings. Two experiments are shown, in blue the number of template orientations $N_\alpha=1$ and the number of scalings $N_a$ is varied, in orange-dashed the number of scalings $N_a=1$ and the number of rotations is $N_\alpha$ is varied.
\label{fig:RotScalTimings}}
\end{figure}

%\clearpage
%\clearpage

%%%%%%%%%%%%%%%%%%%%%%%%%%%%%%%%%%%%%%%%%%%%%
%%%%%%%%%%%%%%% Bibliography
%%%%%%%%%%%%%%%%%%%%%%%%%%%%%%%%%%%%%%%%%%%%%

\bibliographystyle{IEEEtran}
\bibliography{references}

%%%%%%%%%%%%%%%%%%%%%%%%%%%%%%%%%%%%%%%%%%%%%
%%%%%%%%%%%%%%%End of Document
%%%%%%%%%%%%%%%%%%%%%%%%%%%%%%%%%%%%%%%%%%%%%

\end{document}